\documentclass[11pt,letterpaper]{article}
\usepackage[margin=1in]{geometry}
\usepackage[T1]{fontenc}
\usepackage[authoryear,round]{natbib}
\usepackage{amsmath,amssymb,amsthm}
\usepackage{graphicx,booktabs,wrapfig,float}
\usepackage{needspace}
\usepackage{algorithm,algpseudocode,placeins}
\usepackage[font=small]{caption}
\usepackage{xcolor}
\definecolor{paperlink}{HTML}{245A81}
\usepackage{hyperref}
\hypersetup{
  colorlinks=true,
  allcolors=paperlink,
  pdfborder={0 0 0},
  pdftitle={Enabling Token-wise Importance Sampling in Off-Policy RLVR with Proximal KL Regularization},
  pdfauthor={Hao Bai, Jiarui Yao, Tong Zhang}
}
\usepackage{url}
\newtheorem{theorem}{Theorem}
\newtheorem{proposition}[theorem]{Proposition}
\newtheorem{lemma}[theorem]{Lemma}
\newtheorem{corollary}[theorem]{Corollary}

\newcommand{\E}{\mathbb E}
\newcommand{\Var}{\operatorname{Var}}
\newcommand{\KL}{\operatorname{KL}}
\theoremstyle{remark}
\newtheorem{remark}[theorem]{Remark}

\newsavebox{\theoryguidebox}

\newcommand{\ours}{\textit{Proximal KL}}
\newcommand{\rewardonly}{PG}

\title{Enabling Token-wise Importance Sampling in Off-Policy RLVR with Proximal KL Regularization}
\author{
    \textbf{Hao Bai \quad Jiarui Yao \quad Tong Zhang} \\[0.5em]
    \normalsize University of Illinois Urbana-Champaign
}

\date{}

\begin{document}
\maketitle
\begin{abstract}
In reinforcement learning with verifiable rewards (RLVR), asynchronous generation,
differences between rollout and training policies, and response reuse across optimizer minibatch
updates can all produce off-policy data. Token-wise importance weighting allows policy-gradient methods
to reuse these responses without multiplying likelihood ratios across a long response,
but it generally biases the gradient. We prove that, under coverage and local
expressiveness conditions, adding full future KL credit makes the exact proximal
optimizer the update's unique stationary policy. The result holds even when responses
come from a mixture of older policies and requires neither full-response importance
weights nor a fitted critic. We also prove convergence of continuous-time tabular
population updates under persistent coverage and time-varying behavior policies, and
give a finite-sample reward guarantee for linear logistic models under additional
assumptions. These results motivate \ours{}, a practical off-policy RLVR method that
separates the proximal reference from the sampling policy and replaces ordinary PPO
clipping with bounded future KL credit. 
Across all seven matched off-policy comparisons on Qwen3-1.7B and Qwen3-4B,
\ours{} achieves higher micro-averaged accuracy than GRPO over eight math
benchmarks, spanning two learning rates, response reuse, and rollout staleness.
The absolute improvement in micro-averaged accuracy reaches $2.58\%$.
A separate on-policy comparison shows comparable aggregate accuracy for the two methods.
These gains accompany stable optimization: in the periodic-refresh
experiments, \ours{} avoids the training collapses observed
in PG and CISPO, attains higher training accuracy than GRPO, and reaches
higher accuracy in fewer policy updates. To explain these gains, ablations first compare
how off-policy sampling is corrected: token-wise weighting outperforms sequence-level
and unweighted controls. With token-wise correction retained, future KL feedback
reduces sensitivity to the tested baseline perturbation relative to PPO clipping.
The extent of this feedback matters as well: a finite future cost horizon improves AIME24
mean@16 and pass@16 over the full tail, while fixed-policy measurements show that
truncation reduces correction variance over the tested horizons.
\par\smallskip
\noindent Project page: \url{https://biechi.github.io/blog/offpolicy-proximal-kl/}
\end{abstract}

\section{Introduction}
\label{sec:introduction}

Reinforcement learning with verifiable rewards (RLVR) optimizes a policy from
verifier-scored responses, often generated by a behavior policy that differs from the
policy being optimized. This off-policy mismatch has several sources. Asynchronous
systems overlap generation and optimization, so responses may come from earlier policy
versions \citep{fu2025areal,bai2026asyncwebrl,zhu2025slime,radixark2026miles}.
Reusing collected trajectories across optimizer minibatch updates introduces mismatch
as the current policy changes, even when collection is synchronous. We study the common
problem of learning from these off-policy responses.

Policy-gradient methods such as GRPO and DAPO use token-wise importance ratios
\citep{shao2024deepseekmath,yu2025dapo}. Exact importance correction of a response-level
objective instead multiplies the ratios across the response, which can have high
variance for long responses. Token-wise weighting avoids this product but generally
biases the gradient. We ask when an update based on token-wise weights and sampled
rewards, without a fitted value or Q function, still retains the proximal target.
This target is the optimizer of a KL-regularized objective.

We study a proximal objective that penalizes the reverse KL divergence from the current policy
to a reference policy held fixed during each optimization round
\citep{geist2019regularized,tomar2020mdpo}. The reference defines the improvement
target, while the behavior policy determines the sampling distribution. Following batch-size-invariant policy optimization
\citep{hilton2022batch}, we separate these policies so the target remains fixed even
when the data-generating distribution differs from the current policy. For
autoregressive responses, the regularized policy gradient assigns each token the current and all future token log-ratio costs, as
expressed by Future-KL credit \citep{yao2026futurekl}. We combine this full credit with
a single importance ratio for the current token.

\begingroup
Our main result shows that, under coverage and local expressiveness, the expected
token-wise update with full future KL credit vanishes \textbf{if and only if} the learner equals
the proximal optimizer. The result holds for any supported stale behavior policy or
mixture of behavior-policy versions, although the expected update can differ from the
exact proximal gradient away from the optimizer. The proof uses the soft
path-consistency identity established by PCL and Trust-PCL
\citep{nachum2017pcl,nachum2018trustpcl}. With independent tabular logits and persistent
coverage, we further prove convergence of continuous-time updates
(Appendix~\ref{app:async-theory}).

To study learning from sampled responses, we also give a finite-sample reward
guarantee after repeated reference refreshes for a linear logistic model, whose token
logits depend linearly on trainable parameters through fixed features
(summarized in Theorem~\ref{ext:reward-informal}, with the precise statement in
Appendix~\ref{ext:shared}). Under its assumptions, the final policy approaches the
comparator's expected reward for a single response. Its mismatch condition bounds the sum of
learner-weighted squared token log ratios along each response path, and the update
never forms the response-level product of ratios. This analysis motivates a
practical algorithm that can also be applied to nonlinear models such as Transformers.
\par
\endgroup

\begingroup
Guided by these results, we develop \textit{Proximal KL Regularization} (\ours{}). The
practical update records each token's behavior probability, periodically refreshes the
proximal reference, and subtracts bounded future KL costs from a group-relative reward
advantage before applying the token-wise importance weight. A bound on each log-ratio
cost limits individual contributions, while a tunable horizon limits how many future
costs are included. The update omits ordinary PPO clipping. Section~\ref{sec:algorithm}
discusses the trade-off between bias and variance when bounding the costs, and
Section~\ref{sec:ablation} evaluates these choices empirically.
\par
\endgroup

\begingroup
\textbf{Contributions.} To our knowledge, this is the first stationary-policy guarantee
for critic-free token-wise importance weighting with full future KL credit under
supported stale behavior mixtures. We also prove tabular convergence and, for a
realizable linear logistic model, a finite-sample reward guarantee under
correct-solution coverage using a full-credit group update with repeated reference
refreshes (Theorem~\ref{ext:reward-informal}). 
Our empirical comparisons show that \ours{} achieves higher micro-averaged
held-out accuracy than GRPO in all seven matched off-policy settings in
Section~\ref{sec:main-results}. These gains span eight math benchmarks,
1.7B and 4B models, and the tested learning rates and rollout staleness levels,
including fresh rollouts reused across optimizer minibatches. In the separate
on-policy comparison, \ours{} attains aggregate accuracy comparable to GRPO.
Complementary periodic-refresh experiments show that \ours{} sustains training
gains through the collapses observed in PG and CISPO and attains higher training
accuracy than GRPO at comparable optimizer-update counts
(Section~\ref{sec:update-dynamics}). The
ablations connect these outcomes to update design: current-token
importance correction outperforms Sequence-level IS and IS-free settings,
and future KL feedback is less sensitive to the tested baseline
perturbation than PPO clipping. The tightest tested future-cost bound combines strong accuracy with smaller
gradient-norm excursions, while a finite horizon outperforms both
strong local and full-tail corrections. Fixed-policy measurements
confirm that truncation reduces correction variance over the tested
horizons (Section~\ref{sec:ablation}).

\par
\endgroup

\section{Related Work}

\begingroup
\paragraph{Proximal updates.}
Proximal methods control policy changes relative to a reference policy,
as in TRPO \citep{schulman2015trpo}. Regularized MDPs and mirror-descent
policy optimization formalize KL-based updates
\citep{geist2019regularized,tomar2020mdpo}.
Batch-size-invariant policy optimization \citep{hilton2022batch} distinguishes the proximal reference from the
behavior policy, separating the optimization target from the sampling distribution.
AReaL and AsyncWebRL apply this distinction to asynchronous optimization
\citep{fu2025areal,bai2026asyncwebrl}.
Slime and Miles support asynchronous rollout and training pipelines that reduce
bubble time \citep{zhu2025slime,radixark2026miles}.
FRPO derives the on-policy Future-KL gradient identity for autoregressive
regularization \citep{yao2026futurekl}. We use this identity for expected reward regularized against a fixed proximal reference
(Appendix~\ref{app:future-derivation}), then characterize the stationary policy when
full future credit is combined with a single current-token importance ratio under
behavior-policy mismatch, including stale behavior mixtures.
\par
\endgroup

\begingroup
PCL establishes soft path consistency and minimizes the corresponding
errors on off-policy trajectories \citep{nachum2017pcl}. Trust-PCL
extends this approach to relative-entropy regularization against an
updated reference \citep{nachum2018trustpcl}. Soft Q-learning learns
soft Bellman values and an associated sampling policy
\citep{haarnoja2017softq}; Unified PCL represents policy and value
with one Q model \citep{nachum2017pcl}. \citet{schulman2017equivalence} derive regularized Monte Carlo policy gradients and
relate them to soft Q-learning with value fitting. Our analysis concerns the stationary
policy of an update that combines policy scores with token-wise importance weights,
sampled rewards, and full future credit. The update requires no fitted values or
optimization of a consistency residual (Section~\ref{sec:motivation}).
\par
\endgroup

\paragraph{Policy-gradient methods.}
REINFORCE uses sampled policy scores \citep{williams1992reinforce},
while PPO uses a clipped policy-gradient surrogate
\citep{schulman2017ppo}. For language-model reasoning, GRPO constructs
group-relative advantages \citep{shao2024deepseekmath}, and DAPO
develops clipping and sampling choices \citep{yu2025dapo}.
V-trace and Retrace correct bootstrapped value or Q targets for
off-policy data using truncated importance traces
\citep{espeholt2018impala,munos2016retrace}. V-trace also weights its
policy update by a current-action ratio. Their fixed-point guarantees concern corrected value learning. We establish the unique
stationary policy of the actor update itself, using sampled rewards, full future KL
credit, and current-token importance weighting without a fitted critic.

Compatible approximation and policy-improvement closure provide structural
conditions for policy-gradient convergence
\citep{agarwal2021pg,yuan2023loglinear,bhandari2024global}.
Appendix~\ref{ext:representation-condition} relates these conditions to our
bounded Gibbs-target representation assumption and its concrete examples.

We also acknowledge related unpublished work by Audrey Huang and collaborators
on coverage collapse and reference re-anchoring in policy optimization after
pre-training (private communication with Nan Jiang).

\begingroup
\section{Setup and Proximal Update}
\label{sec:setup}
\par
\endgroup

\begingroup
The proximal reference defines the optimization target, while the behavior policy
generates training responses. Asynchronous RLVR systems such as AReaL, StreamRL,
AsyncWebRL, Slime, and Miles overlap generation and policy updates
to reduce bubble time
\citep{fu2025areal,zhong2025streamrl,bai2026asyncwebrl,zhu2025slime,radixark2026miles}.
The behavior policy may therefore differ from both the current learner
and the fixed proximal reference.
\par
\endgroup

\begingroup
We regularize expected reward by the reverse KL divergence from the learner to the
reference
\citep{geist2019regularized,tomar2020mdpo}, with the reference policy
updated between optimization rounds. Let the prompt set and generation
tree be finite. Write $x\sim\mathcal D$ for a prompt, $\tau$ for its
response, and $R(x,\tau)\in\{0,1\}$ for the verifier reward. The current policy $\pi_\theta$ is optimized
relative to $\pi_k$, fixed during round $k$;
responses are sampled from $\pi_{\mathrm{behav}}$:
\begin{equation}
 J_k(\theta)=\mathbb E_{x\sim\mathcal D,\,\tau\sim\pi_\theta}[R(x,\tau)]
 -\beta\mathbb E_{x\sim\mathcal D}
 D_{\mathrm{KL}}\!\left(\pi_\theta(\cdot\mid x)\Vert\pi_k(\cdot\mid x)\right),
 \qquad \beta>0.
 \label{eq:mech-objective}
\end{equation}
Its maximizer over response policies is
$\pi_k^*(\tau\mid x)=\pi_k(\tau\mid x)e^{R(x,\tau)/\beta}/Z_k(x)$,
where $Z_k(x)=\mathbb E_{\tau\sim\pi_k(\cdot\mid x)}e^{R(x,\tau)/\beta}$.
The objective-gap identity in Equation~\eqref{eq:prox-gap} gives this standard Gibbs
characterization; we include it for completeness. For binary rewards, the Gibbs step
multiplies the reference probability of each correct response by $e^{1/\beta}$ before
renormalizing. Thus the reference determines which responses are reweighted, and
$\beta$ controls the strength of this reweighting.
\par
\endgroup

\begingroup
Proximal regularization also has benefits in the on-policy setting
\citep{schulman2015trpo,tomar2020mdpo}. For binary rewards, its exact
unrestricted optimizer improves every prompt whose success probability
lies strictly between zero and one
(Proposition~\ref{thm:prox-improvement}).
Appendix~\ref{app:shared-head-separation} illustrates the benefit of
proximal regularization in a
shared-parameter model; Section~\ref{sec:main-results} reports an
on-policy comparison.
\par
\endgroup

\begingroup
Differentiating the response-level objective gives the Future-KL
policy-gradient expression \citep{schulman2017equivalence,yao2026futurekl}:
each token's log-probability gradient is weighted by the reward advantage minus the
current and future log-ratio costs. Penalizing only the current token generally omits
part of the regularization gradient
(Appendix~\ref{app:future-derivation}). Let $T$ be the maximum response
length, padding shorter responses with deterministic terminal actions.
For $\tau=(a_1,\ldots,a_T)$ and
$h_i=(x,a_{<i})$, define the learner-to-reference ratio and the learner-to-behavior importance weight as
\begin{equation}
  \rho_i=\frac{\pi_\theta(a_i\mid h_i)}{\pi_k(a_i\mid h_i)},\qquad
 w_i=\frac{\pi_\theta(a_i\mid h_i)}{\pi_{\mathrm{behav}}(a_i\mid h_i)}
 =\frac{\pi_k(a_i\mid h_i)}{\pi_{\mathrm{behav}}(a_i\mid h_i)}\rho_i.
 \label{eq:importance-ratios}
\end{equation}
\par
\endgroup

\begingroup
Importance weighting and future KL credit serve different purposes: the weights correct
the sampling distribution, while the KL costs define the regularized return. Exact response-level
importance sampling uses the product of current-to-behavior ratios in
Equation~\eqref{eq:importance-ratios}, whose variance can grow
exponentially with response length.
Starting from the policy-gradient expression for
objective~\eqref{eq:mech-objective}, we replace the response-wise
importance weight by one ratio for each sampled token. For the current
policy $\theta$, the expected token-wise policy-gradient field is
\begin{equation}
 \begin{aligned}
 \widehat A_i^{\mathrm{KL}}(\theta)
 &=R-b(h_i)-\beta\sum_{j=i}^{T}\log\rho_j(\theta),\\[-1pt]
 M_{\mathrm{behav}}(\theta)
 &=\mathbb E_{\tau\sim\pi_{\mathrm{behav}}}\sum_i
 \frac{\pi_\theta(a_i\mid h_i)}{\pi_{\mathrm{behav}}(a_i\mid h_i)}
 \widehat A_i^{\mathrm{KL}}(\theta)
 \nabla_\theta\log\pi_\theta(a_i\mid h_i).
 \end{aligned}
 \label{eq:main-single-token-surrogate}
\end{equation}
The expectation also includes $x\sim\mathcal D$. The baseline $b(h_i)$ depends only on
the history, is held fixed during differentiation, and may be zero. We likewise treat
the advantage as a fixed coefficient when computing an update, then recompute it after
the policy changes. The resulting field uses observed rewards and policy log
probabilities without a fitted value or Q function. We analyze its stationary points
directly, allowing the field to depend on the behavior policy.
\par
\endgroup

\makeatletter
\@nobreakfalse
\makeatother
\section{Algorithm}
\label{sec:algorithm}

\textbf{Proximal KL.} Our update combines token-wise importance weighting with bounded future KL credit. The
update uses verifier rewards, policy log probabilities, and statistics from groups of
responses, without fitting a value or Q function.

\begingroup
\begingroup
Algorithm~\ref{alg:proximal-kl} specifies the token-wise proximal update for
differentiable autoregressive policies, including Transformers. Each round holds the
reference fixed for $L$ optimizer updates. Updates may reuse sampled response
groups, and the behavior policy follows its own refresh schedule. Each sampled
token is paired with its probability under the behavior policy; the frozen
reference defines the KL cost.

For a group of $G\ge2$ responses, let $\widehat A_m$ denote the reward
advantage assigned to response $m$, held fixed during differentiation.
For response $m$ of length $T_m$, write $h_{m,i}=(x,a_{m,<i})$ for
the history before token $a_{m,i}$. The two ratios in
Equation~\eqref{eq:importance-ratios} are
\begin{equation}
  \rho_{m,i}=\frac{\pi_\theta(a_{m,i}\mid h_{m,i})}
 {\pi_k(a_{m,i}\mid h_{m,i})},\quad
 w_{m,i}=\frac{\pi_\theta(a_{m,i}\mid h_{m,i})}
 {\pi_{\mathrm{behav},m}(a_{m,i}\mid h_{m,i})}
 =\frac{\pi_k(a_{m,i}\mid h_{m,i})}
 {\pi_{\mathrm{behav},m}(a_{m,i}\mid h_{m,i})}\rho_{m,i}.
 \label{eq:sample-token-ratios}
\end{equation}
Here $\pi_{\mathrm{behav},m}$ is the behavior policy that generated
response $m$. The KL correction uses
the learner-to-reference log ratio $\ell_{m,i}=\log\rho_{m,i}$;
the score loss uses the learner-to-behavior weight $w_{m,i}$.
Thus $\rho_{m,i}$ measures policy movement from the frozen reference,
while $w_{m,i}$ corrects the sampled token's action distribution for
the behavior policy that generated it. They coincide when
$\pi_{\mathrm{behav},m}=\pi_k$; using $\rho_{m,i}$ alone in the loss
would otherwise omit the reference-to-behavior correction in
Equation~\eqref{eq:sample-token-ratios}.
Define the practical clipped correction by
\begin{equation}
 C_{m,i}^{(H)}=\beta\sum_{j=i}^{\min(T_m,i+H-1)}
       \operatorname{clip}(\ell_{m,j},-\epsilon_F,\epsilon_F).
 \label{eq:practical-proximal-kl}
\end{equation}
Here $H$ is the window size for the future KL correction: at token $i$,
the window covers the current token and up to $H-1$ subsequent tokens.
The full-tail choice $H=\mathrm{full}$
uses every remaining token from position $i$ to the end of the response;
equivalently, the sum ends at $T_m$ without window truncation.
Setting $\epsilon_F=\infty$ disables cost clipping.
There is no window-length normalization, and
$|C_{m,i}^{(H)}|\le\beta\epsilon_F\min(H,T_m-i+1)$ for finite bounds.
A tail-decay variant retains unit weight for the first $H$ positions and
weights the remaining clipped costs with a linear decrease to zero at the
response end.

For each optimizer minibatch, minimize the following loss, averaged over its
nonpadding action tokens:
\begin{equation}
 \mathcal L^{(H)}
 =-\frac1{N_{\mathrm{tok}}}\sum_{m,i}
 \operatorname{sg}\!\left[
 \frac{\pi_\theta(a_{m,i}\mid h_{m,i})}
      {\pi_{\mathrm{behav},m}(a_{m,i}\mid h_{m,i})}
 (\widehat A_m-C_{m,i}^{(H)})\right]
 \log\pi_\theta(a_{m,i}\mid h_{m,i}).
 \label{eq:implemented-loss}
\end{equation}
The sum is over sampled action tokens in the minibatch,
$N_{\mathrm{tok}}$ is the number of nonpadding action tokens, and $\operatorname{sg}$ stops
all gradients through its argument. Coefficients are recomputed at
every learner update. Adam minimizes this loss without ordinary PPO clipping.

Cost clipping limits the sampled KL credit, and stop-gradient holds the resulting
coefficient fixed when differentiating the loss. Even when a log-ratio cost reaches its
bound, the token's score gradient remains active. Thus saturation of a cost does not
suppress that token's entire update. Appendix~\ref{app:practical-loss} specifies the group-normalized reward
advantage and relates this loss to the ascent surrogate.

\begin{algorithm}[htbp]
\caption{Proximal KL with token-wise importance weighting}
\label{alg:proximal-kl}
\begin{algorithmic}[1]
\Require Initial learner $\theta_0$; rounds $K$; optimizer updates per round $L$;
\Statex \hspace{\algorithmicindent} group size $G\ge2$; KL coefficient $\beta$; window $H$; cost bound $\epsilon_F$; Adam settings.
\Ensure Final policy $\pi_{\theta_K}$.
\State Initialize the behavior policy with $\theta_0$; retain its probability for each sampled token.
\For{$k=0,\ldots,K-1$}
  \State Freeze $\pi_k\gets\pi_{\theta_k}$ and set $\vartheta_0\gets\theta_k$.
  \For{$j=0,\ldots,L-1$}
    \State Set current learner $\theta\gets\vartheta_j$.
    \State Select new or reused response groups with behavior probabilities and verifier rewards.\label{algline:collect}
    \State Compute reward advantages $\widehat A_m$ for these groups.\label{algline:advantage}
    \State Select optimizer minibatch $\mathcal B_j$ from these groups.\label{algline:minibatch}
    \State Recompute $\rho_{m,i}$ and $w_{m,i}$ by Equation~\eqref{eq:sample-token-ratios}.
    \State Form $\ell_{m,i}=\log\rho_{m,i}$ per token.
    \State Form $C_{m,i}^{(H)}$ by Equation~\eqref{eq:practical-proximal-kl}.\label{algline:credit}
    \State Form loss $\mathcal L^{(H)}$ by Equation~\eqref{eq:implemented-loss} on $\mathcal B_j$.\label{algline:loss}
    \State Update $\vartheta_{j+1}$ with Adam.\label{algline:update}
    \State Refresh the behavior policy according to its update schedule.
  \EndFor
  \State Set $\theta_{k+1}\gets\vartheta_L$.\label{algline:refresh}
\EndFor
\State \Return $\pi_{\theta_K}$.
\end{algorithmic}
\end{algorithm}

Appendix~\ref{ext:shared} specifies
the changes to Algorithm~\ref{alg:proximal-kl} used for its finite-sample
analysis, including projection of each group-gradient step onto a
parameter ball around the reference. This projection is used only in
the theoretical variant.
\par
\endgroup

\par
\endgroup

\begin{figure}[!ht]
 \centering
 \includegraphics[width=\linewidth]{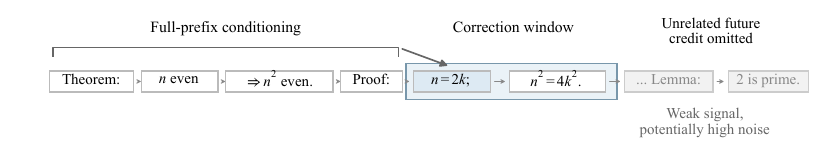}
 \caption{Finite future-KL window. The policy conditions on the full prefix,
 while KL credit includes the current token and up to $H-1$ subsequent tokens.}
 \label{fig:future-window-schematic}
\end{figure}

\begingroup
\paragraph{Trade-off between bias and variance in \ours{} future tail.}
\label{sec:practical-bias-variance}
In the on-policy population setting, full, unmodified future KL credit supplies the
complete regularization gradient. Estimating the full tail, however, can produce
substantial sampling fluctuations. Clipping the costs limits extreme contributions,
while shortening the window limits their accumulation. Both choices can introduce
approximation bias in exchange for controlling these fluctuations, serving a stabilizing
role analogous to gradient clipping in SGD.
Figure~\ref{fig:future-window-schematic} shows how an intermediate window
retains nearby future costs while limiting the number of terms in the correction.
Token-wise off-policy weighting introduces a separate approximation, as
discussed in Section~\ref{sec:setup}. Section~\ref{sec:ablation} evaluates
these choices empirically.
\par
\endgroup

\ours{} uses the following ascent surrogate; training minimizes
its negative:
\begin{equation}
 \mathcal L_{\mathrm{PKL}}(\theta)
 =\mathbb E_{\pi_{\mathrm{behav}}}\sum_i
 \frac{\pi_\theta}{\pi_{\mathrm{behav}}}
 \operatorname{sg}[\widehat A_m-C_i^{(H)}].
 \label{eq:experiment-proximal}
\end{equation}

Here $m$ indexes the sampled response and $i$ indexes its tokens; the policy ratios
and $C_i^{(H)}$ are evaluated at token $i$ of response $m$. The expectation includes
prompts and sampled responses. We suppress the response index on token quantities
in these comparison formulas.

When accumulated log-ratio costs are positive, they reduce the coefficient multiplying
the policy score. This provides feedback on movement from the reference without a
clipping gate that depends on the advantage sign. With raw verifier rewards and a
history-only baseline, a full future window ($H\ge T$) with no cost clipping
supplies the regularization gradient on-policy and preserves the proximal stationary
target under the population assumptions in Section~\ref{sec:motivation}.
Group normalization can change this field, so the guarantee does not
automatically extend to updates using normalized advantages.
Horizon truncation controls the extent of the practical KL feedback.

To put this feedback in context, we compare it with updates that control policy movement
through clipping and with an unclipped policy-gradient update. These comparisons distinguish the effect
of future KL credit from the choice of importance weighting.

\textbf{GRPO.} Our baseline is a variant of the GRPO loss \citep{shao2024deepseekmath}.
The original objective also includes a sampled KL penalty to a separate
reference model, typically the initial SFT policy $\pi_0$. This term
penalizes departure from that reference in addition to PPO clipping.
Following DAPO \citep{yu2025dapo}, the surrogate below omits this
initial-reference penalty: its objective is
verifier reward, without a constraint that the final policy remain close
to initialization. By contrast, the KL term in \ours{} penalizes movement
from the periodically refreshed $\pi_k$ during each improvement round.
It controls updates while allowing cumulative movement away from $\pi_0$;
it therefore has a different role from an initial-reference penalty.

Our GRPO variant uses symmetric PPO clipping and decouples the proximal
reference from the behavior policy \citep{hilton2022batch,fu2025areal},
with the loss averaged over action tokens. Its ascent surrogate is:
\begin{equation}
 \mathcal L_{\mathrm{GRPO}}(\theta)
 =\mathbb E_{\pi_{\mathrm{behav}}}\sum_i
 \frac{\pi_k}{\pi_{\mathrm{behav}}}
 \min\!\left(\rho_i\widehat A_m,
 \operatorname{clip}(\rho_i,1-\epsilon_{\mathrm{clip}}, 1+\epsilon_{\mathrm{clip}})\widehat A_m\right).
 \label{eq:experiment-grpo}
\end{equation}
Here $\epsilon_{\mathrm{clip}}$ is the PPO clipping width.
The surrogate omits the future KL gradient, while token-wise importance
weighting leaves prefixes and suffixes under the behavior distribution.

\begin{figure}[!ht]
 \centering
 \includegraphics[width=\linewidth]{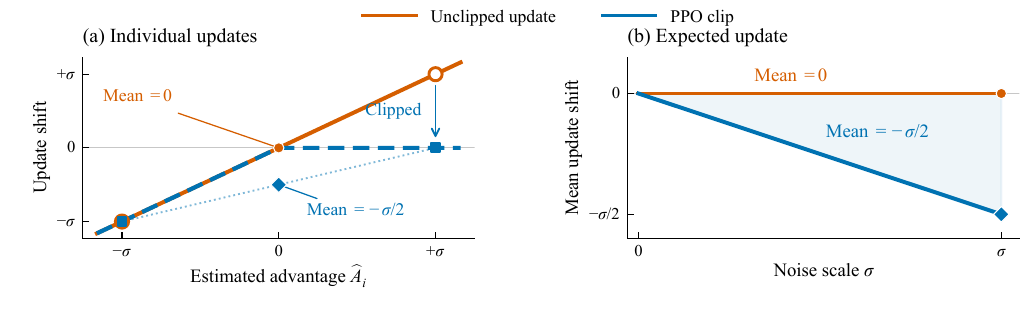}
 \caption{At fixed $\rho_i>1+\epsilon_{\mathrm{clip}}$, the unperturbed advantage is zero
 and $\widehat A_m=\pm\sigma$ with equal probability. (a) PPO clipping
 removes positive update shifts. (b) The mean shift is $-\sigma/2$ with
 clipping and zero without it; clipping also reduces conditional coefficient variance.
 The common importance-weighted policy score is factored out.}
 \label{fig:clipping-gradient}
\end{figure}

PPO clipping treats advantage noise differently according to the proximal ratio and
advantage sign \citep{schulman2017ppo,hilton2022batch}. Above the upper ratio bound, it
removes positive score-function contributions and can change the expected
update. For the symmetric noise example in
Figure~\ref{fig:clipping-gradient}, this reduces conditional coefficient variance but
shifts the mean. In the ideal unclipped, token-weighted update, a detached history-only
baseline cancels as long as the additive KL correction is unchanged
(Appendix~\ref{app:baseline-identities}).

\textbf{CISPO.} This update \citep{minimax2025m1} clips the detached
current-to-behavior importance weight without future KL correction.
With lower and upper bounds $c_{\mathrm{lo}}$ and $c_{\mathrm{hi}}$, its ascent surrogate is:
\begin{equation}
 \mathcal L_{\mathrm{CISPO}}(\theta)
 =\mathbb E_{\pi_{\mathrm{behav}}}\sum_i
 \operatorname{sg}\!\left[
 \operatorname{clip}\!\left(\frac{\pi_\theta}{\pi_{\mathrm{behav}}},c_{\mathrm{lo}},c_{\mathrm{hi}}\right)
 \widehat A_m\right]\log\pi_\theta(a_i\mid h_i).
 \label{eq:experiment-cispo}
\end{equation}
Both bounds apply to the detached weight while retaining the
score-function contribution.

\textbf{PG.} PG uses reward advantages with current-to-behavior importance
weighting and no future KL regularization. Its surrogate is equivalent
to the GRPO surrogate without PPO clipping:
\begin{equation}
 \mathcal L_{\mathrm{PG}}(\theta)
 =\mathbb E_{\pi_{\mathrm{behav}}}\sum_i
 \frac{\pi_\theta}{\pi_{\mathrm{behav}}}
 \operatorname{sg}[\widehat A_m].
 \label{eq:experiment-reward-only}
\end{equation}
Removing PPO clipping eliminates the bias introduced by its gate, but the
reward-gradient estimator still omits the regularization term in $\nabla J_k$
(Appendix~\ref{app:bias-reference}). Under learner-policy sampling, future KL
credit supplies that term when using raw rewards and a history-only baseline.
Token-wise importance weighting introduces a separate off-policy approximation, and large weights
can amplify stochastic gradient contributions. Adding future KL credit to
this update gives the \ours{} surrogate in Equation~\eqref{eq:experiment-proximal}.

\paragraph{Alternative importance weights.}
We compare the token-wise ratio in Equation~\eqref{eq:experiment-proximal}
with two alternatives that retain each token's corrected advantage.
Section~\ref{sec:ablation} evaluates how these choices affect training
stability and held-out accuracy.

\textit{Sequence-level IS.} This control replaces the token-wise ratio by its
geometric mean over the response. Following the ascent-surrogate convention
above, its update is
\begin{equation}
 \mathcal L_{\text{sequence-level IS}}(\theta)
 =\mathbb E_{\pi_{\mathrm{behav}}}\sum_i
 \operatorname{sg}\!\left[
 \left(\frac{\pi_\theta(a_{m,1:T_m}\mid x)}
 {\pi_{\mathrm{behav},m}(a_{m,1:T_m}\mid x)}\right)^{1/T_m}
 (\widehat A_m-C_i^{(H)})\right]
 \log\pi_\theta(a_i\mid h_i).
 \label{eq:experiment-sequence}
\end{equation}
The $1/T_m$ power forms the geometric mean over response $m$'s
$T_m$ action tokens, following GSPO-token \citep{zheng2025gspo}.

\textit{IS-free.} This control removes importance weighting, giving
\begin{equation}
 \mathcal L_{\text{IS-free}}(\theta)
 =\mathbb E_{\pi_{\mathrm{behav}}}\sum_i
 \operatorname{sg}[\widehat A_m-C_i^{(H)}]
 \log\pi_\theta(a_i\mid h_i).
 \label{eq:experiment-nois}
\end{equation}
As above, response indices on token quantities
and the common sampled-action-token normalization are suppressed;
training minimizes the negative surrogate.
The detached coefficient multiplies each token's log probability,
so each token retains its own future KL credit in both controls.
Neither uses PPO clipping or a separate proximal-to-behavior factor.
Appendix~\ref{app:gspo-proximal} discusses the corresponding token-averaged updates.

\FloatBarrier
\section{Theoretical Guarantees}
\label{sec:theory}
\label{sec:motivation}

Full future KL credit gives token-wise importance weighting a common proximal target,
even when the responses come from stale policies. We first state this mechanism for
population updates. We then summarize a finite-sample reward guarantee for a
projected group-gradient variant; Appendix~\ref{ext:shared} gives its full specification
and proof.

The key observation is that, \textbf{at the proximal target, the reward minus the full future
KL cost depends only on the current history.} The sampled continuation no longer changes
this coefficient. Current-token importance weighting then makes its expected score
contribution zero, even when later tokens were generated by an older policy. The first
result shows that, under the stated conditions, this cancellation also identifies the
target uniquely.

\begin{theorem}[Token-wise importance weighting preserves the proximal target]
\label{thm:prox-recursion}
Fix $\pi_k$ and assume the finite-support, behavior-coverage,
policy-expressiveness, and population-update conditions (A1) through (A4)
in Appendix~\ref{app:prox-assumptions}. Let $M_{\mathrm{behav}}(\theta)$
be the expected token-wise off-policy policy-gradient field
in Equation~\eqref{eq:main-single-token-surrogate}. For any stale behavior
policy, or mixture of behavior-policy versions with their logged probabilities,
\[
 M_{\mathrm{behav}}(\theta)=0
 \quad\Longleftrightarrow\quad
 \pi_\theta=\pi_k^*.
\]
Thus the token-wise update has the same unique stationary policy as the
proximal objective, although its expected update can differ from
$\nabla J_k(\theta)$ at other policies.
\end{theorem}
Here local expressiveness means that parameter changes span every independent token
log-odds direction. The proof and precise assumptions are in
Appendix~\ref{app:async-theory}. Local expressiveness lets us translate a zero
parameter update into a condition at each history, while behavior coverage ensures
that every relevant history contributes to that condition.

Different behavior policies can therefore change the update direction without changing
its stationary policy. Convergence additionally requires progress toward that target.
For independent tabular logits and history visit probabilities uniformly bounded away
from zero, global trajectories of the continuous-time population update converge
under time-varying behavior policies (Theorem~\ref{lem:prox-pg-flow}). This establishes convergence toward a fixed
reference's target when the expected update is available exactly. To account for
sampled responses and repeated reference refreshes, we next study a linear logistic
model with bounded fixed features. A frozen-feature linear output head is one example;
applying a linearized model to a nonlinear Transformer additionally requires control
of approximation errors. In this setting, the finite-sample result below shows how
repeatedly fitting and refreshing the reference can improve the final policy's reward.

\paragraph{Analyzed variant.}
Each round holds the reference fixed and uses fresh groups of responses to fit its
proximal target. The update combines token-wise importance weights with full future
KL credit, centers the reward-minus-KL residual within each group, and projects each
step onto a parameter ball. A randomly selected iterate becomes the next reference.
Appendix~\ref{ext:shared} specifies this analyzed variant.
Table~\ref{tab:theory-practice} summarizes how the population update
underlying Theorem~\ref{thm:prox-recursion} is adapted in
Algorithm~\ref{alg:proximal-kl}.

\begin{theorem}[Finite-sample reward guarantee, informal]
\label{ext:reward-informal}
Consider the analyzed variant in a linear logistic model with bounded fixed features,
deterministic binary rewards, and responses of at most $T$ tokens. At every reference,
assume its Gibbs target is exactly represented by a parameter within Euclidean
distance $B/\beta$ of the reference parameter, where $B\ge1$ is uniform across
rounds. Assume uniformly controlled token ratios and a sufficiently
small aggregate token-wise behavior mismatch throughout training, with precise conditions specified
in Theorem~\ref{ext:reward} in Appendix~\ref{ext:shared}. Let $\pi_\star$ be a comparator whose correct-response
probability at each prompt is at most $C_\star$ times that of the initial policy.
For any $0<\varepsilon<1$, suitable algorithm parameters and training budgets ensure
that the policy $\pi_K$ obtained after $K$ reference refreshes satisfies
\[
 \E J(\pi_K)\ge J(\pi_\star)-\varepsilon,
 \qquad J(\pi)=\E_x\Pr_\pi(R=1\mid x).
\]
For a fixed group size, the number of scored training responses is polynomial in
$T$, the representation bound $B$, $C_\star$, and $1/\varepsilon$.
The expectation includes training randomness and output selection.
\end{theorem}

\begin{remark}[Aggregate token-wise log-ratio second moment]
\label{rem:token-mismatch-example}
At history $h_i$, let $p_i=\pi_\theta(\cdot\mid h_i)$ and
$q_i=\pi_{\mathrm{behav}}(\cdot\mid h_i)$ be the learner and behavior token
distributions, and write $w_i(a)=p_i(a)/q_i(a)$.
Theorem~\ref{ext:reward} controls the sum of learner-weighted second moments
$\E_{p_i}[(\log w_i)^2]$, taking a uniform bound on their sum along each response path and over
training updates. It denotes this aggregate by $V_{\log}$.
For example, if $|\log w_i|\le a/\sqrt T$ uniformly, for a fixed
constant $a>0$ sufficiently small to meet the theorem's conditions, then
$V_{\log}\le a^2$ independently of response length. In contrast, the
sequence ratio $\prod_{i=1}^T w_i$ can grow exponentially in $\sqrt T$:
multiplying the tokenwise bounds gives $\prod_{i=1}^T w_i\le e^{a\sqrt T}$.
Higher moments of history likelihood ratios sharpen the dependence on accuracy.
The argument applies to any real moment order at least two. For each fixed
$\delta\in(0,1]$, under the theorem's $\delta$-dependent token-ratio condition,
the accuracy dependence is
$1/\varepsilon$ plus $c_\delta/\varepsilon^{1+\delta}$, apart from the
problem and round-count factors displayed there, where
$c_\delta=\exp((1+\delta^{-1})V_{\log})$.
Smaller $\delta$ can require tighter token-ratio control and increase $c_\delta$.
On policy, $V_{\log}=0$ and the direct history-transfer bound yields
$1/\varepsilon$ dependence, apart from the logarithmic round-count factors.
\end{remark}

The guarantee shows that repeated proximal updates can turn initial correct-solution
coverage into a policy whose single-response success probability approaches the
comparator's. Each round approximately fits the current reference's proximal target;
refreshing the reference then allows further reward improvement.
Appendix~\ref{ext:representation-examples} gives examples of the representation bound;
Remark~\ref{ext:coverage-specializations} gives the uniform-coverage case.

\paragraph{Single-response reward compared with best-of-$N$.}
The comparator can also describe a procedure that samples several responses at
inference time. This gives a concrete interpretation of the single-response guarantee.
Let $p_0(x)=\Pr_{\pi_0}(R=1\mid x)$. Drawing $N$ responses from the initial
policy and selecting a correct one whenever available gives oracle success probability
$1-(1-p_0(x))^N\le Np_0(x)$. Thus the guarantee applies with $C_\star=N$:
one response from the trained policy approaches the reward of oracle best-of-$N$
selection at initialization, under the theorem's conditions.
The training budget and the oracle's inference budget are accounted for separately.

These results explain the roles of the two policies in the practical algorithm.
For implementation, the reference determines the improvement target, while the
logged behavior policy determines the importance correction. Full future KL credit
preserves the target in the analyzed population update. Our experiments examine the
stability and performance trade-offs introduced by bounded costs and finite windows.

\section{Empirical Experiments}
\label{sec:experiments}

We evaluate \ours{} in two off-policy training settings. Periodic
rollout-policy refresh tests stability across model sizes. Training with controlled rollout staleness combines asynchronous generation
with response reuse across multiple optimizer minibatches. These settings
test stable optimization and held-out reasoning accuracy when the behavior
policy differs from the policy being optimized.

\subsection{Setup}
\label{sec:experiment-setup}

All compared methods omit the KL penalty to the initial model. The
reward-advantage normalization follows the training setting below and
is shared across methods within each setting.

\paragraph{Working Default.}
We recommend $H=\mathrm{full}$, $\beta=0.30$, and $\epsilon_F=0.05$
as a general starting configuration for \ours{}. This choice is shared
across all four model sizes in the periodic-refresh experiments. A shorter
horizon provides a model-specific adjustment; for 4B, we suggest $H=32$ as a
starting point with the same coefficient and cost bound. The ablations
examine how horizon tuning can further improve accuracy. The actual
configurations used for the reported results are listed in
Appendices~\ref{app:periodic-refresh} and~\ref{app:controlled-lag-hparams}.

\paragraph{Training with Periodic Policy Refresh.}
\label{sec:simulation}
At each training iteration, we collect a fresh batch of responses and use
that batch once for a single optimizer update. The training-policy weights
therefore change every iteration, whereas the rollout-policy weights are
synchronized with the training policy only once every 64 optimizer updates
and remain fixed between refreshes. Successive batches can thus be generated
by the same older policy while training progresses. This provides a controlled
simulation of delayed policy synchronization in asynchronous RL, where
response generation uses earlier policy versions while optimization continues.
We train Qwen3 models \citep{yang2025qwen3} on verifiable math tasks
using Slime \citep{zhu2025slime}, comparing PG, CISPO, GRPO, and \ours{}
across model sizes under this shared refresh schedule.
Reward advantages subtract the mean reward within each prompt group,
without standard-deviation normalization.
Figure~\ref{fig:simulation} reports training accuracy, defined in
Appendix~\ref{app:periodic-refresh-protocol}; all configurations are in
Appendix~\ref{app:periodic-refresh}.

\paragraph{Training with Controlled Rollout Staleness.}
\label{sec:controlled-lag-setup}
Following AReaL \citep{fu2025areal}, we combine asynchronous rollout
generation with multiple optimizer minibatches per batch on filtered math
prompts. We explicitly control the permitted rollout staleness and verify
the realized policy-version gaps from training logs.
Reward advantages standardize rewards within each prompt group, as
defined in Equation~\eqref{eq:group-normalized-advantage}.
We compare GRPO and \ours{}, adding \rewardonly{} for the stability
test, and include a separate comparison with fresh rollouts and one
optimizer minibatch. PPO clipping is disabled in that separate comparison
because it is designed to discourage large policy changes when reusing rollout
data, whereas this on-policy comparison has zero staleness and uses each batch
for only one update.
Held-out evaluation covers MATH500
\citep{lightman2023verify}, OlympiadBench \citep{he2024olympiadbench},
MinervaMath \citep{lewkowycz2022minerva}, Gaokao2023MathEn, AIME24/25
\citep{maa2024aime}, and AMC23/24. MATH500, OlympiadBench, MinervaMath,
and Gaokao2023MathEn use greedy pass@1. AIME24/25 and AMC23/24 use
mean@16, the fraction of correct responses among 16 attempts per problem.
The learning curves also report pass@16, the fraction of problems solved by
at least one of those attempts. Appendix~\ref{app:controlled-lag} gives the
reporting details; Appendix~\ref{app:controlled-lag-hparams} specifies the
data, training configurations, and evaluation budgets.

\subsection{Stability Test}
\label{sec:simulation-mismatch}
\label{sec:update-dynamics}

\begin{figure}[!ht]
 \centering
 \includegraphics[width=\linewidth]{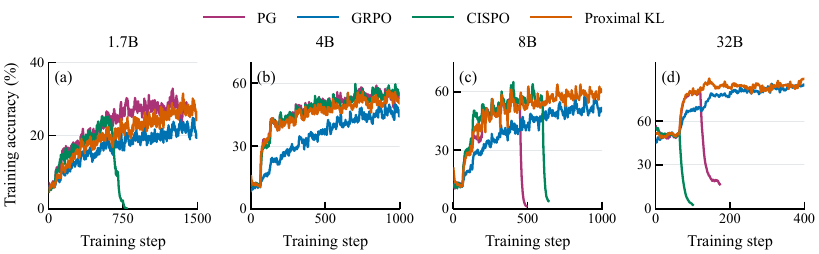}
 \caption{Periodic-refresh training with FP8 rollouts and BF16 optimization
 at $1\times10^{-6}$, with a rollout refresh interval of 64 optimizer updates.
 Panels (a) through (d): 1.7B, 4B, 8B, and 32B.
 All panels use BF16 KV.
 EMA $\alpha=0.10$. Configurations: Appendix~\ref{app:periodic-refresh}.}
 \label{fig:simulation}
\end{figure}

\textbf{Proximal KL sustains training gains when PG and CISPO collapse.}
Figure~\ref{fig:simulation} distinguishes early reward gains from stable
optimization. At 1.7B, CISPO initially improves alongside the other methods
before its training accuracy falls sharply. At 8B and 32B, both PG and CISPO
undergo large accuracy declines, despite substantial accuracy earlier in
training. \ours{} retains its gains across all four displayed model sizes
using the same working default. All four methods apply importance weights,
so importance correction alone
does not ensure stable optimization. The comparison instead supports combining
token-wise weighting with proximal KL feedback as model scale increases.

\begin{figure}[!ht]
 \centering
 \includegraphics[width=\linewidth]{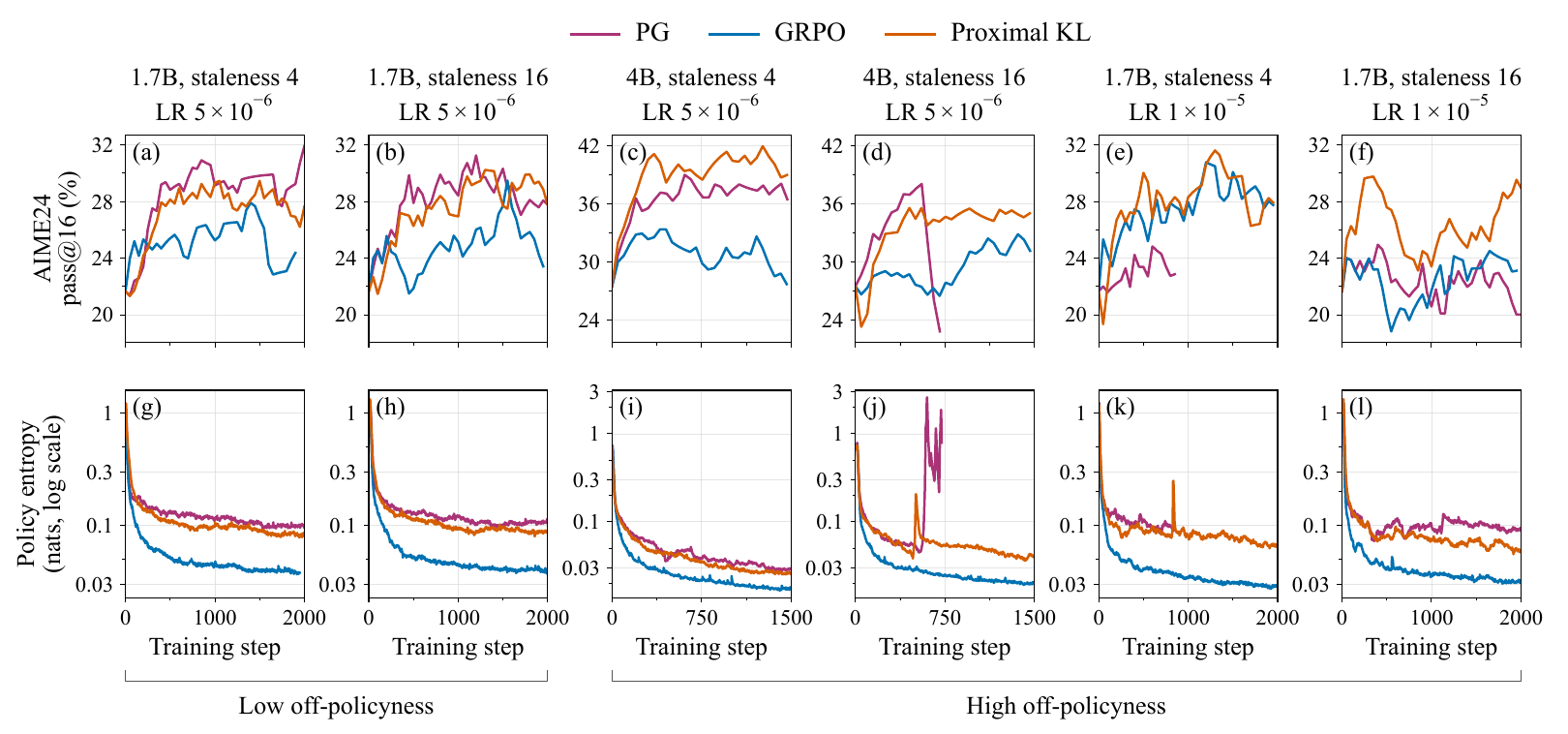}
 \caption{Stability across six settings. Top: AIME24 pass@16; bottom:
 policy entropy on logarithmic axes. The first four columns match
 Figure~\ref{fig:main-multiseed}; the last two use 1.7B, $1\times10^{-5}$,
 at staleness levels 4 and 16. EMA $\alpha=0.20$.
 Bottom groups distinguish the smaller-model/lower-rate settings from
 larger-model or higher-rate settings; each column specifies its rollout staleness.
 Display limits and configurations: Appendix~\ref{app:controlled-lag-hparams}.}
 \label{fig:large-lr-stability}
\end{figure}

Training with controlled staleness reveals a similar limitation of
\rewardonly{}: its competitive performance at 1.7B and the lower learning
rate does not carry over reliably to larger models or larger learning
rates. The 4B run with staleness 16 suffers a sharp AIME24 pass@16 decline with
large entropy excursions (Figure~\ref{fig:large-lr-stability}(d,j));
the higher-learning-rate runs show limited sustained gains and a later
decline at staleness 16 (panels e,f). In contrast, GRPO and \ours{} avoid
comparable collapses over their displayed trajectories in both settings.
We also ablate GRPO's clipping width: increasing
$\epsilon_{\mathrm{clip}}$ from $0.2$ to $0.4$ produces larger gradient
excursions and a temporary decline in training accuracy
(Appendix~\ref{app:grpo-clip-width}).

\textbf{Proximal KL improves training accuracy with fewer updates than GRPO.}
GRPO avoids the collapses of PG and CISPO, but its training accuracy rises
more slowly than that of \ours{} in Figure~\ref{fig:simulation}.
The difference is visible across model sizes: \ours{} moves ahead during
training at 1.7B, 4B, and 8B, and reaches the high-accuracy region earlier at
32B. For example, the 8B EMA-smoothed training accuracy first reaches $50\%$
after about 200 updates with \ours{}, compared with about 640 for GRPO.
Thus, proximal KL feedback improves update efficiency as well as stability
compared with GRPO;
removing ordinary PPO clipping need not sacrifice controlled progress.
The higher mean training accuracy over the common final 100-update window
at 1.7B, 4B, and 8B also shows that the advantage persists beyond the initial rise.

Under controlled rollout staleness, \ours{} also
achieves higher late AIME24 pass@16 in the four $5\times10^{-6}$ settings
spanning both model sizes and staleness levels (Figure~\ref{fig:large-lr-stability}, panels (a) through (d)).
At the larger learning rate $1\times10^{-5}$, the checkpoint evaluations
in Table~\ref{tab:main-results} also favor \ours{} in Average and on MATH500,
OlympiadBench, MinervaMath, and Gaokao greedy pass@1 at both staleness levels.
These results show higher training accuracy under periodic policy refresh
and improved held-out accuracy in the comparisons with controlled staleness. Given the stability
limitations of PG and CISPO and our computational budget, we concentrate
the benchmark comparison in Section~\ref{sec:main-results} on GRPO and \ours{}.

\subsection{Held-out Set Evaluation}
\label{sec:main-results}

\begin{table}[!ht]
\centering\small
\setlength{\tabcolsep}{2pt}
\caption{Raw checkpoint accuracy (\%). 1.7B: mean $\pm$ sample SD over three seeds. Average pools 4,149 responses per seed across eight benchmarks and reports the mean across seeds; Bold: higher mean in each column within a matched comparison; both bold if the unrounded gap is below $0.50\%$. Protocols are in Appendix~\ref{app:controlled-lag}.}
\label{tab:main-results}
\resizebox{\linewidth}{!}{\begin{tabular}{llll lllllllll}\toprule
Setting & Stale- & LR & Method & MATH500 & Olympiad & Minerva & Gaokao & AIME24 & AIME25 & AMC23 & AMC24 & Average \\
& ness & & & pass@1 & pass@1 & pass@1 & pass@1 & mean@16 & mean@16 & mean@16 & mean@16 & \\\midrule[0.8pt]
\multicolumn{13}{@{}l}{\textit{Model: Qwen3-1.7B}} \\
\midrule[0.8pt]
- & - & - & Base & 56.10 & 23.15 & 22.43 & 51.17 & 4.27 & 2.92 & 27.66 & 13.47 & 24.15 \\
\midrule
On-policy & 0 & \texttt{5e-6} & GRPO & \textbf{68.93}\,{\scriptsize$\pm$\,1.03} & \textbf{32.84}\,{\scriptsize$\pm$\,1.28} & 28.92\,{\scriptsize$\pm$\,0.77} & \textbf{59.53}\,{\scriptsize$\pm$\,0.45} & 8.89\,{\scriptsize$\pm$\,0.96} & \textbf{5.97}\,{\scriptsize$\pm$\,0.94} & \textbf{46.30}\,{\scriptsize$\pm$\,2.17} & \textbf{26.53}\,{\scriptsize$\pm$\,0.24} & \textbf{34.50} \\
On-policy & 0 & \texttt{5e-6} & \ours{} & 67.80\,{\scriptsize$\pm$\,1.20} & \textbf{33.14}\,{\scriptsize$\pm$\,0.67} & \textbf{29.90}\,{\scriptsize$\pm$\,1.06} & 58.40\,{\scriptsize$\pm$\,0.92} & \textbf{10.07}\,{\scriptsize$\pm$\,1.20} & \textbf{5.62}\,{\scriptsize$\pm$\,0.75} & 45.10\,{\scriptsize$\pm$\,0.39} & 25.42\,{\scriptsize$\pm$\,0.28} & \textbf{34.09} \\
\midrule
Off-policy & 0 & \texttt{5e-6} & GRPO & 65.40\,{\scriptsize$\pm$\,2.46} & 28.24\,{\scriptsize$\pm$\,0.62} & 28.68\,{\scriptsize$\pm$\,2.05} & \textbf{58.05}\,{\scriptsize$\pm$\,1.06} & 8.96\,{\scriptsize$\pm$\,1.57} & \textbf{6.74}\,{\scriptsize$\pm$\,0.43} & 42.24\,{\scriptsize$\pm$\,2.26} & 22.78\,{\scriptsize$\pm$\,0.97} & 31.99 \\
Off-policy & 0 & \texttt{5e-6} & \ours{} & \textbf{67.20}\,{\scriptsize$\pm$\,1.00} & \textbf{29.67}\,{\scriptsize$\pm$\,1.04} & \textbf{29.41}\,{\scriptsize$\pm$\,0.74} & \textbf{58.40}\,{\scriptsize$\pm$\,0.84} & \textbf{9.72}\,{\scriptsize$\pm$\,0.43} & 5.69\,{\scriptsize$\pm$\,1.39} & \textbf{44.84}\,{\scriptsize$\pm$\,2.86} & \textbf{24.17}\,{\scriptsize$\pm$\,1.60} & \textbf{33.13} \\
\midrule
Off-policy & 4 & \texttt{5e-6} & GRPO & 64.80\,{\scriptsize$\pm$\,1.11} & \textbf{28.98}\,{\scriptsize$\pm$\,0.94} & \textbf{29.17}\,{\scriptsize$\pm$\,1.18} & 56.74\,{\scriptsize$\pm$\,0.80} & 6.88\,{\scriptsize$\pm$\,0.55} & \textbf{5.83}\,{\scriptsize$\pm$\,1.08} & 44.06\,{\scriptsize$\pm$\,1.08} & \textbf{24.21}\,{\scriptsize$\pm$\,2.16} & 32.14 \\
Off-policy & 4 & \texttt{5e-6} & \ours{} & \textbf{68.00}\,{\scriptsize$\pm$\,0.20} & \textbf{29.48}\,{\scriptsize$\pm$\,1.63} & \textbf{28.92}\,{\scriptsize$\pm$\,0.77} & \textbf{57.27}\,{\scriptsize$\pm$\,0.60} & \textbf{8.68}\,{\scriptsize$\pm$\,0.60} & 5.07\,{\scriptsize$\pm$\,0.43} & \textbf{45.68}\,{\scriptsize$\pm$\,1.57} & \textbf{24.17}\,{\scriptsize$\pm$\,0.37} & \textbf{33.00} \\
\midrule
Off-policy & 16 & \texttt{5e-6} & GRPO & 66.13\,{\scriptsize$\pm$\,1.51} & \textbf{29.43}\,{\scriptsize$\pm$\,1.35} & 28.31\,{\scriptsize$\pm$\,0.37} & 56.22\,{\scriptsize$\pm$\,0.75} & 7.57\,{\scriptsize$\pm$\,1.56} & \textbf{6.88}\,{\scriptsize$\pm$\,0.83} & \textbf{45.05}\,{\scriptsize$\pm$\,1.33} & 23.01\,{\scriptsize$\pm$\,1.47} & 32.41 \\
Off-policy & 16 & \texttt{5e-6} & \ours{} & \textbf{68.60}\,{\scriptsize$\pm$\,1.74} & \textbf{29.57}\,{\scriptsize$\pm$\,0.23} & \textbf{29.78}\,{\scriptsize$\pm$\,0.37} & \textbf{57.18}\,{\scriptsize$\pm$\,0.45} & \textbf{9.51}\,{\scriptsize$\pm$\,1.18} & \textbf{6.39}\,{\scriptsize$\pm$\,1.20} & \textbf{44.64}\,{\scriptsize$\pm$\,1.73} & \textbf{23.94}\,{\scriptsize$\pm$\,0.29} & \textbf{33.18} \\
\midrule[0.8pt]
Off-policy & 4 & \texttt{1e-5} & GRPO & 58.67\,{\scriptsize$\pm$\,3.06} & 24.73\,{\scriptsize$\pm$\,1.77} & 22.30\,{\scriptsize$\pm$\,1.66} & 52.48\,{\scriptsize$\pm$\,0.90} & \textbf{11.60}\,{\scriptsize$\pm$\,0.43} & \textbf{5.28}\,{\scriptsize$\pm$\,2.84} & 47.66\,{\scriptsize$\pm$\,4.32} & 21.30\,{\scriptsize$\pm$\,1.95} & 30.39 \\
Off-policy & 4 & \texttt{1e-5} & \ours{} & \textbf{64.33}\,{\scriptsize$\pm$\,2.39} & \textbf{29.08}\,{\scriptsize$\pm$\,2.53} & \textbf{24.88}\,{\scriptsize$\pm$\,3.72} & \textbf{55.96}\,{\scriptsize$\pm$\,4.05} & 9.65\,{\scriptsize$\pm$\,1.15} & 4.03\,{\scriptsize$\pm$\,0.64} & \textbf{52.29}\,{\scriptsize$\pm$\,2.32} & \textbf{23.33}\,{\scriptsize$\pm$\,1.60} & \textbf{32.97} \\
\midrule
Off-policy & 16 & \texttt{1e-5} & GRPO & 61.20\,{\scriptsize$\pm$\,2.11} & 25.27\,{\scriptsize$\pm$\,1.12} & 25.25\,{\scriptsize$\pm$\,0.93} & 53.52\,{\scriptsize$\pm$\,1.20} & 7.50\,{\scriptsize$\pm$\,1.50} & \textbf{5.07}\,{\scriptsize$\pm$\,1.88} & \textbf{47.19}\,{\scriptsize$\pm$\,3.12} & \textbf{23.75}\,{\scriptsize$\pm$\,2.41} & \textbf{30.93} \\
Off-policy & 16 & \texttt{1e-5} & \ours{} & \textbf{63.73}\,{\scriptsize$\pm$\,0.61} & \textbf{27.35}\,{\scriptsize$\pm$\,0.37} & \textbf{27.21}\,{\scriptsize$\pm$\,2.24} & \textbf{55.00}\,{\scriptsize$\pm$\,0.66} & \textbf{9.10}\,{\scriptsize$\pm$\,1.67} & 3.19\,{\scriptsize$\pm$\,1.20} & 45.63\,{\scriptsize$\pm$\,2.98} & 20.37\,{\scriptsize$\pm$\,1.60} & \textbf{30.98} \\
\midrule[0.8pt]
\multicolumn{13}{@{}l}{\textit{Model: Qwen3-4B}} \\
\midrule[0.8pt]
- & - & - & Base & 63.75 & 35.61 & 25.74 & 56.92 & 8.59 & 8.33 & 41.25 & 30.42 & 34.01 \\
\midrule
Off-policy & 4 & \texttt{5e-6} & GRPO & 81.20 & 43.92 & 39.71 & \textbf{71.54} & 15.83 & \textbf{10.21} & \textbf{79.22} & 33.47 & 47.17 \\
Off-policy & 4 & \texttt{5e-6} & \ours{} & \textbf{82.40} & \textbf{44.66} & \textbf{40.44} & \textbf{71.28} & \textbf{18.33} & \textbf{10.63} & 74.84 & \textbf{39.31} & \textbf{48.13} \\
\midrule
Off-policy & 16 & \texttt{5e-6} & GRPO & 79.60 & 43.32 & 39.34 & 69.45 & 15.42 & \textbf{11.04} & \textbf{70.63} & 37.78 & 46.13 \\
Off-policy & 16 & \texttt{5e-6} & \ours{} & \textbf{81.00} & \textbf{45.99} & \textbf{40.44} & \textbf{72.32} & \textbf{21.04} & 10.21 & \textbf{70.31} & \textbf{42.22} & \textbf{48.35} \\
\bottomrule
\end{tabular}}
\end{table}

\textbf{Proximal KL improves overall reasoning accuracy in every matched off-policy comparison.}
Table~\ref{tab:main-results} reports higher Average for \ours{} than GRPO
in all seven off-policy method pairs. Average pools correct responses across
eight benchmarks, measuring response accuracy over the complete evaluation suite.
The gain holds across both model sizes, both tested learning rates, and all tested
rollout staleness levels, reaching an absolute improvement of $2.58\%$.
Thus, the stable training gains translate into improved held-out accuracy
across the tested off-policy settings.

The accuracy advantage also holds when response reuse is the source of
off-policy updates: with fresh rollouts reused across four minibatches,
\ours{} achieves higher Average. In the separate on-policy comparison,
which uses fresh rollouts and one optimizer minibatch, the two methods
attain comparable aggregate accuracy. This control assesses the practical
methods under fresh sampling; it is distinct from the idealized proximal
improvement results in Section~\ref{sec:setup}. We therefore characterize
the empirical result as comparable on-policy performance, with consistently
higher aggregate accuracy in the tested off-policy comparisons.
\par

\begin{figure}[ht]
 \centering
 \includegraphics[width=\linewidth]{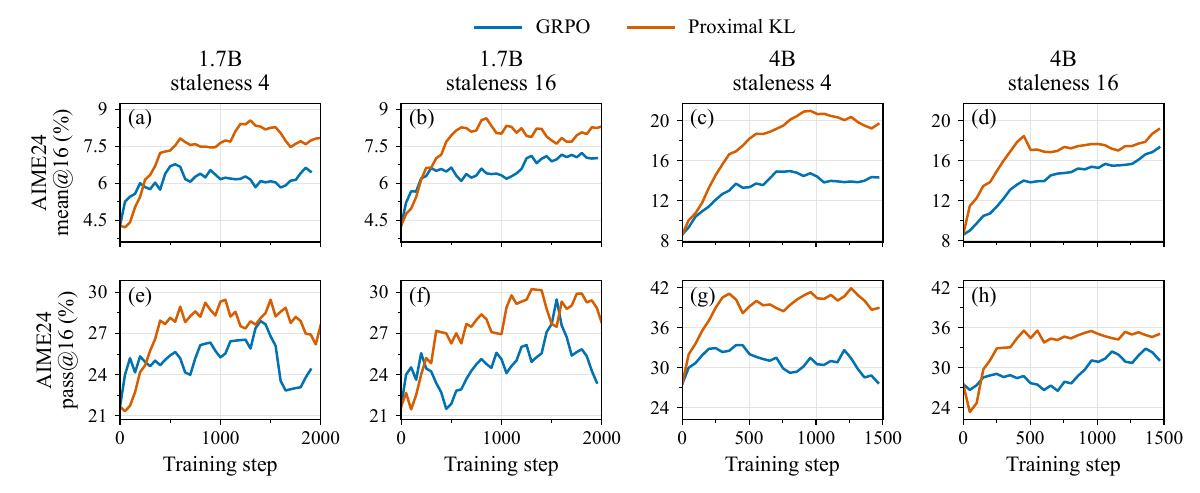}
 \caption{AIME24 at $5\times10^{-6}$. Top: mean@16; bottom: pass@16.
 Columns: 1.7B staleness 4, 1.7B staleness 16, 4B staleness 4, and 4B staleness 16.
 Configurations are given in
 Appendix~\ref{app:controlled-lag-hparams}. EMA $\alpha=0.20$; shared Base within each model
 at step 0.}
 \label{fig:main-multiseed}
\end{figure}

The AIME24 learning curves show that the gain extends to both response
accuracy and prompt coverage. \ours{} reaches higher later mean@16 and
pass@16 across the four settings in Figure~\ref{fig:main-multiseed}.
These trajectories complement the raw checkpoint evaluations: the
advantage is visible during training as well as in held-out checkpoint
comparisons.

\begin{table}[!htb]
\centering\footnotesize
\setlength{\tabcolsep}{3pt}
\caption{Qwen3-4B accuracy changes (\%) from staleness 4 to 16 at $5\times10^{-6}$. Differences use unrounded Table~\ref{tab:main-results} scores. $\uparrow$/$\downarrow$: increase/decrease; $\rightarrow$: unchanged. Bold: larger signed change.}
\label{tab:lag-4b}
\resizebox{\linewidth}{!}{\begin{tabular}{lllllllllll}
\toprule
Staleness & Method & MATH500 & Olympiad & Minerva & Gaokao & AIME24 & AIME25 & AMC23 & AMC24 & Average \\
& & pass@1 & pass@1 & pass@1 & pass@1 & mean@16 & mean@16 & mean@16 & mean@16 & \\
\midrule
$4\rightarrow16$ & GRPO & \textcolor{blue}{$\downarrow\,1.60$} & \textcolor{blue}{$\downarrow\,0.59$} & \textcolor{blue}{$\downarrow\,0.37$} & \textcolor{blue}{$\downarrow\,2.09$} & \textcolor{blue}{$\downarrow\,0.42$} & \textcolor{red}{$\uparrow\,\mathbf{0.83}$} & \textcolor{blue}{$\downarrow\,8.59$} & \textcolor{red}{$\uparrow\,\mathbf{4.31}$} & \textcolor{blue}{$\downarrow\,1.04$} \\
$4\rightarrow16$ & \ours{} & \textcolor{blue}{$\downarrow\,\mathbf{1.40}$} & \textcolor{red}{$\uparrow\,\mathbf{1.34}$} & $\rightarrow\,\mathbf{0.00}$ & \textcolor{red}{$\uparrow\,\mathbf{1.04}$} & \textcolor{red}{$\uparrow\,\mathbf{2.71}$} & \textcolor{blue}{$\downarrow\,0.42$} & \textcolor{blue}{$\downarrow\,\mathbf{4.53}$} & \textcolor{red}{$\uparrow\,2.92$} & \textcolor{red}{$\uparrow\,\mathbf{0.22}$} \\
\bottomrule
\end{tabular}}
\end{table}

\paragraph{Proximal KL retains its advantage in the higher-staleness configuration.}
\label{sec:rollout-lag}
Using the checkpoint scores already reported in Table~\ref{tab:main-results},
we examine how Qwen3-4B accuracy changes between rollout staleness levels 4 and 16.
Both settings use learning rate $5\times10^{-6}$ and four optimizer
minibatches per batch. Table~\ref{tab:lag-4b} reports each method's
score at staleness 16 minus its score at staleness 4; a larger change means better
retention of accuracy. We compare the reported training configurations,
which are specified in Appendix~\ref{app:controlled-lag-hparams}.

Average rises for \ours{} and falls for GRPO as rollout staleness increases.
The benchmark-level changes support this distinction: \ours{} improves
OlympiadBench and Gaokao greedy pass@1 and AIME24 mean@16, while GRPO
declines on these metrics. Across the eight benchmarks, \ours{} has the
more favorable change on six. These comparisons show that \ours{} retains its accuracy advantage over
GRPO in the tested higher-staleness configuration.

\subsection{Ablations}
\label{sec:ablation}

To understand why the working default is effective and how tuning can
improve it, we examine the update choices that sustain its gains. We
first test the need for current-token importance correction and the
sensitivity introduced by PPO clipping. We then examine the bound on
individual future costs and the correction horizon. These ablations
support bounded future KL feedback as a general choice. A finite horizon
can further improve accuracy, and separate fixed-policy measurements
show that truncation reduces correction variance.

\begingroup
\setlength{\intextsep}{5pt}
\begin{figure}[!ht]
 \centering
 \includegraphics[width=\linewidth]{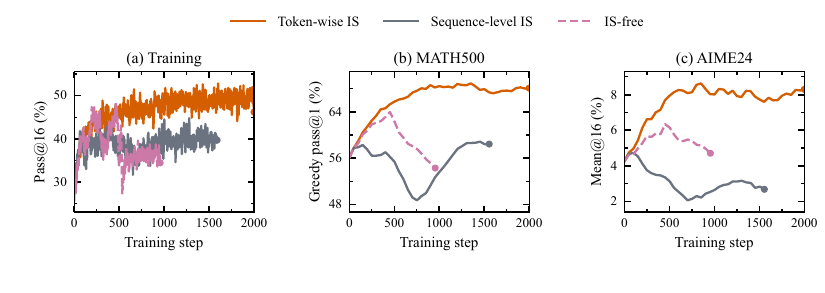}
 \caption{Importance-weighting comparison: 1.7B, $5\times10^{-6}$,
 staleness 16, full tail, $\beta=0.30$. Token-specific future KL credit is
 retained in all three updates. EMA $\alpha=0.20$; shared evaluation
 step 0.}
 \label{fig:sequence-correction}
\end{figure}
\endgroup

\noindent\textbf{Token-wise correction achieves higher held-out accuracy than sequence-level IS and IS-free updates.}
\label{sec:sequence-correction}
Motivated by Theorem~\ref{thm:prox-recursion}, we compare current-token
weighting with sequence-level and unweighted alternatives while retaining
future KL credit. Under the theorem's population assumptions, full future
KL credit lets current-token weighting preserve the proximal stationary
target without an exact response-level ratio product. Removing IS
altogether leaves the expected score under the behavior policy,
so the proximal optimizer need not remain stationary. Figure~\ref{fig:sequence-correction}
tests token-wise weighting against Sequence-level IS and IS-free in Equations~\eqref{eq:experiment-sequence}
and~\eqref{eq:experiment-nois}; all three retain token-specific future
KL credit. The sequence-level control uses the geometric mean of token ratios.

Token-wise IS sustains training pass@16 and achieves stronger MATH500
greedy pass@1 and AIME24 mean@16 than both controls in the displayed
evaluations. Both controls undergo substantial declines in held-out
accuracy, despite retaining the same future KL credit. Thus, reducing
correction to the current token is compatible with strong practical
performance, whereas discarding correction loses that benefit. These
results are consistent with the theoretical motivation for token-wise
correction and support its practical use alongside future KL credit.

\begin{figure}[!ht]
 \centering
 \includegraphics[width=\linewidth]{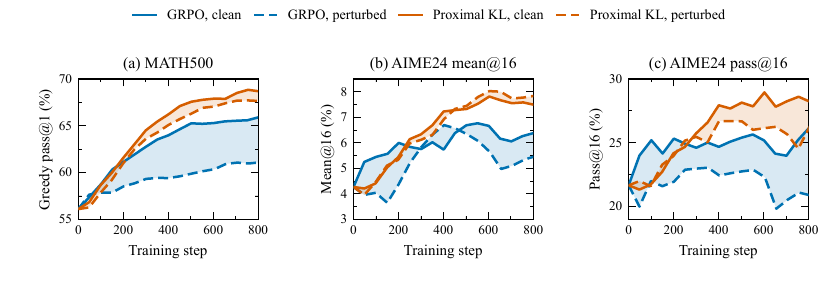}
 \caption{Baseline perturbation at 1.7B, $5\times10^{-6}$, staleness 4. Solid/dashed: unperturbed (clean)/perturbed. Shading shows the score difference between the two curves. EMA $\alpha=0.20$.}
 \label{fig:baseline-shift}
\end{figure}

\paragraph{Proximal KL is less sensitive to the tested baseline perturbation than PPO clipping.}
\label{sec:mechanism}\label{sec:baseline-shift}
To understand why \ours{} achieves higher accuracy than the PPO-clipped
baseline, we investigate whether the two updates respond differently
to errors in the estimated reward baseline. The intervention shifts the common baseline
within prompt groups with nonconstant rewards, preserving verifier rewards and
within-group advantage differences
(Appendix~\ref{app:baseline-shift}). We compare each method's perturbed trajectory with its own
unperturbed trajectory, using the configurations in
Appendix~\ref{app:ablation-hparams}. The separation between each pair of curves measures the
change under this intervention.

\ours{} better preserves performance than GRPO on MATH500 greedy
pass@1 and AIME24 mean@16 and pass@16
(Figure~\ref{fig:baseline-shift}). Comparing each method with its own
clean reference distinguishes sensitivity to the perturbation from the
difference between their original scores. The smaller changes for
\ours{} support using additive future KL feedback in place of an
advantage-dependent clipping gate. This experiment connects the
clipping mechanism discussed in Section~\ref{sec:algorithm} to changes
in held-out accuracy during practical training.

\begin{figure}[!ht]
 \centering
 \includegraphics[width=\linewidth]{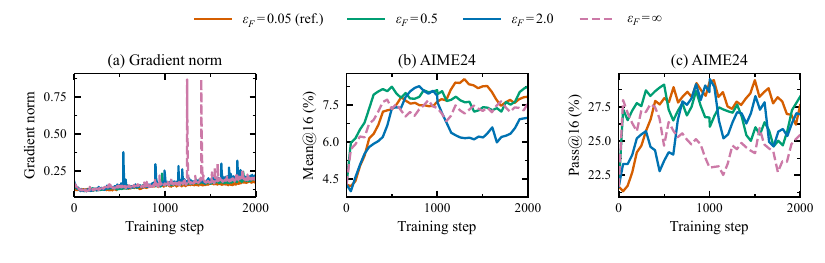}
 \caption{Future-cost clipping: 1.7B, $5\times10^{-6}$, staleness 4,
 full tail, $\beta=0.30$; EMA $0.20$, shared evaluation step 0.
 Gradient norms precede optimizer clipping.}
 \label{fig:epsilon-f}
\end{figure}

\noindent\textbf{\ours{} benefits from bounded future KL costs.}\label{sec:epsilon-f}
Having replaced PPO clipping with future KL feedback, we ask whether
the correction can be used without bounding its individual token costs. Figure~\ref{fig:epsilon-f} varies
$\epsilon_F$ while retaining the full future horizon and the same KL
coefficient. The unbounded setting retains regularization and removes
only the bound on each future cost. We compare AIME24 mean@16 and
pass@16 alongside gradient norms measured before optimizer clipping,
so accuracy and large gradient excursions can be assessed together.

The default bound achieves good AIME24 mean@16 and pass@16, with
smaller pre-clipping gradient-norm excursions than the loosest and
unbounded settings. A moderately larger bound also retains good accuracy,
showing that effective regularization does not require a single precise
clipping threshold. These observations support bounding individual
future costs before they accumulate in the correction. Removing PPO
clipping therefore remains compatible with controlling correction
magnitude through the future-cost bound.

\begin{figure}[!ht]
 \centering
 \includegraphics[width=\linewidth]{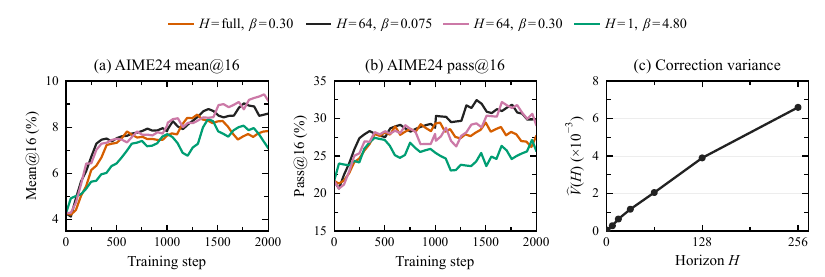}
 \caption{Future KL horizon ablation. (a,b) 1.7B,
 $5\times10^{-6}$, staleness 4; EMA $0.20$.
 $(H,\beta)=(1,4.80)$ and $(64,0.075)$ match $\beta H=4.8$.
 (c) Fixed-policy correction variance.}
 \label{fig:horizon-tradeoff}
\end{figure}

\paragraph{A finite horizon can outperform the full-tail working default.}\label{sec:tail-horizon}
Starting from the full-tail working default, we ask whether truncating
future credit can further improve accuracy. Figure~\ref{fig:horizon-tradeoff}
compares the full tail with intermediate windows and a one-token
correction. The comparison includes both a fixed KL coefficient and
a matched product $\beta H$, allowing us to examine whether increasing
the local coefficient can compensate for removing future terms.

Both intermediate-window settings outperform the full tail in AIME24
mean@16 and pass@16 near the end of their shared observation range,
including with the default KL coefficient held fixed. With $\beta H$
matched, the one-token correction improves mean@16 more slowly than the
intermediate window. Thus, shortening the tail can improve on the
working default, while increasing the local coefficient does not replace
future-token credit. This is the practical trade-off suggested by the
theory: truncation omits part of the ideal future credit, but can reduce
fluctuations in the correction. The fixed-policy diagnostic below measures
this second effect independently of training.

\paragraph{Truncating the future sum reduces correction variance at fixed policies.}
To test the variance argument for horizon truncation in
Section~\ref{sec:practical-bias-variance}, we measure how extending the future
sum changes correction variability at fixed policies. For each
retained response, we accumulate the scalar correction at its first
token, compute its sample variance across responses to the same prompt,
and average these variances across prompts. Policies, responses, the
correction coefficient, and the cost bound are held fixed across
horizons. Appendix~\ref{app:paired-horizon-variance} specifies the
calculation; this paired comparison isolates the effect of extending
the sum on the same data.

Figure~\ref{fig:horizon-tradeoff}(c) shows that, with policies and
responses held fixed, the average variance of the accumulated correction
increases over the tested horizons. Truncating the tail therefore reduces
correction variability in this diagnostic. Longer sums fluctuate more
across responses to the same prompt, on average, even without any change
in the policies. This supports horizon truncation as a way to control
fluctuations in future KL feedback. The measurement quantifies variation in the scalar correction; the
training curves show how horizon choice affects held-out accuracy.

\par
\begingroup
\section{Conclusion}

\ours{} combines token-wise importance weighting with future KL credit for off-policy
RLVR, where the behavior policy differs from the policy being optimized. Under the stated coverage and expressiveness conditions, we prove that the
full-tail critic-free update has the exact proximal optimizer as its unique stationary
policy and that tabular population flow converges to this target. Repeated reference
refreshes allow further reward improvement: for linear logistic models, we give a
finite-sample reward guarantee for a projected group update under the stated
representation, sampling, and correct-solution coverage assumptions. The bound depends
on an aggregate token-wise log-ratio second moment. The practical
algorithm uses clipped costs and an optional finite window. 
The held-out evaluations in Section~\ref{sec:main-results} show higher
micro-averaged accuracy for \ours{} than GRPO in all seven matched off-policy settings,
while the separate on-policy comparison yields comparable aggregate accuracy.
Periodic-refresh experiments
complement these gains with stable training where PG and CISPO can collapse,
and higher training accuracy than GRPO at comparable optimizer-update counts. The
ablations favor retaining current-token correction and future credit together,
while bounding individual costs and truncating the horizon control correction
fluctuations. An adaptive choice of $H$ could improve this balance by retaining
useful future credit without accumulating unnecessary variability.
\par

\par
\endgroup

\section*{Acknowledgments}
We thank Xuanchang Zhang for discussions on parts of the theoretical analysis and Chenlu Ye for discussions on the periodic policy-refresh experiments.

This work was supported in part by NSF grants 2416897, 2505932, and 2601666 and ONR grant N000142512318. Computing resources were provided by the National Center for Supercomputing Applications (NCSA) through Delta (NSF OAC-2005572) and DeltaAI (NSF OAC-2320345), as well as Illinois Computes and the NAIRR Pilot (NAIRR250157).

Generative AI tools were used under the authors' direction for assistance with literature searches, coding, conducting and analyzing experiments, deriving theoretical results and proofs, and polishing the manuscript.

\bibliographystyle{plainnat}
\bibliography{references}
\clearpage
\appendix
\numberwithin{equation}{section}
\numberwithin{theorem}{section}
\allowdisplaybreaks[1]
\section{Practical Update Design}
\label{app:estimators}

\paragraph{Proximal and behavior policies.}
The reference is frozen between its scheduled refreshes, as specified
in Appendix~\ref{app:controlled-lag-hparams}. The stored rollout probabilities identify the
behavior policy that generated each response. In GRPO, the ratio $\rho_i$
measures movement from the proximal policy for PPO clipping.
\ours{} instead uses the current-to-behavior weight
$\pi_\theta/\pi_{\mathrm{behav}}$ directly in
Equation~\eqref{eq:experiment-proximal}. Its proximal reference enters the
future KL correction; the importance weight need not be split into
two ratios.

\paragraph{Bias relative to the proximal objective.}
\label{app:bias-reference}

Fix $\theta$, the reference policy, and the sampling law.
With exact response importance correction, raw rewards, and a
detached history-only baseline, a reward-gradient estimator satisfies
\begin{equation}
 \mathbb E\widehat g_{\mathrm{reward}}-\nabla J_k
 =\beta\nabla_\theta\mathbb E_{x\sim\mathcal D}
 D_{\mathrm{KL}}\!\left(\pi_\theta(\cdot\mid x)
 \Vert\pi_k(\cdot\mid x)\right).
 \label{eq:reward-proximal-bias}
\end{equation}
The difference is exactly the omitted KL gradient, which vanishes at the reference
policy. Full future KL credit supplies this term under exact sampling, and response-wide
importance weighting preserves the equality off policy. Replacing response-wide weights
with token-wise weights generally changes the gradient field. Group normalization and
clipping the future log-ratio costs introduce additional approximations.

\paragraph{Group normalization and KL correction.}
\label{app:practical-loss}

For a group of $G$ responses, define
\begin{equation}
 \overline R=\frac1G\sum_{m=1}^{G}R_m,\qquad
 s_R^2=\frac1{G-1}\sum_{m=1}^{G}(R_m-\overline R)^2,\qquad
 \widehat A_m=\frac{R_m-\overline R}{s_R+1\times10^{-5}}.
 \label{eq:group-normalized-advantage}
\end{equation}
For each token, subtract the KL correction $C_{m,i}^{(H)}$ from its reward advantage
(Equation~\eqref{eq:practical-proximal-kl}). The correction includes the current token
and stops at the response boundary, with no normalization by window length. The window
determines which costs enter the sum, while $\beta$ sets their per-token strength.

\paragraph{Detached update coefficients.}
The loss in Equation~\eqref{eq:implemented-loss} uses the group-normalized
reward advantage and future KL correction defined above. Holding its coefficients
fixed during differentiation gives the score-function ascent direction of
Equation~\eqref{eq:experiment-proximal}. The loss averages over sampled tokens;
the population analysis averages token-score sums over responses.

\begin{table}[!htbp]
\centering
\small
\caption{From the stationary-policy result to the implemented update.}
\label{tab:theory-practice}
\begin{tabular}{@{}p{0.20\linewidth}p{0.25\linewidth}p{0.43\linewidth}@{}}
\toprule
Ideal condition & Implemented choice & Effect on the sampled update \\
\midrule
Full future KL tail & Horizon $H$, including $H=\mathrm{full}$ & Selects how far credit extends; a finite window omits later costs. \\
Unclipped future log ratio & Optional log-ratio clipping & Bounds each future cost and changes its expectation when active; the choice is independent of the sampling policy. \\
History-only baseline & Group-normalized reward advantage & Sets the reward scale from sampled responses; the resulting field can differ from the ideal one. \\
Population response average & Minibatch token average & Adds sampling noise; token normalization can change the relative weights of responses. \\
\bottomrule
\end{tabular}
\end{table}
The horizon controls how far sampled future credit extends, and the log-ratio bound
limits each cost. Both choices can introduce approximation bias. Their effect on total
update variance depends on the covariance between retained and omitted terms. Group
normalization and token averaging can also change the expected update field and its
stationary policy.

\paragraph{Scope of the implementation.}

The implementation uses shared Transformer parameters and Adam with
the estimator choices listed in Table~\ref{tab:theory-practice}.
Section~\ref{sec:motivation} explains their relation to the stationary-policy
guarantee, and Section~\ref{sec:ablation} evaluates clipping and
future-credit horizon. Appendix~\ref{app:controlled-lag-hparams} gives the settings.

\paragraph{Sequence-level IS and IS-free controls.}
\label{app:sequence-correction}\label{app:gspo-proximal}
Figure~\ref{fig:sequence-correction} compares the updates in
Equations~\eqref{eq:experiment-sequence} and~\eqref{eq:experiment-nois}
with token-wise weighting. All three retain token-specific future KL
credit and average the loss over nonpadding action tokens.
Sequence-level IS uses the geometric mean of the current-to-behavior
token ratios. Detaching this weight and multiplying each token's log
probability by its corrected advantage gives the GSPO-token gradient
\citep{zheng2025gspo}. The IS-free control sets the weight to one.
The geometric mean differs from the full-response ratio product;
both controls retain the cost bound $\epsilon_F$ and omit PPO clipping.

\FloatBarrier

\section{Training with Periodic Policy Refresh}
\label{app:periodic-refresh-experiments}

This appendix collects the training protocol, model and precision settings,
and method-specific hyperparameters for the periodic-refresh experiments.

\subsection{Training Protocol}
\label{app:periodic-refresh-protocol}
Each iteration collects a fresh response batch and uses it for one optimizer
update. Rollout weights are synchronized with the training policy every 64
optimizer updates and remain fixed between refreshes, so successive batches
can come from the same older policy while the training policy changes.
The key configurations are given in Appendix~\ref{app:periodic-refresh}.
Correct boxed answers receive reward $+1$; incorrect or invalid answers receive $-1$.
Training accuracy is the fraction of correct responses, computed as
$(1+\overline R)/2$, where $\overline R$ is the mean verifier reward before centering.
For the comparison between \ours{} and GRPO, we average raw training accuracy over the last
100 updates observed for both methods within each setting, restricting this
summary to settings with at least 100 updates for each method.

\subsection{Hyperparameters}
\label{app:periodic-refresh}
We train Qwen3-1.7B-Base, 4B-Base, 8B-Base, and Qwen3-32B on DAPO-Math-17K
\citep{yang2025qwen3,yu2025dapo} using Slime \citep{zhu2025slime}. We use a post-trained checkpoint for 32B because a corresponding
pre-trained checkpoint was unavailable to us. The smaller models use base checkpoints. Rollouts use FP8 W/A and
BF16 KV, while optimization uses BF16. Table~\ref{tab:periodic-refresh-hparams}
summarizes the key choices for the training curves in Figure~\ref{fig:simulation}.

All methods use Adam, with gradient-norm clipping at 1.
Advantages subtract the group mean without standard-deviation normalization,
and losses average over response tokens. PG and \ours{} use current-token
current-to-behavior importance weights; Section~\ref{sec:algorithm} defines
all four update rules. The curves use EMA $\alpha=0.10$.

\begin{table}[!ht]
\centering\small
\setlength{\tabcolsep}{3pt}
\renewcommand{\arraystretch}{1.08}
\caption{Key hyperparameters for training with periodic policy refresh.
Merged cells denote shared settings; refresh intervals count optimizer updates.}
\label{tab:periodic-refresh-hparams}
\begin{tabular*}{\linewidth}{@{\extracolsep{\fill}}llcccc@{}}\toprule
Method & Setting & \multicolumn{4}{c}{Qwen3} \\
\cmidrule(lr){3-6}
 & Model size & 1.7B & 4B & 8B & 32B \\
\midrule
Shared & Learning rate & \multicolumn{4}{c}{$1\times10^{-6}$} \\
 & Rollout refresh interval & \multicolumn{4}{c}{64} \\
 & Prompts $\times$ responses per batch & \multicolumn{4}{c}{$32\times8$} \\
 & Optimizer updates per batch & \multicolumn{4}{c}{1} \\
 & Prompt / response limits (tokens) & \multicolumn{4}{c}{1,024 / 16,384} \\
 & Sampling temperature & \multicolumn{4}{c}{1.0} \\
 & Displayed updates (maximum) & 1,500 & \multicolumn{2}{c}{1,000} & 400 \\
\midrule
GRPO & PPO clip (proximal reference) & \multicolumn{4}{c}{Yes ($\epsilon_{\mathrm{clip}}=0.2$)} \\
 & Reference refresh interval $L$ & \multicolumn{4}{c}{32} \\
\midrule
\ours{} & Future horizon $H$ & \multicolumn{4}{c}{Full} \\
 & Coefficient $\beta$ & \multicolumn{4}{c}{0.30} \\
 & Future log-ratio bound $\epsilon_F$ & \multicolumn{4}{c}{0.05} \\
 & Reference refresh interval $L$ & \multicolumn{4}{c}{16} \\
 & PPO clip & \multicolumn{4}{c}{No} \\
\midrule
PG & PPO clip & \multicolumn{4}{c}{No} \\
\midrule
CISPO & Importance-weight clipping & \multicolumn{4}{c}{$[0,5]$} \\
\bottomrule\end{tabular*}
\end{table}

\section{Training with Controlled Rollout Staleness}
\label{app:controlled-lag-experiments}
\label{app:main-results}

This appendix collects the training and evaluation protocols, hyperparameters,
and ablation diagnostics for the experiments with controlled staleness.
Appendix~\ref{app:estimators} describes the practical updates, and
Appendix~\ref{app:async-theory} gives the theoretical target guarantee.

\subsection{Training and Evaluation Protocols}
\label{app:training-settings}
\label{app:controlled-lag}
Rollout generation overlaps with policy optimization, and collected batches are reused
across successive optimizer minibatches. We control rollout staleness and
verify the realized policy-version gaps from training logs. The model,
data, optimization, reference-policy, and evaluation settings are given in
Appendix~\ref{app:controlled-lag-hparams}.

Figure~\ref{fig:main-multiseed} reports EMA-smoothed online evaluation curves,
whereas Table~\ref{tab:main-results} reports separate raw checkpoint evaluations.
Checkpoint scores use neither smoothing nor performance-based checkpoint selection.
The stability test reports held-out accuracy together with policy entropy;
the gradient diagnostics measure norms before optimizer clipping.
Appendix~\ref{app:figure-hparams} gives the shared settings and evaluation budgets.

\paragraph{Evaluation metrics.}
Mean@16 measures correctness per response, while pass@16 measures the fraction of
prompts solved at least once. For independent attempts with per-prompt success
probability $p$, expected pass@16 is $1-(1-p)^{16}$. Because this transformation is
nonlinear, the two metrics describe different aspects of performance.
The figures show both; Tables~\ref{tab:main-results} and~\ref{tab:lag-4b}
report response accuracy only. MATH500, OlympiadBench, MinervaMath, and
Gaokao2023MathEn use greedy pass@1 on 500, 674, 272, and 383 problems.
AIME24, AIME25, AMC23, and AMC24 use mean@16 on 30, 30, 40, and 45
problems, respectively, with 16 responses per problem.
Base rows report evaluations of the corresponding base models.
Average in Table~\ref{tab:main-results} is micro accuracy: the total
number of correct generated responses across the eight benchmarks divided
by 4,149 (500, 674, 272, 383, 480, 480, 640, and 720 responses, respectively).
Each generated response contributes once, without counting pass@16 as an
additional observation.

For three-seed rows, we compute this average within each seed and
report only its mean across seeds.

\subsection{Hyperparameters}
\label{app:hyperparameters}
\label{app:controlled-lag-hparams}
\label{app:figure-hparams}

\begin{table}[!ht]
\centering\small
\setlength{\tabcolsep}{3pt}
\renewcommand{\arraystretch}{1.08}
\caption{Hyperparameters for training with controlled staleness. Columns specify controlled rollout staleness 4; the comparison at staleness 16 uses the corresponding setting. At 1.7B, the $1\times10^{-5}$ \ours{} run with staleness 16 uses $H=64$, $\beta=0.60$. At 4B, staleness 4 uses a 32-token unit-weight prefix and a linearly decayed remainder; staleness 16 uses $H=8$, with $\beta=0.25$ in both cases. PG shares GRPO's batch, sampling, optimizer, reward-normalization, and staleness settings. Merged cells denote shared values.}
\label{tab:run-coverage}
\label{tab:hyperparameters}
\begin{tabular*}{\linewidth}{@{\extracolsep{\fill}}llccc@{}}\toprule
Method & Setting & \multicolumn{3}{c}{Off-policy} \\
\cmidrule(lr){3-5}
 & Model & \multicolumn{2}{c}{Qwen3-1.7B} & Qwen3-4B \\
 & Learning rate & $5\times10^{-6}$ & $1\times10^{-5}$ & $5\times10^{-6}$ \\
\midrule
GRPO & PPO clip & \multicolumn{3}{c}{Yes} \\
 & Prompts per batch & \multicolumn{3}{c}{128} \\
 & Responses per prompt & \multicolumn{3}{c}{16} \\
 & Response limit (tokens) & \multicolumn{3}{c}{8,192} \\
 & Sampling temperature & \multicolumn{3}{c}{1.0} \\
 & Optimizer & \multicolumn{3}{c}{Adam} \\
 & Gradient-norm limit & \multicolumn{3}{c}{1.0} \\
 & Minibatches per batch & \multicolumn{3}{c}{4} \\
 & Reward normalization & \multicolumn{3}{c}{Within each prompt group} \\
 & Rollout staleness & \multicolumn{3}{c}{4} \\
 & Evaluation budget (updates) & \multicolumn{2}{c}{2,000} & $\approx 1{,}500$ \\
\midrule
PG & PPO clip & \multicolumn{3}{c}{No} \\
 & Coefficient $\beta$ & \multicolumn{3}{c}{0} \\
 & Minibatches per batch & \multicolumn{3}{c}{4} \\
 & Batch, sampling, and optimizer & \multicolumn{3}{c}{As GRPO above} \\
\midrule
\ours{} & Future horizon $H$ & \multicolumn{2}{c}{Full} & 32 + linear tail \\
 & Coefficient $\beta$ & \multicolumn{2}{c}{0.30} & 0.25 \\
 & PPO clip & \multicolumn{3}{c}{No} \\
 & Future log-ratio bound $\epsilon_F$ & \multicolumn{3}{c}{0.05} \\
 & Reference refresh (global batches) & \multicolumn{3}{c}{1} \\
 & Prompts per batch & \multicolumn{3}{c}{128} \\
 & Responses per prompt & \multicolumn{3}{c}{16} \\
 & Response limit (tokens) & \multicolumn{3}{c}{8,192} \\
 & Sampling temperature & \multicolumn{3}{c}{1.0} \\
 & Optimizer & \multicolumn{3}{c}{Adam} \\
 & Gradient-norm limit & \multicolumn{3}{c}{1.0} \\
 & Minibatches per batch & \multicolumn{3}{c}{4} \\
 & Reward normalization & \multicolumn{3}{c}{Within each prompt group} \\
 & Rollout staleness & \multicolumn{3}{c}{4} \\
 & Evaluation budget (updates) & \multicolumn{2}{c}{2,000} & $\approx 1{,}500$ \\
\bottomrule\end{tabular*}
\end{table}

\paragraph{Hyperparameter Selection.}
We selected the final configurations for \ours{} through a small number
of manual adjustments under a limited tuning budget. Starting from a
working configuration, we inspected early learning curves and adjusted
hyperparameters when the improvement was insufficient. Each final
configuration was obtained after a few such adjustments and retained
for stable training. The horizon and clipping ablations separately
examine how these choices affect optimization.

We train Qwen3-1.7B-Base and Qwen3-4B-Base on 9,453 prompts derived from
FRPO's filtered DAPO-Math data \citep{yao2026futurekl,yu2025dapo}, using
AReaL \citep{fu2025areal}. Table~\ref{tab:hyperparameters} gives the shared
batch, sampling, and optimizer settings. Adam uses
$(\beta_1,\beta_2)=(0.9,0.999)$, $\epsilon_{\mathrm{Adam}}=1\times10^{-8}$, a constant
learning rate, and no weight decay. Off-policy batches use four
optimizer minibatches. Reward advantages use the group normalization in
Equation~\eqref{eq:group-normalized-advantage}. The off-policy GRPO control
uses symmetric PPO clipping with width $\epsilon_{\mathrm{clip}}=0.2$. At nonzero
staleness, the proximal reference is recomputed before each global batch
and frozen for its minibatches.

To improve training efficiency, all 4B experiments in this controlled-staleness
setting use dynamic sampling \citep{yu2025dapo}. We discard prompt groups
with identical verifier rewards and collect replacement groups to fill the
training batch, using the same filtering rule for all compared methods.

\paragraph{Staleness and reference settings.}
The asynchronous comparisons control rollout staleness at 4 or 16.
For the 1.7B GRPO comparison at staleness 4, the recorded mean staleness
is approximately 4. The separate comparisons with zero staleness use
either one or four optimizer minibatches.
Their \ours{} runs use the full tail, $\beta=0.03$, no future-cost
clipping, and a reference refreshed every 50 updates. The one-minibatch
controls omit PPO clipping.

\paragraph{Evaluation.}
Table~\ref{tab:main-results} reports three-seed means for 1.7B at
2,000 updates and 4B scores at approximately 1,500 updates.
The performance and stability curves use EMA $\alpha=0.20$.

\FloatBarrier

\subsection{Ablation Configurations}
\label{app:ablation-hparams}

This section gives the perturbation construction and variance estimator
used in Section~\ref{sec:ablation}.

\par\medskip\phantomsection\noindent\textbf{Baseline Perturbation.}\label{app:baseline-shift} An underestimated baseline makes every response in a group appear better, without changing their relative ranking. We test how GRPO and \ours{} respond by adding one to each normalized advantage in groups with nonconstant rewards:
\begin{equation}
 \widehat A'_{g,m}=\widehat A_{g,m}+I_g,\qquad
 I_g=\mathbf 1\{\max_m R_{g,m}>\min_m R_{g,m}\}.
 \label{eq:group-baseline-shift}
\end{equation}
Verifier rewards, reward normalization, and within-group advantage
differences remain unchanged. For \ours{}, future KL credit is computed
as usual and subtracted after the shift.

\paragraph{\normalfont\textit{Why clipping matters.}}
\label{app:baseline-identities}
PPO clipping reacts to the advantage sign: above the upper ratio bound,
positive advantages contribute no gradient, while negative ones still do.
A common upward shift can therefore switch some contributions off.
For detached advantages and a frozen positive behavior correction, the
score coefficient is
$\frac{\pi_\theta}{\pi_{\mathrm{behav}}}q_{\mathrm{clip}}(\rho_i,\widehat A_m)$,
where, away from the clipping boundaries,
\begin{equation}
 q_{\mathrm{clip}}(\rho,A)=\begin{cases}
 \max(A,0),&\rho<0.8,\\
 A,&0.8<\rho<1.2,\\
 \min(A,0),&\rho>1.2.
 \end{cases}
 \label{eq:ppo-gate}
\end{equation}
Above the upper bound, symmetric noise $\pm\sigma$ has zero mean before
clipping. Removing its positive half gives $q_{\mathrm{clip}}$ mean
$-\sigma/2$ and reduces its variance from $\sigma^2$ to $\sigma^2/4$
(Figure~\ref{fig:clipping-gradient}). These coefficient moments factor out
the common importance-weighted policy score.

Without clipping, a history-only baseline expresses no preference among
the possible actions and cancels in expectation: their probabilities sum
to one, whose gradient is zero. For a detached shift $b(h)$ and positive
sampling support,
\begin{equation}
 \mathbb E_{a\sim\pi_{\mathrm{behav}}(\cdot\mid h)}
 \left[\frac{\pi_\theta(a\mid h)}{\pi_{\mathrm{behav}}(a\mid h)}
 b(h)\nabla_\theta\log\pi_\theta(a\mid h)\right]
 =b(h)\nabla_\theta\sum_a\pi_\theta(a\mid h)=0.
 \label{eq:baseline-cancellation}
\end{equation}
An additive future KL correction unchanged by the shift preserves this
expected cancellation; the clipping gate need not.

This contrast motivates the experiment. The group-shared shift depends
on sampled rewards through $I_g$, so sensitivity is measured directly:
each method's perturbed run is compared with its unperturbed run to
assess how well it preserves held-out accuracy
(Section~\ref{sec:baseline-shift}).

\par\medskip\phantomsection\noindent\textbf{Correction Variance at Fixed Policies.}\label{app:paired-horizon-variance}\label{app:variance-hparams} We measure how the total correction varies across responses to the same prompt as $H$ increases. We reuse the same tokenized responses and two frozen Qwen3-1.7B checkpoints, which supply current and reference probabilities. This diagnostic pair differs from the original within-minibatch pair. We use 16 prompts with 16 responses each, $\beta=0.30$, and $\epsilon_F=0.05$.

Each response $y_{pm}$ to prompt $x_p$ gives one scalar: its first-token
correction, summing costs through $H$ or its final token $T_{pm}$,
\begin{equation}
 C_{pm}^{(H)} := C_1^{(H)}(x_p,y_{pm})
 = \beta\sum_{j=1}^{\min(H,T_{pm})}
 \operatorname{clip}(\log\rho_j,-\epsilon_F,\epsilon_F),
 \label{eq:diagnostic-response-correction}
\end{equation}
where $\rho_j$ is the current-to-reference token-probability ratio.
The sum is not divided by $H$ or response length.

We compute the sample variance of these numbers within each prompt,
then average equally across prompts. This excludes differences between prompt
means. With $P$ prompts and $M$ responses per prompt,
\begin{equation}
 \begin{aligned}
  \overline C_p^{(H)} &:= \frac{1}{M}\sum_{m=1}^{M}C_{pm}^{(H)}, \\
  \widehat V(H) &:= \frac{1}{P}\sum_{p=1}^{P}
  \left[\frac{1}{M-1}\sum_{m=1}^{M}
  \bigl(C_{pm}^{(H)}-\overline C_p^{(H)}\bigr)^2\right].
 \end{aligned}
 \label{eq:diagnostic-correction-variance}
\end{equation}
Larger $\widehat V(H)$ means correction sums are more spread out within
prompts, on average. The measured quantity is the scalar correction sum defined above.

Figure~\ref{fig:horizon-tradeoff}(c) plots the measured $\widehat V(H)$
in units of $1\times10^{-3}$, without smoothing.

\subsection{Learning-Rate Selection}
\label{app:learning-rate}
\label{app:dapo-lr}
\label{app:other-lr}

\begin{figure}[H]
 \centering
 \includegraphics[width=\linewidth]{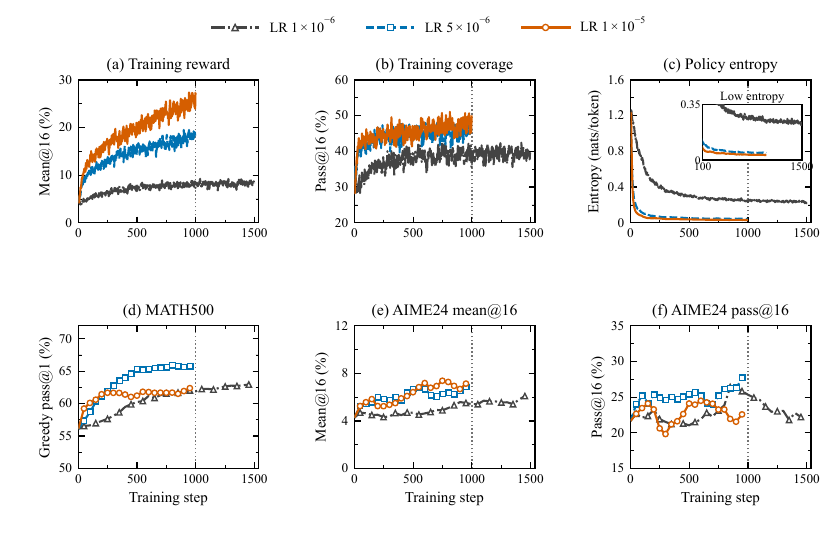}
 \caption{Qwen3-1.7B GRPO learning-rate sweep, staleness 4.
 Dotted line: 1,000 updates; EMA $\alpha=0.20$.}
 \label{fig:dapo-lr}
\end{figure}

The aggregate accuracy advantage of \ours{} over GRPO holds at both
learning rates: its micro-averaged accuracy is higher in every matched off-policy
$5\times10^{-6}$ and $1\times10^{-5}$ comparison in
Table~\ref{tab:main-results}. We choose $5\times10^{-6}$ for the main comparison because both methods achieve
higher greedy pass@1 on all four greedy-evaluated benchmarks at staleness levels 4 and 16
(Table~\ref{tab:main-results}). The Qwen3-1.7B rows at $1\times10^{-5}$ in
Table~\ref{tab:main-results} are discussed in the stability test
(Section~\ref{sec:update-dynamics}). The 1.7B GRPO sweep provides further context
(Figure~\ref{fig:dapo-lr}): averaging the displayed EMA-smoothed observations
at steps 650 to 999, $5\times10^{-6}$ gives the highest MATH500 greedy pass@1
and AIME24 pass@16, while AIME24 mean@16 favors $1\times10^{-5}$.

\subsection{GRPO Clipping Width}
\label{app:grpo-clip-width}

\begin{figure}[!ht]
 \centering
 \includegraphics[width=\linewidth]{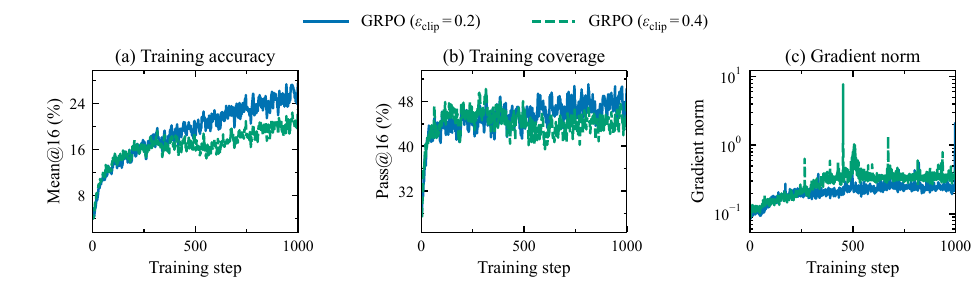}
 \caption{GRPO clipping-width comparison: Qwen3-1.7B, $1\times10^{-5}$,
 staleness setting 4. Accuracy uses EMA $\alpha=0.20$; gradient norms are raw
 and shown on a logarithmic axis.}
 \label{fig:grpo-clip-width}
\end{figure}

We ablate GRPO's PPO clipping width to assess the training stability of
the baseline. Increasing $\epsilon_{\mathrm{clip}}$ from $0.2$ to $0.4$
expands the interval from $[0.8,1.2]$ to $[0.6,1.4]$, allowing larger
changes in the policy ratio before clipping. Both configurations use
four optimizer minibatches and the same reward normalization.

Wider clipping leads to less stable training in this comparison
(Figure~\ref{fig:grpo-clip-width}). With $\epsilon_{\mathrm{clip}}=0.4$,
larger gradient-norm excursions accompany a temporary decline in training
mean@16 and pass@16. The narrower interval maintains a steadier increase
in training mean@16 with smaller gradient excursions. These observations
support retaining $\epsilon_{\mathrm{clip}}=0.2$ for the reported GRPO baseline.

\clearpage
\section{Proximal target and future KL gradient}
\label{app:proximal-theory}

For a fixed full-support reference $\pi_k$ and $\beta>0$, the proximal
objective has an explicit response-level maximizer. We recall it and its
objective-gap identity for completeness.

\begin{proposition}[Exact proximal improvement]
\label{thm:prox-improvement}
\label{thm:gibbs-complete}
For $\beta>0$ and a full-support reference $\pi_k$, the unrestricted
proximal maximizer is
\begin{equation}
 \pi_k^*(\tau\mid x)
 =\frac{\pi_k(\tau\mid x)e^{R(x,\tau)/\beta}}{Z_k(x)},\qquad
 Z_k(x)=\mathbb E_{\tau\sim\pi_k(\cdot\mid x)}e^{R(x,\tau)/\beta}.
 \label{eq:appendix-gibbs}
\end{equation}
For binary rewards, let $p_k(x)=\Pr_{\pi_k}(R=1\mid x)$ and
$p_k^*(x)=\Pr_{\pi_k^*}(R=1\mid x)$.
For every prompt with $0<p_k(x)<1$,
\begin{equation}
 p_k^*(x)=\frac{e^{1/\beta}p_k(x)}{1-p_k(x)+e^{1/\beta}p_k(x)}>p_k(x),
 \qquad
 \frac{p_k^*(x)}{1-p_k^*(x)}
 =e^{1/\beta}\frac{p_k(x)}{1-p_k(x)}.
 \label{eq:promptwise-improvement}
\end{equation}
The probability formula extends to $p_k(x)=0$ and $p_k(x)=1$ by continuity.
For positive initial success probability, repeated exact updates with fixed
$\beta$ make the success probability converge to one.
\end{proposition}

We write $J_k(\pi)$ for the proximal objective evaluated at a policy, so
$J_k(\theta)=J_k(\pi_\theta)$. At the Gibbs target its value is
$J_k(\pi_k^*)=\beta\E_x\log Z_k(x)$.
For any response policy $\pi$, the corresponding objective-gap identity is
\begin{equation}
 J_k(\pi_k^*)-J_k(\pi)
 =\beta\mathbb E_xD_{\mathrm{KL}}\!\left(
 \pi(\cdot\mid x)\Vert\pi_k^*(\cdot\mid x)\right).
 \label{eq:prox-gap}
\end{equation}
\begin{proof}[Proof of Proposition~\ref{thm:prox-improvement}]
Substituting the Gibbs form into the KL divergence gives the
objective-gap identity. Its right-hand side is nonnegative and
vanishes only at $\pi=\pi_k^*$, proving maximality. For binary
rewards, conditioning the Gibbs form on success gives
Equation~\eqref{eq:promptwise-improvement}. Each exact reference refresh
adds $1/\beta$ to the success log-odds, so repeated exact updates send
the success probability to one.
\end{proof}

\subsection{Illustration of proximal feedback with shared parameters}
\label{app:shared-head-separation}

Proximal and trust-region methods are established tools for controlling policy changes
\citep{schulman2015trpo,tomar2020mdpo}. The following example illustrates how proximal
feedback can limit interference between prompts that share parameters. It compares
ordinary Euclidean reward-gradient flow with repeated exact proximal solves in the same
policy class, providing context for the on-policy experiment.

\paragraph{Construction.}
Take two equally likely prompts, a single action $a\in\{0,1\}$, and
reward $R=a$. Let $p_m=\pi_\theta(1\mid x_m)=\sigma(u_m^\top\theta)$,
where $\sigma$ is the logistic function, and set
\[
 \begin{gathered}
 u_1=(1,1)^\top/5,\qquad u_2=(1,-4)^\top/5,\qquad
 \theta_0=(-\Delta,\Delta)^\top,\\
 \Delta=\log\frac{1-p_{\rm init}}{p_{\rm init}},\qquad
 0<p_{\rm init}\le 1/64.
 \end{gathered}
\]
Initially $(p_1,p_2)=(1/2,p_{\rm init})$. The shared head can solve
both prompts: along $\theta=(c,0)^\top$, both success probabilities
approach one as $c\to\infty$.

\begin{proposition}[Reward-gradient flow versus exact proximal updates]
\label{prop:shared-head-separation}
From this initialization, Euclidean gradient flow for
$J(\theta)=(p_1+p_2)/2$ converges in policy space to $(p_1,p_2)=(1,0)$
and never achieves $J>1/2$. Repeated exact proximal updates with any
fixed $\beta>0$ are representable by the same head and converge to
$(p_1,p_2)=(1,1)$.
\end{proposition}

\begingroup
\begin{proof}
Write $z_m=u_m^\top\theta$ and $d_m=p_m(1-p_m)$. Reward-gradient
flow satisfies
\[
 \dot z_1=(2d_1-3d_2)/50,\qquad
 \dot z_2=(-3d_1+17d_2)/50.
\]
In the region $z_1\ge0$, $z_1+z_2\le-\Delta$, we have
$d_1\ge e^{-z_1}/4$ and
$d_2\le e^{-z_1}p_{\rm init}/(1-p_{\rm init})$, hence
$d_2/d_1\le4/63<1/14$.
Thus $\dot z_1>0$ and $\dot z_1+\dot z_2=(-d_1+14d_2)/50<0$.
These strict inequalities hold on the boundary as well, so the flow points into
the region, which contains the initialization.
Invariance of this region implies $p_2\le1-p_1$ and $J\le1/2$.
If the increasing $z_1$ were bounded, $d_1$ and consequently
$\dot z_1$ would be bounded below by positive constants, a contradiction.
Therefore $z_1\to\infty$ and $z_2\to-\infty$.

By Proposition~\ref{thm:prox-improvement}, an exact proximal update
adds $1/\beta$ to each success logit. The parameter increment
$(5,0)^\top/\beta$ realizes both changes simultaneously and has norm
$5/\beta$. Thus this model satisfies the local target representation condition of Theorem~\ref{ext:reward}
at every reference when $B\ge5$. Repeating
these exact solves $K$ times, with $p_{K,m}$ denoting the success
probability on prompt $x_m$, gives
\[
 (p_{K,1},p_{K,2})
 =\bigl(\sigma(K/\beta),\sigma(-\Delta+K/\beta)\bigr)
 \longrightarrow(1,1).
\]
\end{proof}
\par
\endgroup

\paragraph{Proximal feedback.}
For a fixed reference with logits $z_{k,m}$, differentiation gives
\[
 \nabla J_k(\theta)=\frac12\sum_{m=1}^2
 u_m p_m(1-p_m)\{1-\beta(z_m-z_{k,m})\}.
\]
Because $u_1^\top u_2$ is negative, improving one prompt can interfere with the other.
The factor in braces vanishes when a prompt's logit increment reaches
$1/\beta$. At the exact proximal solution, both prompts have this
increment, realized simultaneously by the parameter change
$(5,0)^\top/\beta$. Repeated exact solves therefore improve both prompts,
whereas the reward-gradient flow drives the second prompt's success probability
to zero.

\subsection{The Gradient of Proximal Improvement}
\label{app:future-derivation}

FRPO identifies the regularization credit assigned to each generated token. Because an
earlier token affects later prefixes, its policy-gradient coefficient includes the
subsequent regularization costs. We apply this identity with a linear reward utility and
a reference held fixed throughout one proximal round.

For $h_i=(x,a_{<i})$, define
$\rho_i=\pi_\theta(a_i\mid h_i)/\pi_k(a_i\mid h_i)$,
$\ell_i=\log\rho_i$, and
$\psi_i=\nabla_\theta\log\pi_\theta(a_i\mid h_i)$.
These are the reference ratio, its log cost, and the token score, respectively.
The following derivation gives the future-KL identity directly; it is also the
linear-reward specialization of the Future-KL identity in FRPO
\citep[Theorem~4.1, Equations~(4.2) through (4.5)]{yao2026futurekl}.
The trajectory score is $\sum_i\psi_i$. Differentiating the proximal objective
in Equation~\eqref{eq:mech-objective} therefore gives
\[
 \nabla_\theta J_k(\theta)
 =\E_{\pi_\theta}\left[
    \left(\sum_i\psi_i\right)\left(R-\beta\sum_j\ell_j\right)
    -\beta\sum_j\psi_j\right].
\]
The last term has expectation zero because
$\E_{\pi_\theta}[\psi_j\mid h_j]=0$.
For $j<i$, the earlier cost $\ell_j$ is fixed given $h_i$, so
$\E_{\pi_\theta}[\psi_i\ell_j]=0$ for the same reason.
A history-only baseline $b(h_i)$ also cancels against the conditional score.
Removing these terms leaves the current and future costs:
\begin{align}
 \widehat A_i^{\mathrm{KL}}
 &=R-b(h_i)-\beta\sum_{j=i}^{T}\ell_j,
 \label{eq:full-regularization-term}\\
 \nabla_\theta J_k(\theta)
 &=\mathbb E_{\pi_\theta}\sum_{i=1}^{T}
 \psi_i\widehat A_i^{\mathrm{KL}}.
 \label{eq:mech-future-gradient}
\end{align}

We hold the sampled coefficient fixed when differentiating the loss. Its cost sum
includes the current token and all later tokens, so omitting future costs generally
changes the gradient. FRPO distinguishes assigning these regularization
costs from correcting for stale sampling. We address the sampling correction in
Appendix~\ref{app:async-theory}.

\section{Asynchronous updates with token-wise importance weighting}
\label{app:async-theory}

Asynchronous learning reuses responses from older policies. We study a policy-gradient
(PG) update that uses sampled rewards, full future KL credit, and one importance weight
per sampled token, without fitting a value function. We first characterize its
stationary policy, then prove convergence under stronger assumptions on the population
dynamics.

\subsection{Exact Response Correction}
\label{app:async-correction}

We first state the exact importance correction for a complete response,
then study the stationary target of current-token weighting.

\begin{proposition}[Exact response-level importance correction]
\label{thm:exact-response-correction}
Fix the learner, reference, and behavior policies on the finite
generation tree, with the same prompt distribution and positive token
probabilities. Let $\psi_i=\nabla_\theta\log\pi_\theta(a_i\mid h_i)$,
and use the full future KL advantage from
Equation~\eqref{eq:main-single-token-surrogate} with a history-only
baseline. Then
\begin{equation}
 \begin{gathered}
 \frac{\pi_\theta(\tau\mid x)}{\pi_{\mathrm{behav}}(\tau\mid x)}
 =\prod_{i=1}^T
 \left[
 \frac{\pi_k(a_i\mid h_i)}
      {\pi_{\mathrm{behav}}(a_i\mid h_i)}\rho_i
 \right],\\
 \nabla_\theta J_k(\theta)
 =\mathbb E_{\pi_{\mathrm{behav}}}
 \left[
 \frac{\pi_\theta(\tau\mid x)}{\pi_{\mathrm{behav}}(\tau\mid x)}
 \sum_i\psi_i\widehat A_i^{\mathrm{KL}}\right].
 \end{gathered}
 \label{eq:exact-offpolicy}
\end{equation}
\end{proposition}
\begin{proof}
The autoregressive factorization gives the response likelihood ratio.
Changing measure from the behavior policy to the learner policy then
recovers the full future KL gradient in
Equation~\eqref{eq:mech-future-gradient}.
\end{proof}

\subsection{Token-Wise Importance Weighting and Its Stationary Target}
\label{app:prox-cancellation}
\label{sec:prox-algorithm}

The proof uses the soft-consistency mechanism of PCL and Trust-PCL
\citep{nachum2017pcl,nachum2018trustpcl}. Steps 1 and 2 show that the future-KL advantage
cancels pathwise at the proximal target. Step 4 converts a zero expected score update
into conditions at each action and uses backward induction to identify the unique
stationary policy, even with stale continuations. The soft values are analytical
quantities used only in the proof; the learner needs terminal rewards and full future
credit. Here ``synchronous'' means $\pi_{\mathrm{behav}}=\pi_\theta$ at every update, a
stronger condition than collecting synchronously and then reusing rollouts. In Step 4,
$\widetilde Q(h,a)$ abbreviates the conditional coefficient $\widetilde Q_{\mathrm{behav}}(h,a)$ defined
there.

\begingroup
Algorithm~\ref{alg:proximal-kl} specifies how responses are collected, behavior
probabilities are logged, and the reference is refreshed in the practical update. The
stationary-policy theorem analyzes the full, unmodified population field in
Equation~\eqref{eq:main-single-token-surrogate}, with a detached history-only baseline.
Its assumptions are separate from those of the finite-sample theorem for the linear
logistic model.
\par
\endgroup

\begingroup
\par
\endgroup

\begingroup
\par\medskip\phantomsection\noindent\textbf{Assumptions for Theorem~\ref{thm:prox-recursion}.}
\label{app:prox-assumptions}
\begin{description}
\item[\textnormal{(A1) Finite supported model.}]
The prompt distribution has finite support with positive probability
on every included prompt. Each prompt has a finite token-generation
tree and a fixed terminal reward. Let $T$ be the maximum response length.
After termination, pad to depth $T$ with deterministic absorbing actions,
whose scores and KL costs are zero. The reference $\pi_k$ is fixed
during the round, and both $\pi_k$ and every learner policy under
consideration assign positive probability to every allowed token.
\item[\textnormal{(A2) Covered behavior sampling.}]
Responses come from a fixed behavior policy or a mixture of behavior-policy versions.
For each response, select one version before generation and hold it fixed throughout
that response. Retain the actual version-specific token probabilities used to sample it. Every nonterminal history in the
generation tree has positive behavior occupancy; all such histories
are relevant because the reference has full support. Every contributing
version assigns positive
probability to each allowed next token at that history.
\item[\textnormal{(A3) Representability and local expressiveness.}]
The learner class contains $\pi_k^*$. Collect all relevant independent
next-token log odds into $z(\theta)$. Its Jacobian $Dz(\theta)$ has full
row rank at every parameter considered, so local parameter changes span
every such log-odds direction. Independent tabular logits satisfy this
condition.
\item[\textnormal{(A4) Ideal population update.}]
The field is the exact expectation of Equation~\eqref{eq:prox-field}
under the behavior sampling law, held fixed during differentiation;
the conditional continuation expectations are exact population quantities. The field uses the full, unmodified
future KL sum and an action-independent, history-only baseline held
fixed during differentiation. No finite-horizon truncation, group normalization,
or finite-batch approximation enters
this field.
\end{description}
\par
\endgroup

We index context depth by the next token position, so
$\operatorname{depth}(h_i)=i$.

The token-wise expected update is
\begin{equation}
 M_{\mathrm{behav}}(\theta)=\mathbb E_{\pi_{\mathrm{behav}}}\sum_{i=1}^T
 \frac{\pi_k(a_i\mid h_i)}
      {\pi_{\mathrm{behav}}(a_i\mid h_i)}\rho_i
 \psi_i\widehat A_i^{\mathrm{KL}}.
 \label{eq:prox-field}
\end{equation}

\begingroup
\begin{proof}[Proof of Theorem~\ref{thm:prox-recursion}]
For a history $h$, define the soft partition function, value, and action value as
follows. The expectation in $Z_k(h)$ averages over the entire continuation from $h$
under the reference:
\begin{equation}
 \begin{gathered}
 Z_k(h)=\mathbb E_{\pi_k(\cdot\mid h)}e^{R/\beta},\qquad
 V_k(h)=\beta\log Z_k(h),\\
 Q_k(h,a)=V_k(ha).
 \end{gathered}
 \label{eq:prox-values}
\end{equation}
At a terminal history $h$, $V_k(h)=R$.

\paragraph{\normalfont\textit{Step 1: obtain the token-level proximal policy.}}
The partition function at context $h$ averages the partition functions of
its children:
\begin{equation}
 Z_k(h)=\sum_a\pi_k(a\mid h)Z_k(ha).
 \label{eq:prox-partition-recursion}
\end{equation}
Conditioning the response-level Gibbs law \eqref{eq:appendix-gibbs} on context
$h$ therefore gives
\begin{equation}
 \pi_k^*(a\mid h)=\pi_k(a\mid h)\frac{Z_k(ha)}{Z_k(h)}.
 \label{eq:prox-conditional-gibbs}
\end{equation}
Taking logs and using $V_k=\beta\log Z_k$ gives the regularized Q-function
identity
\begin{equation}
 \log\frac{\pi_k^*(a\mid h)}{\pi_k(a\mid h)}
 =\frac{Q_k(h,a)-V_k(h)}{\beta}.
 \label{eq:prox-soft-policy}
\end{equation}

\paragraph{\normalfont\textit{Step 2: telescope along a sampled suffix.}}
With $\ell_j=\log\rho_j$, define the future-only target and rewrite
the algorithm's advantage as
\begin{equation}
 \widehat Q_i^{\pi_\theta}=R-\beta\sum_{j>i}\ell_j,
 \qquad
 \widehat A_i^{\mathrm{KL}}
 =\widehat Q_i^{\pi_\theta}-\beta\ell_i-b(h_i).
 \label{eq:prox-signal}
\end{equation}
Apply \eqref{eq:prox-soft-policy} to every token after $i$:
\begin{align}
 \beta\sum_{j>i}\log\frac{\pi_k^*(a_j\mid h_j)}{\pi_k(a_j\mid h_j)}
 &=\sum_{j>i}\{V_k(h_{j+1})-V_k(h_j)\}\notag\\
 &=R-V_k(h_{i+1}).
 \label{eq:prox-tail-telescope}
\end{align}
All intermediate values cancel. Since $h_{i+1}=h_i a_i$, the remaining
value is $V_k(h_{i+1})=Q_k(h_i,a_i)$. Substitution into
\eqref{eq:prox-signal} proves the pathwise identities
\begin{equation}
 \widehat Q_i^{\pi_k^*}=Q_k(h_i,a_i),\qquad
 \widehat A_i^{\mathrm{KL}}(\pi_k^*,\tau)=V_k(h_i)-b(h_i).
 \label{eq:prox-pathwise}
\end{equation}
They hold for every supported suffix, before taking any expectation.

\paragraph{\normalfont\textit{Step 3: measure the error away from the solution.}}
To measure the error for an arbitrary learner $\pi_\theta$, subtract the telescoping
identity for $\pi_k^*$ from the learner's future-KL return. This gives
\begin{equation}
 \widehat Q_i^{\pi_\theta}-Q_k(h_i,a_i)
 =-\beta\sum_{j>i}\log\frac{\pi_\theta(a_j\mid h_j)}{\pi_k^*(a_j\mid h_j)}.
 \label{eq:prox-tail-error}
\end{equation}
The target error at token $i$ is exactly the sum of the downstream
log-policy errors, scaled by $-\beta$. This identity also controls the
perturbation term in the convergence proof below.

\paragraph{\normalfont\textit{Step 4: identify the zero of the expected update.}}
Let $d_{\mathrm{behav}}(h)$ include both prompt probability and behavior
prefix occupancy. Write $\psi_\theta(h,a)=\nabla\log\pi_\theta(a\mid h)$.
To include mixed behavior-policy versions, let $v$ denote the version that generated
a response and let $\omega(v\mid h)$ be its conditional probability given the
observed context, before sampling the next action. For
$i=\operatorname{depth}(h)$, define
\begin{equation}
 \widetilde Q_{\mathrm{behav}}(h,a)=\sum_v\omega(v\mid h)
 \mathbb E_{\pi_{\mathrm{behav},v}}
 [\widehat Q_i^{\pi_\theta}\mid h_i=h,a_i=a].
 \label{eq:prox-mixture-q}
\end{equation}
The coefficient averages over versions conditional on $h$, fixes the current action to
$a$, and samples the remaining tokens from that version. Current-action importance
weighting cancels the version's action probability. The resulting mixture weights are
therefore $\omega(v\mid h)$, rather than weights conditional on both the history and
action. A single behavior policy is the special case with one version. Taking the
expectation of \eqref{eq:prox-field} with each response's logged version-specific
probabilities gives
\begin{equation}
 \begin{aligned}
 M_{\mathrm{behav}}(\theta)=\sum_h d_{\mathrm{behav}}(h)\sum_a
 &\pi_\theta(a\mid h)\psi_\theta(h,a)\\
 &\left[\widetilde Q_{\mathrm{behav}}(h,a)-\beta\log\frac{\pi_\theta(a\mid h)}{\pi_k(a\mid h)}-b(h)\right].
 \end{aligned}
 \label{eq:prox-conditional-field}
\end{equation}
In synchronous training, this field equals $\nabla J_k(\theta)$. With stale responses,
it remains the expected update in \eqref{eq:prox-field} but generally differs from the
gradient of $J_k$. The ratio corrects the current action's distribution, while the
history and later tokens still follow the behavior policy. The sequence-wise correction in Equation~\eqref{eq:exact-offpolicy}
recovers the exact gradient.

For tabular logits
at $h$, define $u(h,a)=\widetilde Q_{\mathrm{behav}}(h,a)-\beta\log(\pi_\theta(a\mid h)/\pi_k(a\mid h))$.
Write $p(a)=\pi_\theta(a\mid h)$. The softmax score
$\nabla_{\boldsymbol z_h}\log p(a)=\boldsymbol e_a-\boldsymbol p$
gives the expected logit block
$d_{\mathrm{behav}}(h)
[\operatorname{diag}(\boldsymbol p)-\boldsymbol p\boldsymbol p^\top]
\boldsymbol u$.
For positive $\boldsymbol p$, the quadratic form of the matrix is
$\operatorname{Var}_{a\sim p}[u(h,a)]$; hence its kernel consists exactly
of constant vectors. Positive behavior occupancy therefore makes the block
zero exactly when $u(h,a)$ is constant over actions. A softmax has a
redundant common-shift direction; passing from its full logit vector to
independent log-odds removes only that direction and leaves this zero
condition unchanged.

At the last token,
$\widetilde Q(h,a)=R=Q_k(h,a)$, so this condition is
equivalent to \eqref{eq:prox-soft-policy}; hence the last-token policy must
be $\pi_k^*$. If all later-token policies equal $\pi_k^*$, Step 2 makes
$\widetilde Q(h,a)=Q_k(h,a)$ at the current history, and the same argument forces its
policy to equal $\pi_k^*$. Backward induction proves that the token-wise
field vanishes at every history exactly at the proximal policy.

Conversely, at $\pi_\theta=\pi_k^*$, Equation~\eqref{eq:prox-pathwise}
gives $\widehat A_i^{\mathrm{KL}}=V_k(h_i)-b(h_i)$ on every response.
Conditional on each history and behavior version $v$, let
$w_{i,v}=\pi_\theta(a_i\mid h_i)/\pi_{\mathrm{behav},v}(a_i\mid h_i)$. Then
\[
 \E_{\pi_{\mathrm{behav},v}}
 [w_{i,v}\psi_i\widehat A_i^{\mathrm{KL}}\mid h_i,v]
 =(V_k(h_i)-b(h_i))\E_{\pi_k^*}[\psi_i\mid h_i]=0.
\]
Thus the Gibbs policy also makes the expected field zero.

For a shared-parameter model, collect the independent log odds into a
vector $z(\theta)$ and let $g_z$ be this same field in those coordinates.
The chain rule gives $M_{\mathrm{behav}}(\theta)=Dz(\theta)^\top g_z$.
Full row rank of $Dz(\theta)$ makes its transpose injective, so the
parameter-space field is zero exactly when $g_z=0$. This proves
Theorem~\ref{thm:prox-recursion} under the stated local controllability
assumption.
\end{proof}
\par
\endgroup

\subsection{Population Convergence and Outer Improvement}
\label{app:prox-perturbation}

We next connect the stationary target to improvement across reference refreshes. First,
we prove convergence of the tabular population flow to the frozen-reference target. We
then show that stopping each inner solve at a sufficiently small success-log-odds error
still yields improvement after every refresh.

For the convergence result, we specialize Theorem~\ref{thm:prox-recursion} to
independent softmax logits at every history, continuous-time population
updates, and behavior occupancy uniformly bounded away from zero.  These
conditions yield a direct Lyapunov analysis of the update dynamics.

Fix $\pi_k$ and follow this ordinary population PG field. Let $t$ denote
time, let $d_t(h)$ be the behavior occupancy of history $h$, and define, for
$i=\operatorname{depth}(h)$,
\begin{equation}
 \widetilde Q_t(h,a)=\sum_v\omega_t(v\mid h)
 \mathbb E_{\pi_{\mathrm{behav},t,v}}
 [\widehat Q_i^{\pi_{\theta(t)}}\mid h_i=h,a_i=a].
 \label{eq:prox-stale-q}
\end{equation}

\begingroup
\begin{theorem}[Convergence of the asynchronous population flow]
\label{lem:prox-pg-flow}
Assume (A1) through (A4) at every time $t$, with independent softmax logits at every
nonterminal history and a fixed reference $\pi_k$. Allow the behavior
policy or mixture of behavior-policy versions to vary with time.
Require its mixture weights and conditional continuation probabilities to be
measurable in time, and hold the sampling law fixed when forming each update.
Consider a global absolutely continuous solution of the population logit flow
from finite initial logits. Suppose $d_t(h)\ge d_{\min}>0$ for every nonterminal history
and all $t\ge0$. As a visit probability, $d_t(h)\le1$.
Fix a history $h$ with $m$ allowed actions. Let
$\boldsymbol z_h(t)$ be its learner logit vector, so
$\pi_{\theta(t)}(\cdot\mid h)=\operatorname{softmax}(\boldsymbol z_h(t))$.
Choose target logits $\boldsymbol z_h^*$ with
$\operatorname{softmax}(\boldsymbol z_h^*)=\pi_k^*(\cdot\mid h)$;
these logits are determined up to an additive common shift. Write
$\boldsymbol p_h(t)=\pi_{\theta(t)}(\cdot\mid h)$,
$\boldsymbol p_h^*=\pi_k^*(\cdot\mid h)$, and
\begin{equation*}
 P_h=I-\frac1m\mathbf 1\mathbf 1^\top,
 \qquad
 \boldsymbol r_h(t)=P_h\log\frac{\boldsymbol p_h(t)}{\boldsymbol p_h^*}.
\end{equation*}
The logarithm and division are componentwise. Write
$\boldsymbol\epsilon_h(t)=\widetilde Q_t(h,\cdot)-Q_k(h,\cdot)$ and
$\mathsf C(\boldsymbol p)=\operatorname{diag}(\boldsymbol p)
-\boldsymbol p\boldsymbol p^\top$. Then the ordinary expected PG logit flow
at $h$ is
\begin{equation}
 \dot{\boldsymbol z}_h(t)
 =d_t(h)\mathsf C(\boldsymbol p_h(t))
   \{-\beta\boldsymbol r_h(t)+\boldsymbol\epsilon_h(t)\}.
 \label{eq:prox-pg-flow}
\end{equation}
Consequently, for $\mathcal V_h(t)=\tfrac12\|\boldsymbol r_h(t)\|_2^2$,
\begin{equation}
 \begin{aligned}
 \dot{\mathcal V}_h(t)
 &=-\beta d_t(h)\boldsymbol r_h(t)^\top
   \mathsf C(\boldsymbol p_h(t))\boldsymbol r_h(t)\\
 &+d_t(h)\boldsymbol r_h(t)^\top
   \mathsf C(\boldsymbol p_h(t))\boldsymbol\epsilon_h(t).
 \end{aligned}
 \label{eq:prox-pg-lyapunov}
\end{equation}
The policies satisfy $\boldsymbol p_h(t)\to\boldsymbol p_h^*$ at every
history, including under arbitrary time-varying supported behavior
mixtures satisfying the occupancy bound.
\end{theorem}
\par
\endgroup

\begingroup
\begin{proof}
For a prescribed measurable behavior law, interpret the flow as an absolutely
continuous solution satisfying the differential equation almost everywhere in time
(a Carath\'eodory solution). On the finite response trees, the vector field is
measurable in time and locally Lipschitz in the logits, uniformly over behavior laws.
Log-softmax terms grow at most linearly in the logits, and probabilities and scores
are bounded. Thus the vector field has at most linear growth, which ensures global
existence from finite initial logits. The derivative identities and inequalities
below hold almost everywhere; their integrated forms follow from absolute continuity.
For a behavior law chosen from past learner policies, the convergence assertion
applies to any global absolutely continuous solution satisfying the assumptions,
conditional on existence of such a solution.

Let $\boldsymbol e_a$ be the one-hot vector for action $a$. The softmax
identity $\nabla_{\boldsymbol z_h}\log p_h(a)=\boldsymbol e_a-\boldsymbol p_h$
and \eqref{eq:prox-conditional-field} give the ordinary PG logit block
\begin{equation*}
 d_t(h)\mathsf C(\boldsymbol p_h)
 \left\{\widetilde Q_t(h,\cdot)-\beta\log
 \frac{\boldsymbol p_h}{\pi_k(\cdot\mid h)}\right\};
\end{equation*}
the action-independent baseline vanishes because
$\mathsf C(\boldsymbol p_h)\mathbf 1=\boldsymbol 0$.
The proximal identity \eqref{eq:prox-soft-policy} implies
\begin{equation*}
 Q_k(h,\cdot)-\beta\log
 \frac{\boldsymbol p_h^*}{\pi_k(\cdot\mid h)}=V_k(h)\mathbf 1.
\end{equation*}
Subtracting this identity inside the PG block and again removing constants
with $\mathsf C(\boldsymbol p_h)\mathbf 1=\boldsymbol 0$ proves
\eqref{eq:prox-pg-flow}. Finally,
$\boldsymbol r_h=P_h(\boldsymbol z_h-\boldsymbol z_h^*)$ and the right-hand
side of \eqref{eq:prox-pg-flow} sums to zero, so
$\dot{\boldsymbol r}_h=\dot{\boldsymbol z}_h$. Taking the inner product with
$\boldsymbol r_h$ proves \eqref{eq:prox-pg-lyapunov}.

\paragraph{Backward convergence of ordinary PG.}
At the last token position the future cost sum is empty, so
$\widehat Q_T^{\pi_{\theta(t)}}=R=Q_k(h,a)$ for every learner policy.
Thus $\boldsymbol\epsilon_h(t)=\boldsymbol 0$. Equation~\eqref{eq:prox-pg-lyapunov}
therefore makes $\mathcal V_h(t)$ nonincreasing. The centered log ratios remain
bounded, so normalization and the positive support of $\boldsymbol p_h^*$
keep every component of $\boldsymbol p_h(t)$ uniformly away from zero.
Consequently, $\lambda_h:=\inf_{t\ge0}\min_a p_h(t,a)>0$. For centered
$\boldsymbol r_h$, the weighted-variance identity gives
\begin{equation*}
 \boldsymbol r_h^\top \mathsf C(\boldsymbol p_h)\boldsymbol r_h
 =\operatorname{Var}_{a\sim\boldsymbol p_h}[r_h(a)]
 \ge\lambda_h\|\boldsymbol r_h\|_2^2.
\end{equation*}
Thus $\dot{\mathcal V}_h\le-2\beta d_{\min}\lambda_h \mathcal V_h$, and the last-token
policies converge exponentially to $\boldsymbol p_h^*$.

Proceed backward by induction. Define
\begin{equation*}
 \operatorname{osc}_j(t)=\max_{h:\,\operatorname{depth}(h)=j}
 \operatorname{osc}_a\log
 \frac{\pi_{\theta(t)}(a\mid h)}{\pi_k^*(a\mid h)}.
\end{equation*}
Here $\operatorname{osc}_a f(a)=\max_a f(a)-\min_a f(a)$ is the
range across allowed actions.
Because both policies are normalized, their log probability ratio either
vanishes for every action or takes both nonpositive and nonnegative values.
Its absolute value is therefore bounded by its oscillation $\operatorname{osc}_j(t)$.
If all depths after $i$ converge exponentially, the tail identity
\eqref{eq:prox-tail-error} gives, uniformly over histories at depth $i$,
\begin{equation*}
 \|\boldsymbol\epsilon_h(t)\|_\infty
 \le\beta\sum_{j>i}\operatorname{osc}_j(t)
 \le C_i e^{-c_i t}
\end{equation*}
for some $C_i,c_i>0$. To handle $\boldsymbol r_h=\boldsymbol 0$ without a
separate case, fix $\nu>0$. The flow equation and the negative
semidefiniteness of its first term imply
\begin{equation*}
 \frac{d}{dt}\sqrt{\|\boldsymbol r_h(t)\|_2^2+\nu}
 \le d_t(h)\|\mathsf C(\boldsymbol p_h(t))\|_{\mathrm{op}}
       \|\boldsymbol\epsilon_h(t)\|_2.
\end{equation*}
The right-hand side is integrable. Integrating and then letting
$\nu\downarrow0$ shows that $\boldsymbol r_h(t)$ remains bounded, so
$\boldsymbol p_h(t)$ again remains uniformly in the interior. Applying
Young's inequality to \eqref{eq:prox-pg-lyapunov} then gives constants
$c_h,C_h>0$ such that
\begin{equation*}
 \dot{\mathcal V}_h(t)\le-c_h\mathcal V_h(t)+C_h\|\boldsymbol\epsilon_h(t)\|_2^2.
\end{equation*}
The forcing term decays exponentially, so Gronwall's inequality yields
$\mathcal V_h(t)\to0$ exponentially. Backward induction proves convergence to
$\pi_k^*$ at every history.
\end{proof}
\par
\endgroup

\begin{corollary}[Reference refresh with approximate inner solves]
\label{cor:prox-async-outer}
Assume a finite prompt set, binary reward, and a full-support initial
reference with $0<p_0(x)<1$ for every prompt, where
$p_k(x)=\Pr_{\tau\sim\pi_k(\cdot\mid x)}[R(x,\tau)=1]$.
Fix $\beta>0$ and $\varepsilon_{\mathrm{logit}}\in(0,1/\beta)$.
At round $k$, let $\pi_k^*$ maximize the proximal objective with
reference $\pi_k$, write
$p_k^*(x)=\Pr_{\tau\sim\pi_k^*(\cdot\mid x)}[R(x,\tau)=1]$,
and refresh the reference to a full-support
approximation $\pi_{k+1}$ satisfying, for every prompt,
\begin{equation}
 \left|\operatorname{logit}p_{k+1}(x)
 -\operatorname{logit}p_k^*(x)\right|\le\varepsilon_{\mathrm{logit}}.
 \label{eq:prox-async-stopping}
\end{equation}
Then, with $\operatorname{logit}(p)=\log(p/(1-p))$ and
$\sigma(u)=(1+e^{-u})^{-1}$,
\[
 p_k(x)\ge\sigma\!\left(\operatorname{logit}p_0(x)
       +k(1/\beta-\varepsilon_{\mathrm{logit}})\right)
 \longrightarrow 1.
\]
Under the assumptions of Theorem~\ref{lem:prox-pg-flow} in each round,
its inner flow reaches the required tolerance in finite time.
\end{corollary}

\begin{proof}
By Proposition~\ref{thm:prox-improvement},
\begin{equation*}
 \operatorname{logit}p_k^*(x)
 =\operatorname{logit}p_k(x)+1/\beta.
\end{equation*}
Hence each approximate round increases every success log-odds by at least
$1/\beta-\varepsilon_{\mathrm{logit}}>0$. Summing these increments and
applying $\sigma$ proves the bound and convergence.
Theorem~\ref{lem:prox-pg-flow} gives convergence to $\pi_k^*$ in each
inner flow. The finite prompt set and continuity of success log-odds
then give a finite stopping time satisfying
Equation~\eqref{eq:prox-async-stopping} for all prompts.
\end{proof}

The stopping condition measures each inner solve by its success-log-odds error relative
to the proximal target. Proposition~\ref{thm:prox-improvement} gives the exact step's
improvement, and any fixed tolerance below $1/\beta$ leaves a positive increment at
every outer round. Appendix~\ref{ext:shared} gives a finite-sample reward guarantee using
expected trajectory-KL gaps to measure inner accuracy.

\section{Finite-sample convergence with a linear logistic model}
\label{ext:shared}

This appendix gives the precise version of Theorem~\ref{ext:reward-informal}.
We first specify the model, sampling rules, and projected group update.
Lemmas~\ref{ext:log-moment} through~\ref{ext:inner} establish inner accuracy, and
Lemmas~\ref{ext:src:lem:concave-recursion} through~\ref{ext:src:lem:relative-angular-transfer}
propagate proximal errors across reference refreshes.
Theorem~\ref{ext:reward} then gives the complete reward guarantee and budgets.

\paragraph{Token ratios and their second-moment budget.}
Let $T$ bound the response length and $h_i=(x,a_{<i})$ be the history before
token $i$. At each training update, write
$p_i=\pi_\theta(\cdot\mid h_i)$ for the learner token distribution,
$q_i=\pi_{\mathrm{behav}}(\cdot\mid h_i)$ for the behavior distribution, and
$w_i=p_i/q_i$. Throughout this appendix, $V_{\log}\ge0$ denotes a deterministic
upper bound satisfying
\begin{equation}\label{ext:log-budget-definition}
 \operatorname*{ess\,sup}_{\text{rounds, updates},\,x,\tau}\sum_{i=1}^T
 \E_{a\sim p_i}[(\log w_i)^2]\le V_{\log}.
\end{equation}
Here $\tau=(a_1,\ldots,a_T)$ ranges over supported responses and
$h_i=(x,a_{<i})$ follows that same response path. The essential supremum
includes training randomness. For a fixed pair of policies and prompt, it is
only over response paths. Thus the bound controls the sum of predictable
second moments along a path, rather than summing separate worst-case histories.

\paragraph{Linear logistic model and training rounds.}
For the finite-sample analysis, we use a linear logistic model on finite prompt
and response trees, with deterministic binary rewards.
Let $T$ bound response length and $h_i=(x,a_{<i})$ denote the entire
prefix before token $i$. Short responses can be padded with a
deterministic absorbing action whose score and KL cost are zero.
The response horizon $T\ge1$ is distinct from the practical credit window $H$.
Let $d$ denote the parameter dimension and let
$\phi(h,a)\in\mathbb R^d$ be a fixed feature vector describing
candidate token $a$ at prefix $h$, with $\|\phi(h,a)\|_2\le1$.
In the linear logistic model, the token logit is $\theta^\top\phi(h,a)$,
and softmax converts these logits into probabilities:
\begin{equation}
 \pi_\theta(a\mid h)=
 \frac{\exp(\theta^\top\phi(h,a))}
      {\sum_{a'}\exp(\theta^\top\phi(h,a'))},
 \qquad \theta\in\mathbb R^d.
 \label{ext:head}
\end{equation}
The same vector $\theta$ controls all histories and positions. Features may depend on
the full prefix, but remain fixed throughout training. One example is fine-tuning a
Transformer's linear output layer while freezing its hidden representations. The
analysis also applies to a first-order approximation of full-layer fine-tuning around a
fixed parameter vector: the features are the logit gradients at that vector, and the
logits may include a fixed baseline term. This gives a neural tangent kernel
interpretation of the linearized policy. Applying the result to the nonlinear
Transformer requires additional control of the approximation error.

Each round $k$ holds the reference $\pi_k=\pi_{\theta_k}$ fixed while updating the
learner toward its Gibbs target. The behavior policy supplies the responses used for
these inner updates.

\paragraph{Group update for the convergence analysis.}
The analyzed variant of Algorithm~\ref{alg:proximal-kl} uses full KL credit and
projected gradient steps. Each inner update uses a fresh group of $G\ge2$ responses
under the \hyperref[ext:sampling]{sampling rules below}. This replaces batch collection and
minibatch selection in lines~\ref{algline:collect} and~\ref{algline:minibatch}. With the
full future window, we also replace the normalized advantages and loss in
lines~\ref{algline:advantage}, \ref{algline:credit}, and~\ref{algline:loss} by the full
residual and centered group estimator below.
Let $B\ge1$ be the representation bound in the target representation condition
below, and set $\beta=16B\sqrt T$ and $r_{\mathrm{proj}}=B/\beta$.
Use step size $\eta=[2048\beta T(1+T/G)]^{-1}$.
Let $\Pi_k$ be Euclidean projection onto the parameter ball
$\{\theta:\|\theta-\theta_k\|_2\le r_{\mathrm{proj}}\}$, centered at the frozen reference.
Projection leaves a proposed step unchanged inside the ball and rescales
its displacement from $\theta_k$ to length $r_{\mathrm{proj}}$ otherwise:
\begin{align}
 U_m&=R_m-\beta\sum_{i=1}^T
       \ell_{m,i},\notag\\
 \Psi_m&=\sum_{i=1}^T w_{m,i}
              \nabla_\theta\log\pi_\theta(a_{m,i}\mid h_{m,i}),\qquad
 \bar U=\frac1G\sum_{l=1}^G U_l,\notag\\
 \widehat g_j&=\frac1{G-1}\sum_{m=1}^G\Psi_m
              (U_m-\bar U),\qquad
 \vartheta_{j+1}=\Pi_k(\vartheta_j+\eta\widehat g_j).
 \label{ext:update}
\end{align}
Here $m=1,\ldots,G$ indexes responses and $i=1,\ldots,T$ indexes
tokens, with $h_{m,i}=(x,a_{m,<i})$. The reward of response $m$ is
$R_m=R(x,\tau_m)$, and $\vartheta_j$ is the learner parameter after $j$
inner updates, starting from $\vartheta_0=\theta_k$. At the current learner
$\theta=\vartheta_j$, the token weight and reference log ratio are
\[
 w_{m,i}=\frac{\pi_\theta(a_{m,i}\mid h_{m,i})}
 {\pi_{\mathrm{behav}}(a_{m,i}\mid h_{m,i})},\qquad
 \ell_{m,i}=\log\frac{\pi_\theta(a_{m,i}\mid h_{m,i})}
 {\pi_k(a_{m,i}\mid h_{m,i})}.
\]
The denominator of $w_{m,i}$ is the logged probability under the single
behavior policy selected for the entire group; this policy generates all $G$ responses.

\paragraph{Sampling for the analyzed update.}\label{ext:sampling}
Before each inner update, select a supported behavior policy using only information
available before drawing the group. Independently of the past, draw a fresh prompt
$x\sim\mathcal D$. Conditional on this prompt and the selected behavior policy,
draw $G\ge2$ independent responses and log their actual token probabilities.
Group admission is independent of realized actions, rewards, and lengths;
consumed responses are not reused.

The normalization by $G-1$ corrects the finite-group scaling;
using $1/G$ instead gives the same parameter update after multiplying
the step size by $G/(G-1)$.
At line~\ref{algline:update}, replace Adam by the projected gradient
step in Equation~\eqref{ext:update}. Algorithm~\ref{alg:proximal-kl}
as implemented does not perform this projection; the finite-sample
theorem concerns the projected group variant specified here.

At line~\ref{algline:refresh}, select an independent uniform pre-update
output $\theta_{k+1}=\vartheta_I$,
$I\sim\operatorname{Unif}\{0,\ldots,L-1\}$, where $L$ is the number of inner
updates in round $k$. The coefficients are detached as in Algorithm~\ref{alg:proximal-kl}.
Earlier costs in the full residual cancel against each weighted token
score in expectation, giving the full-tail population field.

\paragraph{Target representation and bounded iterates.}\label{ext:representation-condition}
Assume some $\theta_k^*$ with $\|\theta_k^*-\theta_k\|_2\le r_{\mathrm{proj}}$
exactly represents the Gibbs policy at every legal history.
Choose $\theta_k^*$ as the exact representing parameter closest to $\theta_k$.
For fixed features on the finite trees, exact representation imposes linear
constraints on token log odds, so the set of representers is a nonempty closed
affine set with a fixed constraint matrix. The target log odds depend continuously
on $\theta_k$, so the closest representer is unique and depends continuously,
and hence measurably, on $\theta_k$.
The assumed existence of a representer in the ball ensures that this choice
also lies in the ball.
Starting at $\vartheta_0=\theta_k$, projection ensures that every
inner iterate satisfies
\begin{equation}
 \|\vartheta_j-\theta_k\|_2\le r_{\mathrm{proj}}=B/\beta,
 \qquad j=0,\ldots,L.
 \label{ext:local-iterates}
\end{equation}
Thus projection enforces the iterate bound, and bounded features turn it into the
uniform token bound $|\log(\pi_{\vartheta_j}/\pi_k)|\le2r_{\mathrm{proj}}$. The estimator retains the
full, unmodified KL costs. For binary rewards, the Gibbs target also stays near the
reference: $|\log(\pi_k^*/\pi_k)|\le1/\beta$.

With the choice $\beta=16B\sqrt T$ in Theorem~\ref{ext:reward},
the radius in Equation~\eqref{ext:local-iterates} is $1/(16\sqrt T)$. If exact Gibbs targets admit parameter representations
whose displacement from the reference tends uniformly to zero as
$\beta$ increases, a sufficiently large $B$ places them inside this ball.
Appendix~\ref{ext:representation-examples} gives tabular and shared-feature
examples and the corresponding condition for fixed Transformer tangent
features. Approximate representation introduces additional error terms
(Remark~\ref{ext:approximate-representation}); application to a nonlinear
Transformer also requires control of linearization and gradient errors.

\paragraph{Relation to representation assumptions in prior work.}
Compatible function approximation links log-linear policy updates to
value-function approximation in natural policy-gradient analyses
\citep{agarwal2021pg} and to inexact policy mirror descent
\citep{yuan2023loglinear}.
Closure of the policy class under policy improvement is another related
structural condition \citep{bhandari2024global}.
Our assumption requires exact representation of each reference-relative
Gibbs target with parameter displacement at most $B/\beta$.
Let $Q_k$ be the reference's soft action value in
Equation~\eqref{eq:prox-values}. A sufficient condition is that some
history-dependent offset $c_k(h)$ and vector $v_k\in\mathbb R^d$ with
$\|v_k\|_2\le B$ satisfy $Q_k(h,a)-c_k(h)=v_k^\top\phi(h,a)$ at every legal
history and action; then $\theta_k+v_k/\beta$ represents the target.
\citet[Remark 24]{agarwal2021pg} identify zero transfer bias in linear MDPs;
Appendix~\ref{ext:representation-examples} verifies this representation condition
for tabular policies and a deterministic linear-MDP subclass.

\paragraph{Outer rounds and the final reward guarantee.}
Over outer rounds $k=0,\ldots,K-1$, each inner solve produces an independently selected
output that becomes the reference for the following round. Theorem~\ref{ext:reward} specifies how many inner updates
and outer rounds are sufficient to obtain a reward guarantee for one response from the
final policy.

For a comparator policy $\pi_\star$, write
$p_\star(x)=\Pr_{\pi_\star}[R=1\mid x]$ and
$p_0(x)=\Pr_{\pi_0}[R=1\mid x]$. Assume initial correct-solution coverage
\begin{equation}
 p_\star(x)\le C_\star p_0(x),\qquad C_\star\ge1,
 \quad\text{for every prompt of positive probability}.
 \label{ext:coverage}
\end{equation}
The coverage condition compares total probability mass on correct responses, so it can
hold without bounded likelihood ratios for individual trajectories.
Let $J(\pi)=\E_x\Pr_\pi[R=1\mid x]$.

\paragraph{Best-of-$N$ specialization.}
Let $q_N$ be the oracle distribution obtained by drawing $N$ independent
responses from $\pi_0$ for each prompt and selecting a correct one
when any exists. Its success probability satisfies
\begin{equation}
 p_{q_N}(x)=1-(1-p_0(x))^N\le Np_0(x).
 \label{ext:best-n}
\end{equation}
Thus the theorem applies with $\pi_\star=q_N$ and $C_\star=N$.
The resulting guarantee compares the reward of one response from the trained policy with
that of oracle best-of-$N$ selection at initialization. The theorem accounts separately
for the training budget and the oracle's $N$ draws per prompt.

\paragraph{Uniform behavior mismatch.}
For the supporting lemmas below, assume the token ratios satisfy
$1/2\le w_i\le2$ at every allowed action and history at every actual iterate,
almost surely, and assume $V_{\log}\le1/8$, with $V_{\log}$ as in
Equation~\eqref{ext:log-budget-definition}.
The parameter $B$, local target representation, and these mismatch conditions
are required at every encountered reference. The full reward theorem is stated
after the inner and outer lemmas, as Theorem~\ref{ext:reward}.

At the represented target, every response to a prompt has full residual
$\beta\log Z_k(x)$. Group centering therefore makes every sampled update zero
at the target, a property that controls its sampling second moment.

\subsection{Accuracy of one proximal round}
\label{ext:inner-analysis}
\label{ext:local-proof}

We establish Lemma~\ref{ext:inner} using log-moment and history-transfer
bounds, a drift estimate, and a group-estimator noise bound.

\subsubsection*{Log-moment bounds and history transfer}

We introduce the token errors, then use one conditional log-moment bound
for likelihood-ratio moments and the conditional means needed in the drift
analysis. A single transfer lemma converts these moments into learner-history
expectations.

\paragraph{Policies, expectations, and token errors.}
Fix a reference $\pi_k$, its Gibbs target $\pi_k^*$ from
Equation~\eqref{eq:prox-conditional-gibbs}, and one current learner $\pi_\theta$.
At token position $i$, the history is $h_i=(x,a_{<i})$. Write
\begin{equation}\label{ext:notation-policies}
 p_i(a)=\pi_\theta(a\mid h_i),\qquad
 q_i(a)=\pi_{\mathrm{behav}}(a\mid h_i),\qquad
 p_i^*(a)=\pi_k^*(a\mid h_i),\qquad w_i=\frac{p_i(a_i)}{q_i(a_i)}.
\end{equation}
The expectation $\E_{p_i}$ averages only the next action at the fixed history $h_i$;
$\E_{q_i}$ does the same under the behavior policy. By contrast, $\E_q$ averages a
whole behavior response for a fixed prompt, including its random histories;
$\E_p$ averages a learner response.
All policies have the same finite support, and responses have at most $T$ tokens.
Absorbing padding tokens have probability and importance ratio one and log error zero.

The error relative to the Gibbs target and the cost relative to the reference are
\begin{equation}\label{ext:notation-errors}
 z_i=\log\frac{p_i(a_i)}{p_i^*(a_i)},\qquad
 \ell_i=\log\frac{p_i(a_i)}{\pi_k(a_i\mid h_i)},\qquad
 m_i=\E_{p_i}z_i,\qquad u_i=z_i-m_i.
\end{equation}
These two log ratios serve different purposes: $z_i$ measures error to the target,
whereas $\ell_i$ is the KL cost used by the algorithm. The mean
$m_i=\KL(p_i\Vert p_i^*)$ is a function of the history. The centered error $u_i$
has $\E_{p_i}u_i=0$. Measure its squared size first at one history and then over
behavior histories:
\begin{equation}\label{ext:notation-energy}
 \sigma_i^2=\E_{p_i}u_i^2,\qquad
 v_i=\E_q\sigma_i^2.
\end{equation}
Thus $v_i$ is not a new type of error: it averages the same conditional variance
$\sigma_i^2$ over the histories visited at position $i$.

\paragraph{Centered log error as a directional score.}
For the linear logistic model in Equation~\eqref{ext:head}, suppose the target is
represented exactly by $\theta^*$, so $p_i^*=\pi_{\theta^*}(\cdot\mid h_i)$.
Define the parameter error $\Delta=\theta-\theta^*$ and the token score
$\psi_i=\nabla_\theta\log p_i(a_i)$. At a fixed history let
$Z_i(\theta)=\sum_a\exp(\theta^\top\phi(h_i,a))$ be the softmax normalizer.
Taking the log ratio of the two softmax policies gives
\[
 z_i=\Delta^\top\phi(h_i,a_i)-\log Z_i(\theta)+\log Z_i(\theta^*).
\]
The last two terms do not depend on the sampled action. Subtracting the
$p_i$-mean therefore removes them, while differentiating the log softmax gives
\[
 u_i=\Delta^\top\bigl(\phi(h_i,a_i)-\E_{p_i}\phi(h_i,a)\bigr),\qquad
 \psi_i=\phi(h_i,a_i)-\E_{p_i}\phi(h_i,a).
\]
Consequently the centered log error is exactly the score in the parameter-error
direction. Summing the importance-weighted token scores defines the response score:
\begin{equation}\label{ext:notation-scores}
 u_i=\Delta^\top\psi_i,\qquad
 \Psi=\sum_iw_i\psi_i,\qquad
 \Delta^\top\Psi=\sum_iw_i u_i,\qquad
 \sigma_i^2=\Var_{p_i}[\Delta^\top\phi(h_i,a)].
\end{equation}
Finally, normalization of $p_i^*$ gives
$1=\E_{p_i}e^{-z_i}=e^{-m_i}\E_{p_i}e^{-u_i}$, hence
\begin{equation}\label{ext:notation-log-mean}
 m_i=\log\E_{p_i}e^{-u_i}.
\end{equation}
The following scalar bounds use only these centering and normalization identities;
the linear model is needed when connecting them to parameter updates.

\begin{lemma}[Conditional log moments and sequential composition]
\label{ext:log-moment}
For a centered random variable $X\le b$ with $b>0$, the Bernstein log-MGF bound is
\begin{equation}\label{ext:bernstein-centered}
 \log\E e^{tX}\le\frac{e^{bt}-1-bt}{b^2}\E X^2
 \le\frac{t^2\E X^2}{2(1-bt/3)},\qquad 0\le bt<3,
\end{equation}
where the first inequality holds for every $t\ge0$.
Now let $W_i=\prod_{j<i}w_j$ be a prefix likelihood ratio and assume
$w_i\ge1/2$. Use the pathwise budget $V_{\log}$ from
Equation~\eqref{ext:log-budget-definition}, and let $b>0$ uniformly bound
$\log w_i-\E_{p_i}\log w_i$ from above. For every $t\ge0$,
\begin{equation}\label{ext:bernstein-log-moment}
 \log\E_qW_i^{1+t}
 \le V_{\log}\left[t+\frac{e^{bt}-1-bt}{b^2}\right].
\end{equation}
For $bt<3$, the right side is at most
$V_{\log}[t+t^2/(2(1-bt/3))]$.
In particular, for $0<\delta\le1$, if $|\log w_i|\le3\delta/2$,
\begin{equation}\label{ext:bernstein-prefix-specialization}
 \delta\log\E_qW_i^{1+1/\delta}\le(1+\delta^{-1})V_{\log}.
\end{equation}
These bounds hold for every prefix and the full-response ratio $W_{T+1}$.
\end{lemma}
\begin{proof}
Equation~\eqref{ext:bernstein-centered} is the conditional log-MGF estimate
of \citet[Lemma~2.20 and Equation~(2.13)]{TZ23-lt}.
At a fixed history, write $Y_i=\log w_i$, $\mu_i=\E_{p_i}Y_i$, and
$\nu_i(h_i)=\E_{p_i}Y_i^2$.
Normalization gives $\E_{p_i}e^{-Y_i}=1$. Since $w_i\ge1/2$,
$e^{-y}-1+y\le y^2$ for $y\ge-\log2$, and hence
\[
 0\le\mu_i=\E_{p_i}[e^{-Y_i}-1+Y_i]\le\nu_i(h_i),
 \qquad \Var_{p_i}Y_i\le\nu_i(h_i).
\]
Let $g_b(t)=t+(e^{bt}-1-bt)/b^2$.
Equation~\eqref{ext:bernstein-centered} gives
$\log\E_{p_i}e^{tY_i}\le g_b(t)\nu_i(h_i)$.
The sequential exponential identity
\citep[Lemma~13.1]{TZ23-lt}, applied under the learner policy, yields
\[
 \E_p\exp\!\left(t\sum_{j<i}Y_j-
        \sum_{j<i}\log\E_{p_j}e^{tY_j}\right)=1.
\]
The pathwise budget implies
$\sum_{j<i}\log\E_{p_j}e^{tY_j}\le g_b(t)V_{\log}$.
Since $\E_qW_i^{1+t}=\E_pW_i^t$, this proves
Equation~\eqref{ext:bernstein-log-moment}. The rational bound follows from
Equation~\eqref{ext:bernstein-centered}. Taking $b=3\delta/2$ and $t=1/\delta$
proves Equation~\eqref{ext:bernstein-prefix-specialization}, because $\mu_i\ge0$.

\end{proof}

\begin{lemma}[Moment bounds and history transfer]
\label{ext:moment-transfer}
Let $W=dP/dQ$ for distributions on the same finite support, let
$0\le f\le M$, and write $a=\E_Qf$. For every real $s\ge2$,
\begin{equation}\label{ext:moment-transfer-bound}
 \E_Pf\le a+(\E_QW^s-1)^{1/s}M^{1/s}a^{1-1/s}.
\end{equation}
If $W\le C$ almost surely, then $\E_Pf\le Ca$.
For prefix histories with $1/2\le w_i\le2$ and pathwise budget
$V_{\log}\le1/8$, the second-moment specialization gives
\begin{equation}\label{ext:token-history-bound}
 \E_q(W_i-1)^2\le e^{5V_{\log}/3}-1\le2V_{\log},\qquad
 \E_{h_i\sim p}f_i\le\E_{h_i\sim q}f_i+
 c_0\sqrt{2V_{\log}\E_{h_i\sim q}f_i}
\end{equation}
for any $c_0>0$ and $0\le f_i\le c_0^2$.
\end{lemma}
\begin{proof}
For $w\ge0$ and $s\ge2$,
$|w-1|^s\le w^s-1-s(w-1)$; differentiation on either side of $w=1$
verifies this inequality. Since $\E_QW=1$,
$\E_Q|W-1|^s\le\E_QW^s-1$.
H\"older applied to $\E_Q[(W-1)f]$, together with
$f^{s/(s-1)}\le M^{1/(s-1)}f$, proves
Equation~\eqref{ext:moment-transfer-bound}. The bounded-ratio case follows
from $\E_Pf=\E_Q[Wf]$.
For the prefix specialization, apply the rational bound in
Lemma~\ref{ext:log-moment} with $t=1$ and $b=\log2$:
\[
 \log\E_qW_i^2\le
 \left(1+\frac1{2(1-\log2/3)}\right)V_{\log}
 \le\frac53V_{\log}.
\]
Normalization gives $\E_q(W_i-1)^2=\E_qW_i^2-1$.
For $0\le v\le1/8$, convexity yields
$e^{5v/3}-1\le8(e^{5/24}-1)v\le2v$.
Taking $s=2$ and $M=c_0^2$ proves the transfer bound.
\end{proof}

\begin{lemma}[Martingale continuation bounds]
\label{ext:martingale-continuation}
Let $(\mathcal F_i)_{i=1}^{T+1}$ be a filtration, with $\mathcal F_i$
representing the history before step $i$. Let $\xi_i,z_i$ be
square-integrable $\mathcal F_{i+1}$-measurable variables,
for $i=1,\ldots,T$. Suppose
$\E[\xi_i\mid\mathcal F_i]=0$, and write
$\mu_i=\E[z_i\mid\mathcal F_i]$.
For nonnegative $\mathcal F_i$-measurable $v_i$ with
$\E\sum_i v_i<\infty$ and constants $D,B_0,\Gamma\ge0$, assume
\[
 \E[\xi_i^2\mid\mathcal F_i]\le Dv_i,\qquad
 \Var(z_i\mid\mathcal F_i)\le B_0v_i,\qquad
 \E\left(\sum_i\mu_i\right)^2\le\Gamma^2\E\sum_iv_i.
\]
Then
\begin{equation}\label{ext:martingale-continuation-bound}
 \left|\E\sum_i\xi_i\sum_{t>i}z_t\right|
 \le\sqrt D\,\Gamma\E\sum_iv_i,\qquad
 \Var\left(\sum_iz_i\right)
 \le2(B_0+\Gamma^2)\E\sum_iv_i.
\end{equation}
The same statements hold conditionally on initial information when all
hypotheses hold under that conditioning.
\end{lemma}
\begin{proof}
Set $M=\sum_i\xi_i$ and $A=\sum_i\mu_i$. Conditional centering and
martingale orthogonality give $\E M^2\le D\E\sum_i v_i$ and
\[
 \E\sum_i\xi_i\sum_{t>i}z_t=\E[MA].
\]
Indeed, replace each later $z_t$ by its conditional mean; terms with
$t\le i$ in $MA$ vanish by conditioning on $\mathcal F_i$.
Cauchy-Schwarz proves the first bound. The increments $z_i-\mu_i$
are martingale differences, so
$\E(\sum_i(z_i-\mu_i))^2\le B_0\E\sum_i v_i$.
Writing $\sum_i z_i=\sum_i(z_i-\mu_i)+A$ and using
$(x+y)^2\le2x^2+2y^2$ proves the second bound.
\end{proof}

\subsubsection*{Progress toward the proximal target}

The next lemma combines the exact soft-Q identity with control of the
error introduced by behavior-sampled continuation values. Its conditional
mean estimate uses the same Bernstein bound as the history-transfer argument.

\begin{lemma}[Progress and continuation error]
\label{ext:local-q-aggregate}
Use the token errors and scores in
Equations~\eqref{ext:notation-errors} through~\eqref{ext:notation-scores}.
Suppose that, for some $c_0>0$, the span of $\Delta^\top\phi(h_i,\cdot)$ is at most $2c_0$,
$2c_0\le3/16$, and $1/2\le w_i\le2$, and use the pathwise budget
in Equation~\eqref{ext:log-budget-definition}. Set
\[
 \Gamma=\frac{\sqrt{V_{\log}}}{\log2}+\frac8{15}c_0\sqrt T,
 \qquad e_i=-\beta\sum_{t>i}z_t.
\]
Then $\sigma_i\le c_0$, $0\le m_i\le(8/15)\sigma_i^2$, and
\begin{equation}\label{ext:local-mean-sum}
 \E_q\left(\sum_i\E_{q_i}z_i\right)^2\le\Gamma^2\sum_i v_i.
\end{equation}
For the full residual $U$ in Equation~\eqref{ext:update},
\begin{equation}\label{ext:local-q-decomposition}
 -\E_q[(\Delta^\top\Psi)U]
 =\beta\sum_i v_i-\sum_i\E_q[w_i u_i e_i],
 \qquad \E\widehat g=\E_q[\Psi U].
\end{equation}
The continuation error satisfies
\begin{equation}\label{ext:local-q-error-bound}
 \left|\sum_i\E_q[w_i u_i e_i]\right|
 \le\sqrt2\,\Gamma\beta\sum_i v_i.
\end{equation}
All expectations condition on the reference, learner, behavior policy, and prompt.
\end{lemma}
\begin{proof}
The span condition gives $|u_i|\le2c_0$ and $\sigma_i\le c_0$.
By normalization, $m_i=\log\E_{p_i}e^{-u_i}$.
Apply Equation~\eqref{ext:bernstein-centered} to $-u_i$ with $t=1$ and
$b=2c_0$ to obtain
\[
 0\le m_i\le\frac{\sigma_i^2}{2(1-2c_0/3)}
 \le\frac8{15}\sigma_i^2.
\]
For $|y|\le\log2$, $|e^{-y}-1|\le|y|/\log2$.
Since $\E_{p_i}u_i=0$, Cauchy-Schwarz therefore gives
\[
 |\E_{q_i}u_i|=|\E_{p_i}[(w_i^{-1}-1)u_i]|
 \le\frac{\sqrt{\E_{p_i}[(\log w_i)^2]}}{\log2}\sigma_i.
\]
Write $\gamma_i(h_i)=\sqrt{\E_{p_i}[(\log w_i)^2]}/\log2+(8/15)c_0$.
Then $|\E_{q_i}z_i|\le\gamma_i(h_i)\sigma_i$ and, along every response path,
$\sum_i\gamma_i(h_i)^2\le\Gamma^2$.
Pathwise Cauchy-Schwarz gives
$(\sum_i\E_{q_i}z_i)^2\le\Gamma^2\sum_i\sigma_i^2$;
taking expectations proves Equation~\eqref{ext:local-mean-sum}.

Apply Lemma~\ref{ext:martingale-continuation} with
$\xi_i=w_i u_i$, predictable variance bound $\sigma_i^2$, and $D=B_0=2$.
Indeed, $\E_{q_i}(w_i u_i)=0$ and
$\E_{q_i}(w_i^2u_i^2)\le2\sigma_i^2$, while
$\Var_{q_i}z_i=\Var_{q_i}u_i\le2\sigma_i^2$.
Equation~\eqref{ext:local-mean-sum} supplies its conditional-mean hypothesis.
Its first conclusion, multiplied by $\beta$, proves
Equation~\eqref{ext:local-q-error-bound}.

Finally, let $V_k$ and $Q_k$ be the reference's soft value and action value in
Equation~\eqref{eq:prox-values}. The soft-policy identity in
Equation~\eqref{eq:prox-soft-policy} telescopes along the deterministic response tree to give
\[
 R-\beta\sum_{t>i}\ell_t=Q_k(h_i,a_i)-\beta\sum_{t>i}z_t,
 \qquad Q_k(h_i,a_i)-\beta\ell_i=V_k(h_i)-\beta z_i.
\]
Hence $U=V_k(h_i)-\beta z_i+e_i-\beta\sum_{t<i}\ell_t$.
The history-dependent terms cancel against $w_i u_i$ because
$\E_{q_i}(w_i u_i)=0$, while
$\E_{q_i}(w_i u_i z_i)=\E_{p_i}u_i^2=\sigma_i^2$.
Summing proves the first identity in Equation~\eqref{ext:local-q-decomposition}.
The group-mean identity follows from Lemma~\ref{ext:group-covariance} and
$\E_q\Psi=0$.
\end{proof}

\subsubsection*{Group-estimator mean and noise}

\begin{lemma}[Group covariance moments]
\label{ext:group-covariance}
Let $(S_m,U_m)$, $m=1,\ldots,G$, be independent copies of a random pair
$(S,U)$ with $S\in\mathbb R^d$, $\E S=0$, and $G\ge2$.
Assume $\E\|S\|^2$, $\E U^2$, and $\E[\|S\|^2(U-\E U)^2]$ are finite.
Define $\bar U=G^{-1}\sum_m U_m$ and
\[
 g=\frac1{G-1}\sum_{m=1}^G S_m(U_m-\bar U).
\]
Then $\E g=\E[SU]$. If $\E\|S\|^2\le M$ and $\|S\|\le R_S$ almost surely, then
\begin{equation}\label{ext:group-covariance-bound}
 \E\|g\|^2\le\left(M+\frac{R_S^2}{G-1}\right)\Var(U).
\end{equation}
Both conclusions also hold conditionally when the pairs are independent
copies given the conditioning information.
\end{lemma}
\begin{proof}
Write $Y=U-\E U$, $\mu=\E[SY]$, and $g=A-B_{\mathrm{cross}}$, where
\[
 A=\frac1G\sum_m S_mY_m,\qquad
 B_{\mathrm{cross}}=\frac1{G(G-1)}\sum_{m\ne l}S_mY_l.
\]
Centering and independence give $\E A=\mu$, $\E B_{\mathrm{cross}}=0$, and
\begin{align*}
 \E\|A\|^2
 &=\frac{\E[\|S\|^2Y^2]}G+\frac{G-1}G\|\mu\|^2,\\
 \E\|B_{\mathrm{cross}}\|^2
 &=\frac{\E\|S\|^2\Var(U)+\|\mu\|^2}{G(G-1)},\qquad
 \E\langle A,B_{\mathrm{cross}}\rangle=0.
\end{align*}
Only identical and reversed ordered index pairs survive in the second
moment of $B_{\mathrm{cross}}$; all other terms contain an independent centered factor.
Put $s=\E\|S\|^2\le\min(M,R_S^2)$. Cauchy-Schwarz gives
$\|\mu\|^2\le s\Var(U)$, while
$\E[\|S\|^2Y^2]\le R_S^2\Var(U)$.
The coefficient of $\Var(U)$ is therefore at most
\[
 \frac{R_S^2}G+\left(1+\frac{3-G}{G(G-1)}\right)s.
\]
For $G\ge3$ this is at most $M+R_S^2/(G-1)$.
For $G=2$ it is $R_S^2/2+3s/2\le R_S^2+M$.
This proves Equation~\eqref{ext:group-covariance-bound}, including groups of size two.
\end{proof}

\subsubsection*{Inner accuracy guarantee}

\paragraph{Coverage moments.}
At each actual inner iterate, let $P_i$ and $Q_i$ be the learner and
behavior distributions of the history before token $i$, conditional on the
prompt and the training information before the fresh group is drawn.
Write $W_i=dP_i/dQ_i$. For $2\le s<\infty$, define
\begin{equation}\label{ext:coverage-moments}
 \mathcal C_s=\mathop{\mathrm{ess\,sup}}_{k,j,x,i}
       (\E_{Q_i}W_i^s)^{1/(s-1)},
 \qquad
 \mathcal C_\infty=\mathop{\mathrm{ess\,sup}}_{k,j,x,i,h_i}W_i(h_i).
\end{equation}
The suprema include all encountered reference rounds, inner iterates,
positive-probability prompts, token positions, and selected behavior policies;
they also include training randomness. Assume the coefficient used below is
finite. These coefficients serve exclusively as auxiliary quantities in the proof.

\begin{lemma}[Accuracy of one proximal round]
\label{ext:inner}
Assume the linear logistic model in Equation~\eqref{ext:head} and the
\hyperref[ext:representation-condition]{local target representation condition stated above}. For $B\ge1$ and $G\ge2$, set
$\beta=16B\sqrt T$ and $r_{\mathrm{proj}}=B/\beta$. Use the projected group update in
Equation~\eqref{ext:update} under the \hyperref[ext:sampling]{sampling conditions stated with that update}.
Assume $1/2\le w_i\le2$ for every allowed action and history at every
actual iterate, almost surely. Assume $V_{\log}\le1/8$, with the uniform
second-moment budget in Equation~\eqref{ext:log-budget-definition}.
For one outer round $k$, freeze $\pi_k$ and start at
$\vartheta_0=\theta_k$. Run $L$ inner updates of
Equation~\eqref{ext:update} with
$\eta=[2048\beta T(1+T/G)]^{-1}$. Set
$\theta_{\mathrm{out}}=\vartheta_I$, with
$I\sim\operatorname{Unif}\{0,\ldots,L-1\}$ drawn independently.
Write $A=640B\sqrt T(1+T/G)$, $c_0=(2B+1)/(2\beta)$, and $\Omega=\beta T c_0^2$.
For every finite $s\ge2$, the coverage coefficient in
Equation~\eqref{ext:coverage-moments} gives
\begin{equation}\label{ext:inner-moment-rate}
 \beta\E\KL(\pi_{\theta_{\mathrm{out}}}\Vert\pi_k^*)
 \le\frac{3A}{5L}
 +\frac35[\Omega(\mathcal C_s^{s-1}-1)]^{1/s}
       \left(\frac AL\right)^{1-1/s}.
\end{equation}
If $\mathcal C_\infty<\infty$, the same gap is at most
$3\mathcal C_\infty A/(5L)$.
Trajectory KL is averaged over prompts. The outer expectation is over
sampled groups and the output index.
\end{lemma}

\begin{proof}[Proof of Lemma~\ref{ext:inner}]
Steps 1 through 4 control the update using histories sampled from the behavior policy. Step 5
then converts the resulting energy bound into a bound on learner trajectory KL.
Current-token weights suffice to control the update; prefix weights enter only when
changing the history distribution in Step 5.

Let $\mathcal F_j$ contain the reference, past updates, current learner, and
selected behavior policy before the fresh prompt group at update $j$ is drawn.
Fix one reference. Write
$\theta^*=\theta_k^*$, and retain $\vartheta_j$ for the inner iterates. All expectations within an update condition
on the reference, learner, and selected behavior policy before drawing the fresh
prompt group. The behavior policy may depend on past updates. For a fixed prompt,
write $q_i$, $p_i$, and $p_i^*$ for the behavior, learner, and
represented Gibbs conditional distributions at prefix $h_i$. Let $\psi_i=\nabla\log p_i(a_i)$ and $w_i=p_i(a_i)/q_i(a_i)$.

\paragraph{Step 1: projection and uniform token bounds.}
Projection in Equation~\eqref{ext:update} ensures that every iterate
satisfies $\|\vartheta_j-\theta_k\|_2\le r_{\mathrm{proj}}$.
Bounded features give
\[
 \left|\log\frac{\pi_\theta(a\mid h)}{\pi_{\theta_k}(a\mid h)}\right|
 \le2r_{\mathrm{proj}}=2B/\beta.
\]
Indeed, $t_a=(\theta-\theta_k)^\top\phi(h,a)$ has span at most $2r_{\mathrm{proj}}$,
and the log ratio is $t_a-\log\E_{\pi_{\theta_k}}e^t$.
The log moment lies between the minimum and maximum of $t$.
These bounds hold at every legal history, including histories absent
from the sampled group.

Let $c_0=(2B+1)/(2\beta)$.
The learner's pairwise log-odds change from the reference is at most
$2B/\beta$. The represented Gibbs target's change is at most $1/\beta$
by Equation~\eqref{eq:prox-soft-policy}, since binary-reward soft values
lie in $[0,1]$. Subtracting these changes shows that the span of
$(\theta-\theta^*)^\top\phi(h,\cdot)$ is at most $2c_0$.
Consequently its variance is at most $c_0^2$.
The quantity $c_0$ bounds the conditional token error $\sigma_i$ from
Equation~\eqref{ext:notation-energy}. Lemma~\ref{ext:local-q-aggregate} combines
this local bound with the token-mismatch budget $V_{\log}$ from Equation~\eqref{ext:log-budget-definition}.
Its aggregate relative error bound is
\begin{equation}
 \Gamma=\frac{\sqrt{V_{\log}}}{\log2}+\frac8{15}c_0\sqrt T.
 \label{ext:local-constants}
\end{equation}
Because $\beta=16B\sqrt T$, $B\ge1$, and $V_{\log}\le1/8$,
\[
 2c_0\le\frac3{16\sqrt T},\qquad
 \Gamma\le\frac{1}{\sqrt8\log2}+\frac1{20}<\frac35,\qquad
 \sqrt2\,\Gamma\le\frac{1}{2\log2}+\frac{\sqrt2}{20}<\frac45.
\]
These are deterministic bounds valid at every iterate.

\paragraph{Step 2: exact-Q progress dominates the aggregate Q-estimation error.}
Fix the current iterate $\theta=\vartheta_j$. Recall that $v_i$ is the average
squared centered token log error at position $i$. With exact soft Q-values,
the expected update would make progress $\beta\sum_i v_i$ toward $\theta^*$.
Future KL credit replaces each exact Q-value by an estimate with error
$e_i=-\beta\sum_{t>i}z_t$. Lemma~\ref{ext:local-q-aggregate} separates these two
contributions to the actual update.

Step 1 supplies the span bound required by Lemma~\ref{ext:local-q-aggregate}.
The assumed token-mismatch budget supplies its other hypothesis.
Lemma~\ref{ext:local-q-aggregate} therefore bounds the aggregate Q-estimation error
by a fraction $\sqrt2\,\Gamma$ of the exact-Q progress. Its decomposition therefore yields
\[
 -\E_q[(\Delta^\top\Psi)U]
 \ge\underbrace{\beta\sum_i v_i}_{\text{exact-Q progress}}
     -\underbrace{\sqrt2\,\Gamma\beta\sum_i v_i}_{\text{bound on Q-estimation error}}
 \ge\frac\beta5\sum_i v_i.
\]
The key point is that both terms scale with the same squared token-score errors;
as the learner approaches the target, the Q-estimation error shrinks along with
the progress term. Step 1 ensures $\sqrt2\,\Gamma<1$, so the update retains a positive
direction toward the target even with behavior-sampled suffixes.

Finally, average over prompts and define
\begin{equation}
 \mathcal E_j=\beta\sum_i
 \E_{x\sim\mathcal D,\,h_i\sim\pi_{\mathrm{behav}}}\sigma_i^2
 =\E_{x\sim\mathcal D}\left[\beta\sum_i v_i\right].
\label{ext:local-energy-definition}
\end{equation}
The group-mean identity in Lemma~\ref{ext:local-q-aggregate} converts the response-level
bound directly into the drift needed for the parameter update:
\begin{equation}
 \langle\E[\widehat g_j\mid\mathcal F_j],\theta^*-\vartheta_j\rangle
 \ge\frac15\mathcal E_j.
 \label{ext:local-local-drift}
\end{equation}

\paragraph{Step 3: the full group second moment.}
Recall the trajectory Gibbs formula in Section~\ref{sec:setup}:
$\log(\pi_k^*(\tau\mid x)/\pi_k(\tau\mid x))=R/\beta-\log Z_k(x)$.
Here $Z_k(x)=\E_{\tau\sim\pi_k(\cdot\mid x)}e^{R/\beta}$ is the reference
normalizer. Summing the token errors gives
$\sum_i z_i=\log(\pi_\theta(\tau\mid x)/\pi_k^*(\tau\mid x))$.
Substituting into $U=R-\beta\log(\pi_\theta(\tau\mid x)/\pi_k(\tau\mid x))$ gives
$U=\beta\log Z_k(x)-\beta\sum_i z_i$.
For a fixed prompt, $z_i-\E_{q_i}z_i$ has conditional mean zero. Its conditional
variance equals that of $u_i$, since $z_i=u_i+m_i$ and $m_i$ is fixed by the history.
It is at most $\E_{q_i}u_i^2\le2\E_{p_i}u_i^2=2\sigma_i^2$, where the inequality
uses $p_i/q_i\ge1/2$. Equation~\eqref{ext:local-mean-sum} supplies the
conditional-mean bound. The variance conclusion of
Lemma~\ref{ext:martingale-continuation}, with $B_0=2$, therefore gives
\[
 \Var_q(U)\le2\beta^2(2+\Gamma^2)\sum_iv_i.
\]
Bounded features and $w_i\le2$ give
\[
 \E_q\Psi=0,\qquad \E_q\|\Psi\|^2\le2T,
 \qquad \|\Psi\|\le4T.
\]
The second-moment bound follows from martingale orthogonality and
$\E_{q_i}[w_i^2\|\psi_i\|^2]
 \le2\E_{p_i}\|\phi-\E_{p_i}\phi\|^2\le2$.
Apply Lemma~\ref{ext:group-covariance} conditionally on the prompt with
$S=\Psi$, $M=2T$, and $R_S=4T$. It gives
\[
 \E\|\widehat g\|^2
 \le2T\left(1+\frac{8T}{G-1}\right)\Var_q(U).
\]
This bound includes the cross-response contribution for $G=2$.
Averaging over prompts proves
\begin{equation}
 \E[\|\widehat g_j\|^2\mid\mathcal F_j]\le\Lambda\mathcal E_j,
 \qquad
 \Lambda=4\beta T(2+\Gamma^2)
       \left(1+\frac{8T}{G-1}\right).
 \label{ext:local-local-noise}
\end{equation}
Because $G\ge2$, we have
$2T(1+8T/(G-1))\le32T\left(1+T/G\right)$. Since $\Gamma<3/5$,
$\Lambda\le64(2+9/25)\beta T\left(1+T/G\right)<152\beta T\left(1+T/G\right)$ and
$\eta\Lambda<152/2048<1/5$.

\paragraph{Step 4: gradient recursion and accumulated energy.}
By the definition of $\mathcal F_j$, the quantities $\vartheta_j$, $\theta^*$, and $\mathcal E_j$ are fixed when
conditioning on $\mathcal F_j$. The fresh group supplies the conditional
drift and second-moment bounds in
Equations~\eqref{ext:local-local-drift} and~\eqref{ext:local-local-noise}.
The representation assumption puts $\theta^*$ in the projection ball.
Nonexpansiveness of projection, the drift bound, and $\eta\Lambda\le1/5$
give the conditional recursion
\begin{equation}
 \E[\|\vartheta_{j+1}-\theta^*\|^2\mid\mathcal F_j]
 \le\|\vartheta_j-\theta^*\|^2-\frac\eta5\mathcal E_j.
 \label{ext:local-recursion}
\end{equation}
Taking expectations and telescoping gives

\begin{equation}
 \sum_{j<L}\E\mathcal E_j
 \le\frac{5\|\theta_k-\theta^*\|^2}{\eta}
 \le\frac{5r_{\mathrm{proj}}^2}{\eta}
 =640B\sqrt T\left(1+T/G\right).
 \label{ext:local-energy}
\end{equation}
The equality uses $r_{\mathrm{proj}}=B/\beta$, $\eta^{-1}=2048\beta T\left(1+T/G\right)$, and $\beta=16B\sqrt T$.
Projection bounds every iterate; the residuals and token weights are also bounded. All
conditional expectations in the recursion are therefore finite, and the argument
requires no stopping-time event or failure probability.

\paragraph{Step 5: moment transfer and random output.}
The energy $\mathcal E_j$ averages conditional token variances over behavior
histories. Apply Lemma~\ref{ext:moment-transfer} with
$f_i=\sigma_i^2\le c_0^2$ to transfer these averages to learner histories.
H\"older across token positions and Jensen across prompts give
\[
 \beta\sum_i\E_{P_i}\sigma_i^2
 \le\mathcal E_j+
 [\Omega(\mathcal C_s^{s-1}-1)]^{1/s}\mathcal E_j^{1-1/s}.
\]
The conditional KL is $m_i\le(8/15)\sigma_i^2\le3\sigma_i^2/5$ by
Lemma~\ref{ext:local-q-aggregate}. The trajectory KL chain rule therefore gives
\[
 \beta\KL(\pi_{\vartheta_j}\Vert\pi_k^*)
 \le\frac35\mathcal E_j+
 \frac35[\Omega(\mathcal C_s^{s-1}-1)]^{1/s}\mathcal E_j^{1-1/s},
\]
where trajectory KL includes the prompt average.
Average over the independent uniform output index and training randomness.
Jensen's inequality and Equation~\eqref{ext:local-energy},
$\sum_{j<L}\E\mathcal E_j\le A$, prove
Equation~\eqref{ext:inner-moment-rate}.
For bounded prefix ratios, Lemma~\ref{ext:moment-transfer} instead gives
$\beta\sum_i\E_{P_i}\sigma_i^2\le\mathcal C_\infty\mathcal E_j$,
yielding $3\mathcal C_\infty A/(5L)$.

Every step conditions on the starting reference, so these bounds also hold
conditionally at each random reference encountered during training.
\end{proof}

The moment order changes only the history-transfer step of the analysis.
The update uses current-token ratios throughout, and every displayed
inner-error bound tends to zero as $L$ grows.

\subsection{Reward improvement across reference refreshes}
\label{ext:outer}

The inner result bounds the error of one solve with a frozen reference. To connect
successive solves, let $p_k(x)$ be the single-response success probability of reference
$\pi_k$, and let $p_k^*(x)$ be that of its exact Gibbs target $\pi_k^*$. These satisfy
$\operatorname{logit}p_k^*(x)=\operatorname{logit}p_k(x)+1/\beta$ by
Proposition~\ref{thm:prox-improvement}. We now track how expected proximal gaps,
averaged over prompts, affect this improvement across reference refreshes. This allows
us to use the inner error bound without assuming a uniform success-logit stopping
tolerance.

Theorem~\ref{ext:reward} states the initial coverage requirement,
outer-round count, inner sampling budget, and final reward guarantee.
We state and prove it after the propagation lemmas below.

\par\medskip\phantomsection\noindent\textbf{Propagation of proximal errors.}\label{ext:src:sec:outer} We express binary success probabilities in angular coordinates and track the error
relative to the ideal Gibbs trajectory. The resulting bound depends only on proximal
gaps, so it applies to any inner optimizer.

\begin{lemma}[Accumulation of errors in a concave recursion]
\label{ext:src:lem:concave-recursion}
Let $F:[0,\pi/2]\to[0,\pi/2]$ be increasing and concave, with $F(0)=0$
and $F(\alpha)\ge\alpha$. Consider an ideal sequence satisfying
$\alpha_{k+1}^*=F(\alpha_k^*)$ and an approximate sequence $\alpha_k\in[0,\pi/2]$, both
starting at the same positive value $\alpha_0=\alpha_0^*>0$.
Define the one-step error and the accumulated relative deficit by
\[
 \xi_k=|\alpha_{k+1}-F(\alpha_k)|,\qquad
 a_k=\left(1-\frac{\alpha_k}{\alpha_k^*}\right)_+,
 \qquad (t)_+=\max\{t,0\}.
\]
Then
\begin{equation}\label{ext:src:eq:relative-deficit-recursion}
 a_K\le\sum_{k=0}^{K-1}\frac{\xi_k}{\alpha_{k+1}^*}
       \le\frac{\sum_{k=0}^{K-1}\xi_k}{\alpha_0^*}.
\end{equation}
Thus the relative deficit is bounded by the sum of one-step errors, each scaled
by the corresponding ideal value.
\end{lemma}
\begin{proof}
The definition of $a_k$ gives $0\le a_k\le1$ and
$\alpha_k\ge(1-a_k)\alpha_k^*$.
Concavity and $F(0)=0$ imply
$F((1-a_k)\alpha_k^*)\ge(1-a_k)F(\alpha_k^*)$.
Combining this with monotonicity and the one-step error bound gives
\begin{align*}
 \alpha_{k+1}
 &\ge F(\alpha_k)-\xi_k\\
 &\ge F((1-a_k)\alpha_k^*)-\xi_k\\
 &\ge(1-a_k)F(\alpha_k^*)-\xi_k
  =(1-a_k)\alpha_{k+1}^*-\xi_k.
\end{align*}
Divide by $\alpha_{k+1}^*>0$, rearrange, and take the positive part:
\[
 a_{k+1}\le a_k+\frac{\xi_k}{\alpha_{k+1}^*}.
\]
Starting from $a_0=0$, sum over $k$ to obtain the first bound in
Equation~\eqref{ext:src:eq:relative-deficit-recursion}.
The ideal values are nondecreasing because $F(\alpha)\ge\alpha$, so
$\alpha_{k+1}^*\ge\alpha_0^*>0$. Replacing each denominator by
$\alpha_0^*$ proves the second bound.
\end{proof}

The recursion lemma controls a relative deficit. The following elementary bound
converts that deficit into a reward gap, using initialization coverage and the
number of exact Gibbs rounds.

\begin{lemma}[From relative angular deficit to reward gap]
\label{ext:src:lem:coverage-reward}
Let $p_0\in(0,1]$ be the initial success probability and $p_\star\in[0,1]$
a comparator success probability with $p_\star\le C_\star p_0$, where
$C_\star\ge1$. For a tolerance $\varepsilon_0\in(0,1)$ and $\beta>0$, suppose
$K\ge\lceil\beta\log(3C_\star/\varepsilon_0)\rceil$.
Iterating the success-odds identity in Equation~\eqref{eq:promptwise-improvement}
of Proposition~\ref{thm:prox-improvement} multiplies the initial odds by
$e^{K/\beta}$. Thus the ideal success probability after $K$ exact Gibbs rounds is
\[
 P_K=\frac{e^{K/\beta}p_0}{1-p_0+e^{K/\beta}p_0}.
\]
For any actual final success probability $p_K\in[0,1]$, define
$\alpha_K=\arcsin\sqrt{p_K}$, $\alpha_K^*=\arcsin\sqrt{P_K}$, and
$a_K=(1-\alpha_K/\alpha_K^*)_+$ as in
Lemma~\ref{ext:src:lem:concave-recursion}. Then
\begin{equation}\label{ext:src:eq:reward-relative-deficit}
 p_\star-p_K\le \varepsilon_0+16p_\star a_K^2.
\end{equation}
\end{lemma}
\begin{proof}
Equation~\eqref{ext:src:eq:reward-relative-deficit} is immediate if $p_\star<\varepsilon_0$.
Otherwise $p_\star\le C_\star p_0$ and $e^{K/\beta}\ge3C_\star/\varepsilon_0$ give
$P_K\ge3p_\star/(\varepsilon_0+3p_\star)\ge3/4$, hence $\alpha_K^*\ge\pi/3$.
If $p_\star\le1/2$ and $p_K<p_\star$, then $\alpha_K<\pi/4$, so $a_K>1/4$
and $p_\star-p_K\le16p_\star a_K^2$; the case $p_K\ge p_\star$ is immediate.
For $p_\star>1/2$ and $p_K\le P_K$, the Lipschitz property of cosine gives
\[
 1-p_K\le2(1-P_K)+2(\alpha_K^*-\alpha_K)^2.
\]
The convex function $p_\star-1+2\varepsilon_0/(\varepsilon_0+3p_\star)$ is at most $\varepsilon_0$ on $[\varepsilon_0,1]$,
as verified at both endpoints. Therefore
\[
 p_\star-p_K\le \varepsilon_0+2(\alpha_K^*)^2a_K^2\le \varepsilon_0+16p_\star a_K^2.
\]
The last inequality uses $\alpha_K^*\le\pi/2$ and $p_\star>1/2$, which imply
$2(\alpha_K^*)^2\le\pi^2/2<8\le16p_\star$.
If $p_K>P_K$, the same endpoint bound gives $p_\star-p_K\le \varepsilon_0$.

\end{proof}

\begin{lemma}[Relative angular transfer]
\label{ext:src:lem:relative-angular-transfer}
Assume binary rewards, KL coefficient $\beta>0$ from
\eqref{eq:mech-objective}, and a comparator $\pi_\star$
satisfying \eqref{ext:coverage} with $C_\star\ge1$.
Let $\varepsilon_0\in(0,1)$ be a finite-round reward tolerance and let the
number of reference refreshes satisfy
$K\ge\lceil\beta\log(3C_\star/\varepsilon_0)\rceil$. Any sequence of positive policies
$\pi_0,\ldots,\pi_K$ starting at the covered initialization satisfies
\begin{equation}\label{ext:src:eq:relative-angular-transfer}
 J(\pi_\star)-J(\pi_K)
 \le \varepsilon_0+\frac{2\pi^2C_\star}{\beta}
                  \left(\sum_{k=0}^{K-1}\sqrt{D_k}\right)^2
 \le \varepsilon_0+\frac{2\pi^2C_\star K}{\beta}\sum_{k=0}^{K-1}D_k,
\end{equation}
where $D_k=\beta\E_x\KL(\pi_{k+1}\Vert \pi_k^*)$ is the unrestricted proximal
gap. In particular, for an expected inner-gap tolerance $\zeta>0$,
$\E D_k\le\zeta$ at every round implies
\begin{equation}\label{ext:src:eq:relative-angular-expected}
 J(\pi_\star)-\E J(\pi_K)
 \le \varepsilon_0+\frac{2\pi^2C_\star K^2}{\beta}\zeta.
\end{equation}
\end{lemma}
The bound separates the finite-round tolerance from the accumulated errors of the inner
solves.

\begin{proof}[Proof of Lemma~\ref{ext:src:lem:relative-angular-transfer}]
Fix a prompt $x$. If $p_0(x)=0$, coverage forces $p_\star(x)=0$,
so the reward-gap bound is immediate. Assume henceforth $p_0(x)>0$.
We compare the actual sequence of trained policies with an
ideal sequence in which every proximal problem is solved exactly. Let
$p_k=\Pr_{\pi_k}[R=1\mid x]$ be the actual policy's success probability after
$k$ rounds. Let $P_k$ be the success probability along the ideal sequence,
starting from the same initialization, so $P_0=p_0$.
For binary rewards, an exact Gibbs update multiplies the success odds by
$e^{1/\beta}$ (Proposition~\ref{thm:prox-improvement}). Hence
\[
 P_{k+1}=\frac{e^{1/\beta}P_k}{1-P_k+e^{1/\beta}P_k}.
\]
The actual inner solves approximate their Gibbs targets. Our goal is to control
how far the actual success probabilities $p_k$ fall behind the ideal probabilities
$P_k$ as these approximation errors accumulate. The comparator success probability
is $p_\star=\Pr_{\pi_\star}[R=1\mid x]$; the theorem's initialization coverage
condition gives $p_\star\le C_\star p_0$.

To relate success probabilities to the KL errors of the inner solves, express each
binary success distribution as a point on the unit circle:
\[
 (\sqrt{1-p_k},\sqrt{p_k})=(\cos\alpha_k,\sin\alpha_k),\qquad
 \alpha_k=\arcsin\sqrt{p_k},\qquad
 \alpha_k^*=\arcsin\sqrt{P_k}.
\]
The angle $\alpha_k$ describes the actual policy, and $\alpha_k^*$ describes the
ideal sequence. Distances between these unit-circle points are Hellinger distances,
which we will bound using trajectory KL at the end of the proof.
The Gibbs map in angular coordinates is
\[
 F(\alpha)=\arctan(e^{1/(2\beta)}\tan \alpha),\qquad
 F'(\alpha)=\frac{e^{1/(2\beta)}}
             {1+(e^{1/\beta}-1)\sin^2\alpha}.
\]
Thus $F$ is increasing and concave on $[0,\pi/2]$, with $F(0)=0$ and $F(\alpha)\ge \alpha$.
In particular, $\alpha_{k+1}^*=F(\alpha_k^*)$; $\sin^2 F(\alpha_k)$ is the success
probability after the exact Gibbs step from the actual reference $\pi_k$.
We can therefore apply Lemma~\ref{ext:src:lem:concave-recursion} to these two
angle sequences. Its one-step error
$\xi_k=|\alpha_{k+1}-F(\alpha_k)|$ measures how accurately round $k$ solves the
proximal problem for its actual reference. Its relative deficit
$a_k=(1-\alpha_k/\alpha_k^*)_+$ measures how far training has fallen behind the
ideal sequence. The lemma gives Equation~\eqref{ext:src:eq:relative-deficit-recursion},
which we now convert into a reward bound.

Lemma~\ref{ext:src:lem:coverage-reward} now converts the relative deficit into
$p_\star-p_K\le \varepsilon_0+16p_\star a_K^2$. Combine this with the accumulated-deficit estimate in
Equation~\eqref{ext:src:eq:relative-deficit-recursion}. Since
$(\alpha_0^*)^2\ge p_0$ and $p_\star\le C_\star p_0$, we obtain
\[
 p_\star-p_K\le \varepsilon_0+16C_\star\left(\sum_{k<K}\xi_k\right)^2.
\]

Use the unnormalized squared Hellinger distance
$H^2(p,q)=\sum_a(\sqrt{p(a)}-\sqrt{q(a)})^2$.
For Bernoulli distributions at angular distance $0\le\xi\le\pi/2$,
$H^2=4\sin^2(\xi/2)\ge8\xi^2/\pi^2$, and $H^2\le\KL$.
For a fixed prompt $x$, compare the next actual response policy $\pi_{k+1}$
with the exact Gibbs target $\pi_k^*$ of the current reference. Mapping a response
$\tau$ to its binary reward $R(x,\tau)$ turns these response distributions into
Bernoulli distributions. Their success probabilities are $p_{k+1}$ and
$\sin^2 F(\alpha_k)$, respectively, so their angular distance is precisely
$\xi_k=|\alpha_{k+1}-F(\alpha_k)|$.

The data-processing inequality says that applying this reward map cannot increase
KL divergence. Let $\operatorname{Bern}(p)$ denote the binary distribution with
success probability $p$. Combining data processing with the preceding Hellinger
bound gives, for each prompt,
\begin{align*}
 \frac{8}{\pi^2}\xi_k(x)^2
 &\le H^2\!\left(\operatorname{Bern}(p_{k+1}(x)),
                 \operatorname{Bern}(\sin^2 F(\alpha_k(x)))\right)\\
 &\le\KL\!\left(\operatorname{Bern}(p_{k+1}(x))
             \Vert\operatorname{Bern}(\sin^2 F(\alpha_k(x)))\right)\\
 &\le\KL\!\left(\pi_{k+1}(\cdot\mid x)\Vert\pi_k^*(\cdot\mid x)\right).
\end{align*}
Averaging over prompts and recalling the proximal-gap definition
$D_k=\beta\E_x\KL(\pi_{k+1}(\cdot\mid x)\Vert\pi_k^*(\cdot\mid x))$ yields
\[
 \E_x\xi_k(x)^2\le\frac{\pi^2}{8\beta}D_k.
\]
Finally, average the reward-gap bound above over prompts. The $L_2$ triangle
inequality then gives
\[
 J(\pi_\star)-J(\pi_K)
 \le \varepsilon_0+16C_\star\E_x\left(\sum_{k<K}\xi_k(x)\right)^2
 \le \varepsilon_0+16C_\star\left(\sum_{k<K}\sqrt{\E_x\xi_k(x)^2}\right)^2.
\]
Substituting the angular-error bound proves the first inequality in
Equation~\eqref{ext:src:eq:relative-angular-transfer}. Cauchy-Schwarz gives
the second, and expectation over training proves
\eqref{ext:src:eq:relative-angular-expected}.
\end{proof}

\subsection{Complete finite-sample reward guarantee}
\label{ext:complete-guarantee}
\Needspace{8\baselineskip}
\begin{theorem}[Single-response reward from covered initialization]
\label{ext:reward}
Assume the linear logistic model with bounded features and deterministic
binary rewards. For $B\ge1$ and $G\ge2$, set $\beta=16B\sqrt T$,
$r_{\mathrm{proj}}=B/\beta$, and $\eta=[2048\beta T(1+T/G)]^{-1}$.
At every actual reference $\pi_k=\pi_{\theta_k}$, assume there exists
$\theta_k^*$ with $\|\theta_k^*-\theta_k\|_2\le r_{\mathrm{proj}}$ such that, for every
prompt $x$ and legal response $\tau$,
$\pi_{\theta_k^*}(\tau\mid x)=\pi_k^*(\tau\mid x)$,
where $\pi_k^*$ is the Gibbs target defined in Section~\ref{sec:setup}.
Since softmax assigns positive probability to every legal history, this equality
of response distributions also gives equality of their conditional token distributions.
Start each inner solve at $\vartheta_0=\theta_k$ and run
Equation~\eqref{ext:update} under the \hyperref[ext:sampling]{sampling conditions stated with that update}.
Fix $\delta\in(0,1]$ and use $w_i$ and $V_{\log}$ from
Equation~\eqref{ext:log-budget-definition}. Assume the token condition
\begin{equation}\label{ext:delta-token-control}
 |\log w_i|\le\min\{\log2,3\delta/2\},\qquad V_{\log}\le\frac18
\end{equation}
holds uniformly over allowed actions, histories, rounds, and updates, almost surely.
Write $c_\delta=\exp((1+\delta^{-1})V_{\log})\le\exp((1+\delta^{-1})/8)$.

For $\varepsilon\in(0,1)$ and a comparator satisfying
Equation~\eqref{ext:coverage}, choose
$K=\lceil\beta\log(6C_\star/\varepsilon)\rceil$,
write $\kappa=1+T/G$ and $X=C_\star K^2/\varepsilon$, and run
$L=\lceil2000\kappa X+4000c_\delta \kappa X^{1+\delta}\rceil$
inner group updates at each reference.
After each solve, set $\theta_{k+1}=\vartheta_I$, where
$I\sim\operatorname{Unif}\{0,\ldots,L-1\}$ is drawn independently.
Then $\E J(\pi_K)\ge J(\pi_\star)-\varepsilon$, using at most
\begin{equation}\label{ext:moment-response-budget}
 GK+\frac{2000(G+T)C_\star K^3}{\varepsilon}
 +\frac{4000c_\delta(G+T)C_\star^{1+\delta}K^{3+2\delta}}
 {\varepsilon^{1+\delta}}
\end{equation}
scored responses. All displayed numerical constants are absolute.
\end{theorem}

\begin{proof}[Proof of Theorem~\ref{ext:reward}]
Set $\zeta=\varepsilon\beta/(4\pi^2C_\star K^2)$ and $s=1+1/\delta$.
Lemma~\ref{ext:log-moment} gives $\mathcal C_s\le c_\delta$.
Substituting the stated $L$ into Lemma~\ref{ext:inner} bounds each of its two
terms by $\zeta/2$: use $3A/(5\zeta)=96\pi^2\kappa X$,
$3\Omega/(5\zeta)<X$, $2000>192\pi^2$, and $4000>384\pi^2$.
Applying the lemma conditionally at each reference gives
$\E D_k\le\zeta$, where $D_k=\beta\E_x\KL(\pi_{k+1}\Vert\pi_k^*)$.
Lemma~\ref{ext:src:lem:relative-angular-transfer}, with tolerance
$\varepsilon/2$ and the stated $K$, then yields
\[
 J(\pi_\star)-\E J(\pi_K)
 \le\frac\varepsilon2+\frac{2\pi^2C_\star K^2}{\beta}\zeta
 =\varepsilon.
\]
Multiplying $L$ by $GK$ gives the response bound, including the $GK$
ceiling contribution.
\end{proof}

\begin{remark}[Uniform history coverage]
\label{ext:coverage-specializations}
For the coefficient $\mathcal C_\infty$ defined in
Equation~\eqref{ext:coverage-moments}, the bounded-coverage case of
Lemma~\ref{ext:inner} gives the same reward guarantee with
$L=\lceil96\pi^2\mathcal C_\infty\kappa X\rceil$, under the remaining
assumptions of Theorem~\ref{ext:reward} with its token condition replaced
by $1/2\le w_i\le2$ and $V_{\log}\le1/8$.
A uniform full-response ratio bound also bounds $\mathcal C_\infty$,
since $W_i=\E_q[W_{T+1}\mid h_i]$ for each fixed prompt and policy pair.
The coefficient $\mathcal C_\infty$ can grow exponentially with response
length, whereas $V_{\log}$ aggregates tokenwise log-ratio second moments
additively.
\end{remark}

\begin{remark}[Practical cost clipping]
\label{ext:bounded-credit}
Projection gives $|\log(\pi_{\vartheta_j}/\pi_k)|\le2B/\beta$ at every
legal token by Equation~\eqref{ext:local-iterates}. Thus cost clipping
with $\epsilon_F\ge2B/\beta$ leaves the analyzed update unchanged.
Section~\ref{sec:practical-bias-variance} discusses the bias and variance
trade-off when clipping is active.
\end{remark}

\begin{remark}[Approximate target representation]
\label{ext:approximate-representation}
For a parameter $\widetilde\theta_k^*$ approximating the Gibbs target,
define
$r_k(h,a)=\log(\pi_{\widetilde\theta_k^*}(a\mid h)/\pi_k^*(a\mid h))$.
At this parameter, the full residual satisfies
\[
 U_m=\beta\log Z_k(x)
     -\beta\sum_i r_k(h_{m,i},a_{m,i}).
\]
Hence $|r_k(h,a)|\le\bar r_k$ implies
$|U_m-\beta\log Z_k(x)|\le\beta T\bar r_k$.
The residual can therefore vary among responses within a prompt group.
\end{remark}

\section{Examples of local target representation}
\label{ext:representation-examples}

The local target representation condition of Theorem~\ref{ext:reward} requires the model to represent each Gibbs target
with a parameter change of norm at most $B/\beta$. The examples below meet this
requirement at every reference, including those produced by projected inner solves.
In the tabular, deterministic linear-MDP, and explicit shared-feature examples, the same constant $B$ works
for every $\beta>0$, so choosing $B$ before $\beta=16B\sqrt T$ is valid.

\paragraph{Independent tabular logits.}
Use one coordinate for each legal history-action pair, with one-hot
features in $\mathbb R^d$. Each conditional distribution then has its
own trainable logits. Equation~\eqref{eq:prox-soft-policy} implies that
adding $Q_k(h,a)/\beta$ to the reference logits gives the Gibbs policy.
A common shift of the logits at a history leaves its probabilities
unchanged, so we can choose
\[
 \theta^*_{k,h,a}-\theta_{k,h,a}
 =\frac1\beta\left(Q_k(h,a)
   -\frac{\max_{a'}Q_k(h,a')+\min_{a'}Q_k(h,a')}{2}\right).
\]
For binary rewards, $0\le Q_k(h,a)\le1$, and thus
$\|\theta_k^*-\theta_k\|_2\le\sqrt d/(2\beta)$.
Thus $B\ge\max\{1,\sqrt d/2\}$ ensures local representation at every reference. In a
fully branching response tree, however, $d$ can grow exponentially with $T$. This choice
of $B$ therefore does not generally yield a sampling budget polynomial in the response
horizon.

The two-prompt shared-head example in
Appendix~\ref{app:shared-head-separation} also satisfies the condition
at every reference, with $B\ge5$.

\paragraph{Deterministic linear MDPs.}
Consider a finite-horizon linear MDP \citep{jin2020linear} with deterministic
transitions, zero rewards before the final position, and a deterministic
binary terminal reward. At position $i$, let the normalized features be
$\phi_i(s,a)\in\mathbb R^{d_i}$, with $\|\phi_i(s,a)\|_2\le1$.
For $i<T$, assume the transition kernel has the linear representation
\[
 P_i(\cdot\mid s,a)=\phi_i(s,a)^\top\mu_i(\cdot),\qquad
 \left\|\int V(s')\,d\mu_i(s')\right\|_2
 \le M_i\|V\|_\infty
\]
for every bounded $V$, where $\mu_i$ is a vector of signed measures and
$M_i$ is fixed. At the final position, assume
$R(s,a)=\phi_T(s,a)^\top b_T\in\{0,1\}$ with $\|b_T\|_2\le M_T$.
Use one parameter block per stage and the corresponding block feature
at each history, with the state determined by its prompt and action prefix.
Every reference in this log-linear family is a Markov policy on these states.

For a fixed reference $\pi_k$, define backward soft values by
\begin{align*}
 Q_{k,T}(s,a)&=R(s,a),&
 Q_{k,i}(s,a)&=V_{k,i+1}(F_i(s,a))\quad(i<T),\\
 V_{k,i}(s)&=\beta\log\sum_a\pi_{k,i}(a\mid s)e^{Q_{k,i}(s,a)/\beta}.
\end{align*}
Here $F_i(s,a)$ is the deterministic next state. These are the conditional
log normalizers of the trajectory Gibbs target, so its token policy is
$\pi_{k,i}(a\mid s)\exp\{(Q_{k,i}(s,a)-V_{k,i}(s))/\beta\}$.
Backward induction gives $0\le Q_{k,i},V_{k,i}\le1$ for every reference
and every $\beta>0$. The transition representation implies
\[
 Q_{k,i}(s,a)=\phi_i(s,a)^\top v_{k,i},\qquad
 v_{k,i}=\int V_{k,i+1}(s')\,d\mu_i(s')\ (i<T),\qquad
 v_{k,T}=b_T.
\]
Thus $\|v_{k,i}\|_2\le M_i$, and adding $v_{k,i}/\beta$ to each
reference parameter block represents the Gibbs policy. The combined change
has norm at most $(\sum_i M_i^2)^{1/2}/\beta$.
Choosing $B\ge\max\{1,(\sum_i M_i^2)^{1/2}\}$ therefore verifies the
local representation condition uniformly over references and temperatures.
This permits choosing $B$ before $\beta=16B\sqrt T$.
The transition-operator and terminal-reward coefficient bounds are explicit
parts of this example; deterministic dynamics by themselves do not imply them.

\paragraph{A shared-feature family with $B=1$.}
For any $T\ge1$, take one prompt, actions $a_i\in\{-1,1\}$, and
\[
 \phi(h_i,a_i)=e_i\left(a_i\prod_{t<i}a_t\right)\in\mathbb R^T,
 \qquad R=\frac12\left(1+\prod_{i=1}^T a_i\right),
\]
where $e_i$ is the $i$th coordinate vector and the empty product is one.
The feature has norm one and ties the same parameter coordinate across
all histories at a given position. Each token's normalizer is
$2\cosh\theta_i$, independent of the prefix. Multiplying the response
probability by $e^{R/\beta}$ changes only the final coordinate, so the
Gibbs policy is represented exactly by
\[
 \theta_k^*=\theta_k+\frac{e_T}{2\beta},\qquad
 \|\theta_k^*-\theta_k\|_2=\frac1{2\beta}\le\frac B\beta
 \quad\text{with }B=1.
\]
The representation holds at every reference and horizon, and adding zero feature
coordinates leaves the bound unchanged. In the transformed coordinates
$a_i\prod_{t<i}a_t$, reward depends only on the final coordinate. Thus the example gives
uniform local representation in a shared-feature model whose reward structure is simple.

\paragraph{Fixed tangent features.}
For the linearized model described after Equation~\eqref{ext:head},
a sufficient condition is that, at every reference,
$Q_k(h,a)-c_k(h)=v_k^\top\phi(h,a)$ for some history-dependent offset
$c_k(h)$ and coefficient vector with $\|v_k\|_2\le B$.
Then $\theta_k+v_k/\beta$ represents the Gibbs target, also when the
logits include a fixed offset. Minimizing over valid coefficients and
history offsets gives the required representation norm; for fixed
history offsets, the minimum coefficient norm is the corresponding
reproducing-kernel Hilbert-space norm in the feature kernel.
Because the soft values depend on $\beta$, the norm bound must hold
at $\beta=16B\sqrt T$. One sufficient way to choose $B$ first is to
verify the same bound for every $\beta\ge16\sqrt T$ and every reference
in the model family. Approximate representation contributes the
residual described in Remark~\ref{ext:approximate-representation};
application to a nonlinear Transformer additionally requires control
of logit linearization and gradient errors along training.

\end{document}